\documentclass[preprint,3p]{elsarticle}

\journal{arXiv}
\usepackage{tabularx}

\usepackage[utf8]{inputenc}
\usepackage{bm}
\usepackage{amsmath}
\usepackage[hidelinks]{hyperref}
\hypersetup{colorlinks=true,linkcolor=blue,urlcolor=blue}
\usepackage{multicol}
\usepackage{booktabs}
\usepackage{amsthm}
\usepackage{esint}
\usepackage[table,xcdraw,dvipsnames]{xcolor}
\usepackage{graphicx}
\usepackage{float}
\usepackage{subcaption}
\usepackage{algorithm}
\usepackage{algpseudocode}
\usepackage{pdflscape}
\usepackage{array}
\usepackage{multirow}
\usepackage{diagbox}
\usepackage{cleveref}
\usepackage{amsfonts}
\usepackage{amssymb}
\usepackage{ulem}
\usepackage{fourier-orns}
\usepackage[labelfont=bf,labelsep=period]{caption}
\usepackage[labelfont=rm,font=footnotesize]{subcaption}

\numberwithin{equation}{section}

\DeclareMathAlphabet{\altmathcal}{OMS}{cmsy}{m}{n}
\DeclareMathAlphabet{\altmathcalb}{OMS}{cmsy}{b}{n}
\DeclareMathAlphabet{\mathcalboondox}{U}{BOONDOX-calo}{m}{n}
\DeclareMathAlphabet{\mathbbmsl}{U}{bbm}{m}{sl}

\newtheorem{theorem}{Theorem}
\newtheorem{corollary}{Corollary}
\newtheorem{assumption}{Assumption}
\newtheorem{lemma}{Lemma}

\newdefinition{remark}{Remark}

\newcommand{\orcid}[1]{\href{https://orcid.org/#1}{\texorpdfstring{\includegraphics*[width=8pt]{figs/orcid}}~~}{}}

\definecolor{drot}{rgb}{0.7,0,0.1}

\newcommand{\bo}[1]{\mathbf{#1}}

\begin{document}
	\baselineskip14pt
	\sloppy
	
	\begin{frontmatter}
		
		\title{$\Omega$: Operator-based Mixture Ensemble for Generative Assimilation}

		\author[UW]{Pouria Behnoudfar}
		\author[UW]{Nan Chen\corref{corr} }
		\cortext[corr]{Corresponding author}
		\address[UW]{{University of Wisconsin-Madison, Department of Mathematics, Madison, WI, USA}}
		
		\begin{abstract}
			Characterizing non-Gaussian posterior distributions in partially observed high-dimensional nonlinear systems remains a fundamental challenge in data assimilation. Ensemble Kalman filters rely on Gaussian approximations that can be inaccurate for strongly non-Gaussian posteriors, whereas particle filters suffer from severe scalability limitations. Recent score-based generative approaches improve posterior characterization but typically require supervised training with ground-truth posterior samples, which are unavailable in most practical applications. We introduce $\Omega$ (Operator-based Mixture Ensemble for Generative Assimilation), a scalable framework that integrates conditional Gaussian surrogate modeling, unsupervised score learning, and generative sampling. The conditional Gaussian surrogate provides a nonlinear non-Gaussian baseline approximation while admitting closed-form conditional posterior distributions for the unresolved variables. First, $\Omega$ exploits these closed-form conditional distributions to analytically recover the high-dimensional unobserved component, reducing computational cost and mitigating the curse of dimensionality. Second, $\Omega$ learns only the residual discrepancy beyond an analytical baseline through denoising score matching using ensemble trajectories alone, eliminating the need for ground-truth posterior samples and substantially reducing the learning burden. Third, $\Omega$ reconstructs the full non-Gaussian posterior distribution of both observed and unobserved variables through a Gaussian-mixture representation, enabling multimodal, skewed, and heavy-tailed statistics to be captured. Finally, $\Omega$ employs annealed Langevin sampling to iteratively refine ensemble members from the baseline toward the target posterior. The framework is validated on several turbulent models with intermittency and extreme events. Across all examples, $\Omega$ consistently improves posterior accuracy and the recovery of unobserved variables relative to the ensemble Kalman filter and conditional-Gaussian baselines, while remaining computationally tractable in high-dimensional settings.
			
		\end{abstract}
		\begin{keyword}
			Operator Learning; Non-Gaussian Data Assimilation; Score-Based Generative Modeling; Conditional Gaussian Nonlinear Systems; Langevin Sampling
		\end{keyword}
	\end{frontmatter}
	
	\section{Introduction}
	\label{sec:intro}

	Data assimilation (DA) is a framework for estimating the evolving state of a dynamical system by combining model forecasts with available observations~\cite{law2015data, majda2012filtering, asch2016data, evensen2022data,poulet2024slip}. It plays a central role in forecasting and uncertainty quantification across a wide range of applications, including geophysics, oceanography, and Earth science. At each assimilation time, the objective is to characterize the filtering posterior distribution and propagate this information forward through the dynamical model. A fundamental challenge is that the posterior distribution is often strongly non-Gaussian due to nonlinear dynamics, intermittency, regime transitions, and extreme events~\cite{bocquet2010beyond, van2019particle}. This challenge becomes particularly severe in partially observed high-dimensional systems, where sparse observations and unresolved processes further complicate posterior inference~\cite{Freitag2013Synergy, chen2023stochastic, wang2000data, chen2019multi,poulet2023physics}.

	Traditional DA methods seek to balance computational efficiency and statistical accuracy. The ensemble Kalman filter (EnKF) and its variants~\cite{evensen1994sequential, evensen2003ensemble, houtekamer2005ensemble, BEHNOUDFAR2026115035} are among the most widely used approaches because of their robustness, scalability, and success in large-scale applications such as meteorology~\cite{Anderson1996Data, kalnay2003atmospheric}, groundwater modeling~\cite{vrugt2003shuffled}, and other geophysical systems~\cite{lahoz2014data}. However, the EnKF relies on Gaussian approximations and linear covariance updates, which can introduce systematic errors when the posterior is strongly non-Gaussian and often require empirical tuning through inflation and localization~\cite{anderson2012localization, anderson2007exploring, buehner2007spectral, whitaker2012evaluating}. Particle filters (PF) and related sequential Monte Carlo methods~\cite{van2000unscented, gustafsson2010particle} can represent arbitrary posterior distributions, but their performance deteriorates rapidly in high-dimensional systems because of weight degeneracy and the associated curse of dimensionality~\cite{elfring2021particle}. Consequently, achieving accurate non-Gaussian posterior characterization while maintaining computational tractability remains a major open problem in DA.
	
	Recent advances in machine learning have motivated a new generation of DA methodologies based on deep generative models. Normalizing flows~\cite{rezende2015variational, kobyzev2020normalizing, papamakarios2021normalizing} provide flexible posterior representations but require tractable density evaluation and often scale poorly to high-dimensional state spaces. Variational autoencoders and related latent-variable models~\cite{kingma2013auto,khemakhem2020variational,behnoudfar2026bridging} have also been applied to state estimation, although they typically impose Gaussian latent representations that partially reintroduce the Gaussian bottleneck they seek to overcome. More recently, score-based and diffusion generative models~\cite{song2020score, ho2020denoising} have emerged as powerful tools for posterior sampling in inverse problems and DA because of their ability to represent complex multimodal distributions through iterative Langevin or reverse-diffusion refinement~\cite{manshausen2025generative, xun2026posterior}. Despite their success, existing score-based DA methods typically rely on supervised training using reference states or posterior samples that are unavailable in most practical applications~\cite{rozet2023score, batzolis2021conditional}, incur substantial computational cost during inference~\cite{price2025probabilistic, manshausen2025generative, hasan2025pc}, and often focus their non-Gaussian corrections on the observed variables, leaving the posterior distribution of unresolved variables largely governed by Gaussian approximations.

	In this work, we introduce $\Omega$ (Operator-based Mixture Ensemble for Generative Assimilation), a hybrid framework that integrates conditional Gaussian surrogate modeling, unsupervised score learning, and generative sampling. The framework begins with a Conditional Gaussian Nonlinear System (CGNS) surrogate~\cite{chen2018conditional,chen2022conditional}, which serves as a computationally efficient forecast model while providing closed-form conditional posterior distributions for the unresolved variables. Importantly, the CGNS is not a globally Gaussian approximation: the overall dynamics remain nonlinear and are capable of generating strongly non-Gaussian statistics, while the conditional distributions of the unresolved variables can still be obtained analytically. These closed-form conditional distributions therefore provide a principled baseline approximation of the posterior and enable analytical treatment of the unresolved variables. Building upon this baseline, a residual operator is trained through Noise2Noise denoising score matching (DSM)~\cite{lehtinen2018noise2noise} using only ensemble trajectories, eliminating the need for ground-truth posterior samples and substantially reducing the learning burden. Finally, annealed Langevin sampling refines the ensemble toward the target posterior, and the corrected conditional distributions are assembled into a Gaussian-mixture representation of the full posterior distribution.
	
	The key advantages of $\Omega$ are fourfold. First, we reformulate the DA problem using an augmented state consisting of the observed variables and the CGNS conditional mean of the unresolved variables, which keeps the learning and sampling procedures out of the full unresolved-state space while retaining information about its posterior distribution (Section~\ref{sec:cgns_forecast}). Second, we train a non-Gaussian residual operator on this augmented state that goes beyond the linear covariance corrections of the EnKF, allowing observational information to propagate into the unresolved variables in a fully nonlinear manner (Section~\ref{sec:joint_score}). Third, the residual operator is learned entirely through unsupervised Noise2Noise denoising score matching using ensemble trajectories alone, eliminating the need for paired training data or ground-truth posterior samples (Section~\ref{sec:training}). Fourth, annealed Langevin sampling refines the augmented ensemble toward the target posterior without learning the full posterior from scratch. After the Langevin correction, each unresolved-state sample is drawn analytically from a conditional Gaussian distribution parameterized by the corrected conditional mean. Marginalizing over the ensemble yields a Gaussian-mixture posterior that is fully non-Gaussian through the diversity of the learned corrections, enabling multimodal, skewed, and heavy-tailed statistics to be accurately captured without Langevin ever acting directly in the unresolved-state space (Section~\ref{sec:sampling}).
	
	The remainder of the paper is organized as follows. Section~\ref{sec:cgns} reviews the CGNS filtering framework. Section~\ref{sec:cgns_enkf} introduces the augmented-state formulation based on the EnKF and CGNS. Section~\ref{sec:joint_score} presents the joint residual operator, the CGNS analytical score baseline, the unsupervised Noise2Noise denoising score matching (DSM) training framework, and the annealed Langevin analysis step together with the analytical sampling procedure for the unresolved variables. Section~\ref{sec:experiments} presents numerical experiments on a hierarchy of the turbulent systems with intermittency and extreme events. Finally, Section~\ref{sec:conclud} concludes the paper and discusses future research directions.

	\section{Conditional Gaussian Nonlinear System (CGNS)}
	\label{sec:cgns}
	
	Conditional Gaussian nonlinear systems (CGNS) provide a broad class of nonlinear stochastic dynamical systems that retain an analytically tractable conditional structure~\cite{chen2018conditional}. A key feature of the framework is that the full system can exhibit strongly nonlinear, non-Gaussian, intermittent, and multiscale behavior, while the conditional distribution of a subset of variables remains Gaussian and admits closed-form filtering and smoothing formulae. Importantly, the conditional Gaussian structure does not imply that the joint distribution of the full system is Gaussian. In fact, CGNS models can generate highly non-Gaussian statistics, including intermittency, skewness, heavy tails, and regime transitions~\cite{chen2018conditional}. This unique combination of analytical tractability and nonlinear expressiveness makes CGNS an attractive building block for scalable uncertainty quantification and data assimilation in high-dimensional systems. The framework also serves as the foundation of the proposed $\Omega$ methodology, where the closed-form conditional distributions provide a principled baseline approximation of the posterior and enable analytical treatment of high-dimensional unresolved variables.
	
	Consider a dynamical system in which the state vector $\mathbf{U} = (\bo X, \bo Y) \in \mathbb{R}^{d_X} \times \mathbb{R}^{d_Y}$ is partitioned into an observed component $\bo X \in \mathbb{R}^{d_X}$ and an unobserved (or latent) component $\bo Y \in \mathbb{R}^{d_Y}$. The CGNS is defined by the coupled stochastic differential equations~\cite{chen2018conditional}:
	\begin{align}
		d\bo X &= \left[\mathbf{A}_0(\bo X, t) + \bo {A}_1(\bo X, t)\, \bo Y\right] dt
		+ \boldsymbol{\Sigma}_X(\bo X, t)\, d\bo{W}_X(t),
		\label{eq:cgns_obs} \\
		d\bo Y &= \left[\mathbf{a}_0(\bo X, t) + \bo {a}_1(\bo X, t)\, \bo Y\right] dt
		+ \boldsymbol{\Sigma}_Y(\bo X, t)\, d\bo {W}_Y(t),
		\label{eq:cgns_lat}
	\end{align}
	where $\mathbf{W}_X \in \mathbb{R}^{d_X}$ and $\mathbf{W}_Y \in \mathbb{R}^{d_Y}$ are independent standard Wiener processes. The coefficient functions $\mathbf{A}_0$, $\mathbf{A}_1$, $\mathbf{a}_0$, $\mathbf{a}_1$, $\boldsymbol{\Sigma}_X$, and $\boldsymbol{\Sigma}_Y$ may be arbitrary nonlinear functions of the observed state $\mathbf X$ and time $t$. Although the dynamics are conditionally linear in $\mathbf Y$, the coupled system remains fully nonlinear and can generate strongly non-Gaussian statistics.
	
	\subsection{Closed-Form Posterior for the Latent Variables}
	
	The CGNS framework requires that the observed component $\bo X$ is available {continuously in time}. Formally, the filtration generated by the observation path
	\begin{equation}
		\mathcal{F}_t^X = \{\bo X(s) : 0 \leq s \leq t\}
	\end{equation}
	is assumed to be fully accessible up to the current time $t$. Then, conditioned on the observation path $\mathcal{F}_t^X$, CGNS provides the posterior distribution of the latent state $\bo Y(t)$ for all $t$:
	\begin{equation}
		p\!\left(\bo Y(t) \mid \mathcal{F}_t^X\right) =
		\mathcal{N}\!\left(\boldsymbol{\mu}(t),\, \mathbf{R}(t)\right).
	\end{equation}
	The conditional mean $\boldsymbol{\mu}(t) \in \mathbb{R}^{d_Y}$ and conditional covariance $\mathbf{R}(t) \in \mathbb{R}^{d_Y \times d_Y}$ evolve according to the following {closed-form ODEs}, driven by the observed path $\bo X(t)$ \cite{liptser2013statistics}:
	\begin{subequations}
		\label{eq:closed}
		\begin{align}
			d\bo \mu &= (\bo a_0 + \bo a_1 \boldsymbol{\mu}) dt + \bo R \bo A_1^* (\boldsymbol{\Sigma}_X \boldsymbol{\Sigma}_X^*)^{-1} \left[ d\bo X - (\bo A_0 + \bo A_1 \boldsymbol{\mu}) dt \right],  \label{eq:cgns_mean}\\[2ex]
			d \bo R &= \left[ \bo a_1 \bo R + \bo R \bo a_1^* + \boldsymbol{\Sigma}_Y \boldsymbol{\Sigma}_Y^* - (\bo R A_1^*) (\boldsymbol{\Sigma}_X \boldsymbol{\Sigma}_X^*)^{-1} (\bo A_1 \bo R) \right] dt.
			\label{eq:cgns_cov}
		\end{align}
	\end{subequations}
	Equation~\eqref{eq:closed} is a matrix Riccati equation with a solution that depends on the model coefficients evaluated along $\bo X(t)$. Equation~\eqref{eq:cgns_mean} updates the posterior mean in response to innovations in the observed process $d \bo {X}$. The closed-form nature of these equations is a direct consequence of the conditional linearity of \eqref{eq:cgns_lat}.
	
	In many geophysical and physical applications (see \cite{chen2018conditional} for a comprehensive list of physical models), $\bo X$ may represent a densely sampled observable (e.g., sea surface temperature, radar reflectivity, or high-frequency sensor measurements), while $\bo Y$ encodes the unresolved or slowly varying latent modes.

	\begin{remark}[Exact nonlinear filtering]
		The CGNS equations \eqref{eq:cgns_mean}--\eqref{eq:cgns_cov} constitute an exact solution to the nonlinear filtering problem for this system class, with no Monte Carlo approximation required for the latent component. The nonlinearity of the full system enters through the coefficient functions $\mathbf{a}_0(\bo X,t)$, $\mathbf{A}_1(\bo X,t)$, and $\boldsymbol{\Sigma}_Y(\bo X,t)$, which are evaluated pointwise along the observed trajectory $\bo X(t)$.
	\end{remark}
	
	\subsection{Applications and Significance of the CGNS Framework}

	The broad applicability of the CGNS framework and the analytical tractability of its associated filtering equations provide a strong motivation for its use as the baseline model in $\Omega$. In particular, the closed-form conditional distributions furnish a computationally efficient approximation of the posterior while retaining essential nonlinear and non-Gaussian features of the underlying dynamics.
	
	Many nonlinear stochastic dynamical systems fit naturally into the CGNS framework~\cite{chen2018conditional}. Representative examples include stochastic versions of the Lorenz models, low-order Charney-DeVore atmospheric models, topographic mean-flow interaction models, coupled reaction-diffusion systems in neuroscience and ecology, and multiscale models for geophysical fluid dynamics. These examples span a wide range of applications and demonstrate that the CGNS framework is not a specialized model class, but rather a broad modeling paradigm for nonlinear stochastic systems. The framework has also been used to develop realistic reduced-order models for phenomena such as the Madden-Julian Oscillation and Arctic sea ice variability~\cite{chen2014predicting, chen2022efficient}.
	
	Beyond its modeling capability, the CGNS framework has enabled a wide range of theoretical and practical developments. The closed-form data assimilation formulae have been exploited in nonlinear Lagrangian data assimilation, uncertainty quantification, state estimation, prediction of intermittent atmosphere and ocean phenomena, and efficient numerical solutions of high-dimensional Fokker-Planck equations~\cite{chen2014information, majda2018model, chen2017beating}. These successes demonstrate that the analytical conditional structure of CGNS provides a powerful and versatile foundation for scalable inference in complex nonlinear systems. The analytical conditional structure underlying CGNS has also been exploited in stochastic superparameterization~\cite{grooms2013efficient} and dynamic stochastic superresolution~\cite{branicki2013dynamic}.
	
	\section{Ensemble Kalman Filter with CGNS Forecast Model for a General Nonlinear System}
	\label{sec:cgns_enkf}
	
	The closed-form filtering equations in Section~\ref{sec:cgns} require continuous-time observations of the observed variables. In practice, however, observations are available only at discrete times and the underlying dynamics may not exactly satisfy the CGNS assumptions. To bridge this gap, we embed a CGNS surrogate within an EnKF framework. The CGNS provides a computationally efficient forecast model together with analytical conditional distributions for the unresolved variables, while the EnKF handles discrete-time observations and serves as the baseline posterior approximation for the subsequent score-based correction.
	
	We consider a general dynamical system with a complex nonlinear model. The observations are available at discrete times $t_1 < t_2 < \cdots < t_K$ for only $\bo X$. That is, at each observation time $t_k$, we receive a noisy partial measurement, such that:
	\begin{equation}
		\mathbf{Z}_k = \mathbf{H}\mathbf{U}_k + \boldsymbol{\eta}_k,
		\qquad \boldsymbol{\eta}_k \sim \mathcal{N}(0, \mathbf{R}_{\mathrm{obs}}),
		\label{eq:observation}
	\end{equation}
	where $\bo U_k = \bo U(t_k)$ and $\mathbf{H} = [\bo H_X\,, \,\bo 0]  \in \mathbb{R}^{d_z \times d}$ is an observation operator where $d = d_X + d_Y$, $\mathbf{H}_X \in \mathbb{R}^{d_z \times d_X}$ is the restriction of $\mathbf{H}$ to the observed component, and $\mathbf{0} \in \mathbb{R}^{d_z \times d_Y}$. $\mathbf{R}_{\mathrm{obs}} \in \mathbb{R}^{d_z \times d_z}$ is the observation noise covariance. We denote the history of observations up to time $t_k$ as $\mathbf{Z}_{0:k} = \{\mathbf{Z}_1,\ldots,  \mathbf{Z}_k\}$.
	
	To handle discrete observations while retaining the analytical structure of the CGNS, we embed the CGNS as the forecast model within an EnKF. Rather than tracking raw $\bo Y$ samples in the ensemble, we track the augmented state
	\begin{equation}
		\mathbf{V}_k^{(i)} = \left(\bo X_k^{(i)},\,
		\boldsymbol{\mu}_k^{(i)}\right)
		\in \mathbb{R}^{d_X} \times \mathbb{R}^{d_Y},
		\label{eq:augmented_def}
	\end{equation}
	where $\boldsymbol{\mu}_k^{(i)}$ is the CGNS conditional mean of $\bo Y_k^{(i)}$ given the $i$-th ensemble trajectory $\bo X_{0:k}^{(i)}$. The CGNS conditional covariance $\mathbf{R}_k^{(i)} \in \mathbb{R}^{d_Y \times d_Y}$ is tracked separately for each ensemble. The EnKF then operates in two steps: a forecast step, in which the augmented ensemble is propagated forward through the CGNS dynamics, and an analysis step, in which it is updated using the incoming partial observation $\mathbf{Z}_k$.
	
	The augmented-state formulation plays a central role in $\Omega$. Rather than directly learning or sampling in the high-dimensional unresolved state space, we represent the unresolved variables through their CGNS conditional statistics. The conditional mean $\mu_k^{(i)}$ provides a compact representation of the unresolved-state posterior and serves as the target of the subsequent non-Gaussian correction. This formulation allows the score-learning and Langevin-sampling procedures to operate on the augmented state rather than directly on the full unresolved-state distribution.
	
	\subsection{Forecast Step: CGNS Propagation}
	\label{sec:cgns_forecast}
	
	Between observation times $t_{k-1}$ and $t_k$, the analysis
	augmented ensemble
	$\{\mathbf{V}_{a,k-1}^{(i)}\}_{i=1}^N =
	\{(\bo X_{a,k-1}^{(i)},\, \boldsymbol{\mu}_{a,k-1}^{(i)})\}_{i=1}^N$,
	together with the analysis covariances
	$\{\mathbf{R}_{a,k-1}^{(i)}\}_{i=1}^N$, is propagated forward
	through the CGNS dynamics in three stages.
	
	\paragraph{Stage 1: Propagation of the observed component.}
	For each ensemble member $i = 1, \ldots, N$, the observed component is propagated by integrating the SDE \eqref{eq:cgns_obs} from $t_{k-1}$ to $t_k$:
	\begin{equation}
		d\bo X^{(i)} = \left[\mathbf{A}_0(\bo X^{(i)}, t) + \bo {A}_1(\bo X^{(i)}, t)\, \bo Y^{(i)}\right] dt
		+ \boldsymbol{\Sigma}_X(\bo X^{(i)}, t)\, d\bo{W}^{(i)}_X(t),
		\quad t \in [t_{k-1}, t_k],
		\label{eq:forecast_X}
	\end{equation}
	with initial condition $\bo X^{(i)}(t_{k-1}) = \bo X_{a,k-1}^{(i)}$. Each realization $\mathbf{W}_X^{(i)}$ is an independent Wiener process, so the $N$ trajectories $\{\bo X_{t_{k-1}:t_k}^{(i)}\}_{i=1}^N$ are statistically independent given the analysis ensemble at $t_{k-1}$. The forecast ensembles $\bo X_{f,k}^{(i)} = \bo X^{(i)}(t_k)$ approximate the marginal predictive distribution $p(\bo X_k \mid \mathbf{Z}_{0:k-1})$, which is generally non-Gaussian due to the nonlinear drift $\mathbf{A}_0$.
	
	\paragraph{Stage 2: CGNS conditional forecast for the latent
		component}
	Given the trajectory $\bo X_{t_{k-1}:t_k}^{(i)}$, following a similar approach, we solve \eqref{eq:cgns_mean}--\eqref{eq:cgns_cov} by starting from the analysis conditional at $t_{k-1}$, considering $\bo Y_{k-1} \mid \bo X_{0:k-1}^{(i)} \sim
	\mathcal{N}\!\left(\boldsymbol{\mu}_{a,k-1}^{(i)},\,
	\mathbf{R}_{a,k-1}^{(i)}\right)$,
	the conditional mean $\boldsymbol{\mu}^{(i)}(t)$ and covariance $\mathbf{R}^{(i)}(t)$ are evolved forward by integrating the CGNS {forecast ODEs} along $\bo X_{t_{k-1}:t_k}^{(i)}$. The initial conditions are $\boldsymbol{\mu}^{(i)}(t_{k-1}) =
	\boldsymbol{\mu}_{a,k-1}^{(i)}$ and $\mathbf{R}^{(i)}(t_{k-1}) = \mathbf{R}_{a,k-1}^{(i)}$. At $t = t_k$ we obtain the forecast conditional posterior
	\begin{equation}
		\bo Y_k \mid \bo X_{t_{k-1}:t_k}^{(i)} \sim
		\mathcal{N}\!\left(\boldsymbol{\mu}_{f,k}^{(i)},\,
		\mathbf{R}_{f,k}^{(i)}\right),
		\label{eq:forecast_conditional}
	\end{equation}
	where $\boldsymbol{\mu}_{f,k}^{(i)} = \boldsymbol{\mu}^{(i)}(t_k)$
	and $\mathbf{R}_{f,k}^{(i)} = \mathbf{R}^{(i)}(t_k)$.
	To generate latent trajectories consistent with the conditional distribution \eqref{eq:forecast_conditional}, we employ the martingale-based sampling procedure of~\cite{andreou2024martingale}. The resulting stochastic differential equation preserves both the conditional Gaussian marginals and the temporal correlations implied by the CGNS:
	\begin{equation}
		\label{eq:Y_froward_samp}
		\begin{aligned}
			d\bo Y^{(i)} = \, & d \boldsymbol{\mu}^{(i)} + \left(\bo a_1 - \bo R^{(i)} \bo A_1^* \left(\boldsymbol{\Sigma}_X \boldsymbol{\Sigma}_X^*\right)^{-1} \bo A_1\right)\left(\bo Y^{(i)} - \boldsymbol{\mu}^{(i)}\right) dt \\
			& + \left(\boldsymbol{\Sigma}_Y \boldsymbol{\Sigma}_Y^* + \bo R^{(i)} \bo A_1^* \left(\boldsymbol{\Sigma}_X \boldsymbol{\Sigma}_X^*\right)^{-1} \bo A_1 \bo R^{(i)}\right)^{1/2} d \bo W_Y^{(i)}.
		\end{aligned}
	\end{equation}
	Equation~\eqref{eq:Y_froward_samp} admits a natural interpretation. The first term tracks the evolution of the conditional mean, the second term provides a mean-reverting correction toward the conditional mean, and the third term injects stochastic fluctuations consistent with the conditional covariance. The second term is a mean-reverting drift that pulls $\mathbf{Y}^{(i)} $ back toward $\boldsymbol{\mu}^{(i)} $ whenever it deviates: the matrix $\bo a_1 - \bo R^{(i)} \bo A_1^* \left(\boldsymbol{\Sigma}_X \boldsymbol{\Sigma}_X^*\right)^{-1} \bo A_1$. The diffusion term $\left(\boldsymbol{\Sigma}_Y \boldsymbol{\Sigma}_Y^* + \bo R^{(i)} \bo A_1^* \left(\boldsymbol{\Sigma}_X \boldsymbol{\Sigma}_X^*\right)^{-1} \bo A_1 \bo R^{(i)}\right)^{1/2} d \bo W_Y^{(i)}$  injects the remaining stochastic fluctuation that is not explained by the mean.Together, these three terms ensure that integrating \eqref{eq:Y_froward_samp} forward  from any initial draw $\mathbf{Y}(t_0)^{(i)}  \sim \mathcal{N}(\boldsymbol{\mu}(t_0)^{(i)} , \mathbf{R}_f^{(i)} (t_0))$ produces a trajectory whose marginal distribution at every later time $t$ matches $\mathcal{N}(\boldsymbol{\mu}^{(i)} (t), \mathbf{R}^{(i)} _f(t))$, while also preserving the correct temporal correlations along the path. The forward sampling procedure provides two practical benefits. First, the conditional Gaussian structure supplies the full distribution of the unresolved variables analytically, so generating trajectories
	incurs little additional computational cost beyond evolving the conditional statistics. Second, the unresolved trajectories inherit their distribution from the conditional posterior, avoiding direct Monte Carlo exploration of the full unresolved-state space while preserving the correct temporal correlations.

	\begin{remark}[Discrete-time propagation in CGNS]
		
		A key advantage of the CGNS framework is that it lets us update the conditional posterior statistics only at the observation times, without resolving the continuous dynamics on much smaller internal time steps. That is, instead of integrating these equations with a small internal step required for integrating~\eqref{eq:forecast_X} $\delta t \ll \Delta t_n$, we propagate the conditional statistics directly from $t_n$ to $t_{n+1}$ using the state transition map. Thus, in CGNS, the natural propagation step is simply the observation interval $\Delta t_n = t_{n+1} - t_n$, avoiding the need for finer time discretization~\cite{chen2025modeling}.
	\end{remark}

	\begin{remark}[Forecast covariance evolution is conditionally deterministic]
		The forecast covariance ODE using \eqref{eq:cgns_cov} depends on $\bo X^{(i)}(t)$ only through the model coefficients $\mathbf{A}_1(\bo X^{(i)}, t)$ and $\boldsymbol{\Sigma}_Y(\bo X^{(i)}, t)$, and not through the innovation process appearing in \eqref{eq:cgns_mean}. It can therefore be integrated in parallel with \eqref{eq:forecast_X} indepndent of the assimilation process. By contrast, the mean equation \eqref{eq:cgns_mean} depends explicitly on the observation increment $d\bo X^{(i)}$ and must be integrated separately for each ensemble trajectory.
	\end{remark}
	
	\paragraph{Stage 3: Augmented forecast ensemble}
	The forecast step produces the augmented forecast ensemble
	\begin{equation}
		\left\{\mathbf{V}_{f,k}^{(i)}\right\}_{i=1}^N
		= \left\{\left(\bo X_{f,k}^{(i)},\,
		\boldsymbol{\mu}_{f,k}^{(i)}\right)\right\}_{i=1}^N,
		\label{eq:augmented_forecast}
	\end{equation}
	with forecast covariances $\{\mathbf{R}_{f,k}^{(i)}\}_{i=1}^N$ retained separately for the latent sampling stage (Section~\ref{sec:sampling}). The marginal predictive distribution of $\bo Y_k$ implied by the forecast ensemble is the Gaussian mixture
	\begin{equation}
		p\!\left(\bo Y_k \mid \mathbf{Z}_{0:k-1}\right) \approx
		\frac{1}{N}\sum_{i=1}^{N}
		\mathcal{N}\!\left(\boldsymbol{\mu}_{f,k}^{(i)},\,
		\mathbf{R}_{f,k}^{(i)}\right),
		\label{eq:forecast_mixture}
	\end{equation}
	which deploys the diversity of the forecast means $\{\boldsymbol{\mu}_{f,k}^{(i)}\}$ and the nonlinear dependence of \eqref{eq:cgns_mean} on $\bo X^{(i)}(t)$.
	
	\subsection{Analysis Step: EnKF Update on the Augmented State}
	\label{sec:enkf_analysis}

	EnKF analysis step updates the augmented forecast ensemble $\{\mathbf{V}_{f,k}^{(i)}\}_{i=1}^N$ toward the posterior as:
	\begin{align}
		\mathbf{V}_{a,k}^{(i)} &= \mathbf{V}_{f,k}^{(i)}
		+ \mathbf{K}_k\!\left(\mathbf{Z}_k
		- \mathbf{H}\mathbf{V}_{f,k}^{(i)}
		+ \boldsymbol{\epsilon}_k^{(i)}\right),
		\label{eq:enkf_update} \\
		\mathbf{K}_k &= \mathbf{P}_{f,k}{\mathbf{H}}^\top
		\!\left({\mathbf{H}}\mathbf{P}_{f,k}
		\mathbf{H}^\top +
		\mathbf{R}_{\mathrm{obs}}\right)^{-1},
		\label{eq:enkf_gain}
	\end{align}
	where $\boldsymbol{\epsilon}_k^{(i)} \sim \mathcal{N}(0,\mathbf{R}_{\mathrm{obs}})$ are independent observation perturbations introduced to maintain ensemble
	spread~\cite{burgers1998analysis}, and $\mathbf{P}_{f,k}$ is the empirical forecast covariance of the augmented ensemble:
	\begin{equation}
		\mathbf{P}_{f,k} = \frac{1}{N-1}\sum_{i=1}^{N}
		\left(\mathbf{V}_{f,k}^{(i)} - \bar{\mathbf{V}}_{f,k}\right)
		\left(\mathbf{V}_{f,k}^{(i)} - \bar{\mathbf{V}}_{f,k}\right)^\top,
		\quad \bar{\mathbf{V}}_{f,k} = \frac{1}{N}\sum_{i=1}^{N}
		\mathbf{V}_{f,k}^{(i)}.
		\label{eq:enkf_cov}
	\end{equation}
	The observation operator $\bo H$ remains unchanged since it is applied only to $\bo X^{(i)}$. The EnKF update \eqref{eq:enkf_update} yields the analysis augmented ensemble
	\begin{equation}
		\left\{\mathbf{V}_{a,k}^{(i)}\right\}_{i=1}^N =
		\left\{\left(\bo X_{a,k}^{(i)},\,
		\boldsymbol{\mu}_{a,k}^{(i)}\right)\right\}_{i=1}^N,
		\label{eq:enkf_analysis_ensemble}
	\end{equation}
	which delivers the joint posterior of the augmented state. Although the augmented-state EnKF provides a computationally efficient posterior approximation, the resulting distribution remains largely Gaussian and therefore cannot accurately capture strongly non-Gaussian statistics. The next section introduces a residual score operator that corrects this baseline posterior and drives the ensemble toward the target non-Gaussian distribution.
	
	\subsection{CGNS Surrogates for General Nonlinear Systems}
	
	The $\Omega$ framework does not require the true dynamics to exactly satisfy the CGNS structure. The role of the CGNS is to provide a tractable baseline approximation with closed-form conditional distributions for the unresolved variables. The subsequent score correction compensates for the discrepancy between this surrogate and the true posterior.
	
	When a natural CGNS representation is unavailable, several surrogate constructions are possible. A simple approach is to linearize the latent dynamics and absorb unresolved nonlinear effects into an inflation noise term. More accurate approximations can be obtained through local polynomial expansions or sensitivity-based linearizations~\cite{majda2018model}. For highly nonlinear systems, Koopman-based lifting methods such as the Conditional Gaussian Koopman Network (CGKN)~\cite{chen2025cgkn, chen2025modeling} provide a systematic framework for constructing CGNS surrogates in a higher-dimensional feature space.
	
	Theorem~\ref{thm:n2n_consistency} quantifies the effect of the surrogate approximation. A more accurate surrogate reduces the score discrepancy between the baseline and target posterior, thereby reducing the correction required from the residual operator and accelerating the convergence of the Langevin refinement.


		
		
	
	\section{$\Omega$: Residual Operator Learning for Non-Gaussian Data Assimilation}
	\label{sec:joint_score}
	
	This section introduces the $\Omega$ framework for non-Gaussian data assimilation. Figure~\ref{fig:overview} contains a summary of the methodology. The main idea is to reinterpret the EnKF-CGNS analysis step as a Gaussian score approximation and then learn a nonlinear residual score that corrects the discrepancy between this Gaussian baseline and the true posterior. The learned score is used in an annealed Langevin analysis step on the augmented state $\mathbf{V}_k=(\bo X_k,\boldsymbol{\mu}_k)$, where $\boldsymbol{\mu}_k$ is the CGNS conditional mean of the latent variable $\bo Y_k$. After the augmented state is corrected, the latent posterior is recovered analytically as a Gaussian mixture.
	
	\begin{figure}[H]
		\centering
		\includegraphics[width=\linewidth]{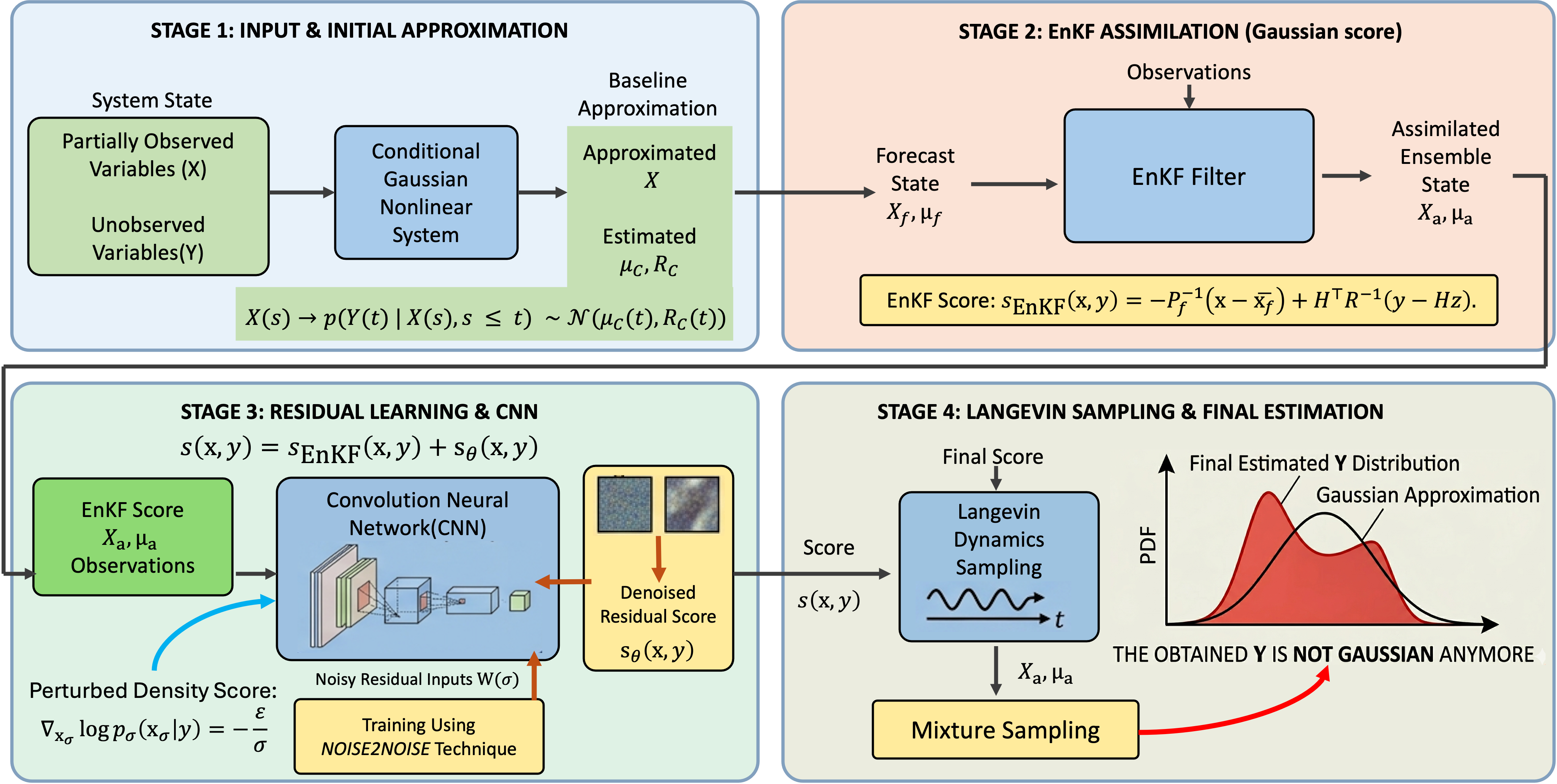}
		\caption{Overview of the $\Omega$ framework. The EnKF-CGNS forecast
			produces an augmented ensemble
			$\{\mathbf{V}_{f,k}^{(i)}\}
			=\{(\bo X_{f,k}^{(i)},\boldsymbol{\mu}_{f,k}^{(i)})\}$.
			A residual score operator $\mathbf{s}_\theta$ trained through
			Noise2Noise denoising score matching corrects the augmented state
			by annealed Langevin MCMC. The corrected conditional means
			$\{\boldsymbol{\mu}_{a,k}^{(i)}\}$ then parametrize Gaussian
			conditionals from which latent samples are drawn analytically,
			producing a non-Gaussian Gaussian-mixture posterior for $\bo Y_k$.}
		\label{fig:overview}
	\end{figure}

	\subsection{Gaussian Baseline Score from the EnKF-CGNS}
	\label{sec:gaussian_baseline}
	
	The goal of the analysis step is to correct the augmented forecast ensemble $\{\mathbf{V}_{f,k}^{(i)}\}_{i=1}^N$ toward the posterior distribution
	\begin{equation}
		p(\mathbf{V}_k \mid \mathbf{Z}_{0:k})
		\propto
		p(\mathbf{Z}_k \mid \mathbf{V}_k)\,
		p(\mathbf{V}_k \mid \mathbf{Z}_{0:k-1}).
		\label{eq:augmented_bayes}
	\end{equation}
	Rather than viewing the EnKF update only as a covariance-based correction, we reinterpret it as an approximation of the posterior score on the augmented state. Taking the gradient of the logarithm of \eqref{eq:augmented_bayes} gives
	\begin{equation}
		\underbrace{\nabla_{\mathbf{V}_k} \log p\!\left(\mathbf{V}_k
			\mid \mathbf{Z}_{0:k}\right)}_{\mathbf{s}_{\mathrm{EnKF},k}
			(\mathbf{V}_k)}
		= \underbrace{\nabla_{\mathbf{V}_k} \log p\!\left(\mathbf{V}_k
			\mid \mathbf{Z}_{0:k-1}\right)}_{\mathbf{s}_{\mathrm{prior},k}
			(\mathbf{V}_k)}
		+ \underbrace{\nabla_{\mathbf{V}_k} \log p\!\left(\mathbf{Z}_k
			\mid \mathbf{V}_k\right)}_{\mathbf{s}_{\mathrm{lik},k}
			(\mathbf{V}_k)}.
		\label{eq:score_decomp}
	\end{equation}
	
	Since observations act only on $\bo X_k$, the likelihood score on the augmented state $\mathbf{V}_k=(\bo X_k,\boldsymbol{\mu}_k)$ is
	\begin{equation}
		\mathbf{s}_{\mathrm{lik},k}(\mathbf{V}_k)
		=
		\begin{pmatrix}
			\mathbf{H}_X^\top \mathbf{R}_{\mathrm{obs}}^{-1}
			(\mathbf{Z}_k-\mathbf{H}_X\bo X_k)
			\\
			\mathbf{0}
		\end{pmatrix}.
		\label{eq:lik_score}
	\end{equation}
	The EnKF-CGNS forecast approximates the prior distribution of the augmented state by
	\begin{equation}
		p(\mathbf{V}_k \mid \mathbf{Z}_{0:k-1})
		\approx
		\mathcal{N}(\bar{\mathbf{V}}_{f,k},\mathbf{P}_{f,k}),
		\label{eq:forecast_gaussian_aug}
	\end{equation}
	where $\bar{\mathbf{V}}_{f,k}$ and $\mathbf{P}_{f,k}$ are the empirical forecast mean and covariance. Therefore the Gaussian baseline posterior score is
	\begin{equation}
		\mathbf{s}_{\mathrm{EnKF},k}(\mathbf{V}_k)
		=
		-\mathbf{P}_{f,k}^{-1}
		(\mathbf{V}_k-\bar{\mathbf{V}}_{f,k})
		+
		\mathbf{H}^\top \mathbf{R}_{\mathrm{obs}}^{-1}
		(\mathbf{Z}_k-\mathbf{H}\mathbf{V}_k),
		\label{eq:enkf_score}
	\end{equation}
	where $\mathbf{H}
	=
	\begin{pmatrix}
		\mathbf{H}_X & \mathbf{0}
	\end{pmatrix}$.
	Writing the forecast covariance blockwise as
	\begin{equation}
		\mathbf{P}_{f,k}
		=
		\begin{pmatrix}
			\mathbf{P}_{f,k}^{XX}
			&
			\mathbf{P}_{f,k}^{X\mu}
			\\
			\mathbf{P}_{f,k}^{\mu X}
			&
			\mathbf{P}_{f,k}^{\mu\mu}
		\end{pmatrix},
		\label{eq:P_blockwise}
	\end{equation}
	shows that the Gaussian baseline propagates observational information into $\boldsymbol{\mu}_k$ only through the linear cross-covariance blocks $\mathbf{P}_{f,k}^{X\mu}$ and $\mathbf{P}_{f,k}^{\mu X}$. This is the key limitation that motivates the nonlinear residual operator below.

	\subsection{Residual Operator Model}
	\label{sec:residual_model}
	
	When the true posterior is non-Gaussian, the affine EnKF score \eqref{eq:enkf_score} cannot represent higher-order posterior corrections. The role of the residual operator is therefore not to replace the EnKF-CGNS baseline, but to learn the missing nonlinear score component. We decompose the posterior score as
	\begin{equation}
		\mathbf{s}_{\mathrm{post},k}(\mathbf{V}_k)
		\approx
		\mathbf{s}_{\mathrm{EnKF},k}(\mathbf{V}_k)
		+
		\mathbf{s}_\theta(\mathbf{V}_k,\sigma,\mathcal{C}_k),
		\label{eq:score_total}
	\end{equation}
	where $\sigma>0$ is the score-matching noise level and
	$\mathbf{s}_\theta$ is a neural residual score operator.
	
	The residual operator is conditioned on the local filtering context
	\begin{equation}
		\mathcal{C}_k
		=
		\bigl(
		\bar{\mathbf{V}}_{f,k},
		\mathrm{diag}(\mathbf{P}_{f,k}),
		\mathbf{Z}_k
		\bigr).
		\label{eq:context}
	\end{equation}
	More generally, $\mathcal{C}_k$ may include additional summary statistics of the forecast ensemble, such as selected covariance entries or low-dimensional features of $\mathbf{P}_{f,k}$. The residual score acts on the full augmented state and is decomposed as
	\begin{equation}
		\mathbf{s}_\theta(\mathbf{V}_k,\sigma,\mathcal{C}_k)
		=
		\begin{pmatrix}
			\mathbf{s}_\theta^X(\mathbf{V}_k,\sigma,\mathcal{C}_k)
			\\
			\mathbf{s}_\theta^\mu(\mathbf{V}_k,\sigma,\mathcal{C}_k)
		\end{pmatrix}.
		\label{eq:score_components}
	\end{equation}
	The component $\mathbf{s}_\theta^X$ learns non-Gaussian corrections in the observed subspace. The component $\mathbf{s}_\theta^\mu$ is more important: it corrects the CGNS conditional mean beyond what can be captured by the linear cross-covariance in the EnKF baseline. Thus, $\Omega$ propagates non-Gaussian observational information into the latent conditional distribution through the corrected means $\boldsymbol{\mu}_{a,k}^{(i)}$.

	\subsection{Unsupervised Training via Noise2Noise Denoising Score Matching}
	\label{sec:training}
	
	The residual operator is trained without access to samples from the true posterior. Instead, we adopt denoising score matching (DSM)~\cite{vincent2011connection} on the augmented state $\mathbf{V}_k=(\bo X_k,\boldsymbol{\mu}_k)$. For a noise level $\sigma>0$, define the perturbed augmented state
	\begin{equation}
		\tilde{\mathbf{V}}_k
		=
		\mathbf{V}_k
		+
		\sigma\boldsymbol{\epsilon},
		\qquad
		\boldsymbol{\epsilon}
		\sim
		\mathcal{N}
		\left(
		\mathbf{0},
		\mathbf{I}_{d_X+d_Y}
		\right),
		\label{eq:perturbed_state}
	\end{equation}
	where $\mathbf{V}_k$ is drawn from the forecast augmented ensemble \eqref{eq:augmented_forecast}. The perturbation kernel is $p_\sigma(\tilde{\mathbf{V}}_k\mid\mathbf{V}_k)
	=
	\mathcal{N}
	\left(
	\mathbf{V}_k,
	\sigma^2\mathbf{I}
	\right)$,
	whose score is
	\begin{equation}
		\nabla_{\tilde{\mathbf{V}}_k}
		\log
		p_\sigma
		(\tilde{\mathbf{V}}_k\mid\mathbf{V}_k)
		=
		-\frac{\boldsymbol{\epsilon}}{\sigma}.
		\label{eq:kernel_score}
	\end{equation}
	Following the standard DSM formulation, the residual operator is
	trained by minimizing
	\begin{equation}
		\mathcal{L}(\theta)
		=
		\mathbb{E}_{\mathbf{V}_k,\boldsymbol{\epsilon},\sigma}
		\left[
		\sigma^2
		\left\|
		\mathbf{s}_{\mathrm{EnKF},k}
		(\tilde{\mathbf{V}}_k)
		+
		\mathbf{s}_\theta
		(\tilde{\mathbf{V}}_k,\sigma,\mathcal{C}_k)
		+
		\frac{\boldsymbol{\epsilon}}{\sigma}
		\right\|^2
		\right].
		\label{eq:dsm_loss}
	\end{equation}
	The expectation is taken over ensemble samples $\mathbf{V}_k$, Gaussian perturbations $\boldsymbol{\epsilon}$, and noise levels $\sigma\sim p(\sigma)$. The factor $\sigma^2$ follows the noise-conditioned score network (NCSN) construction~\cite{song2020improved} and balances contributions from different noise scales. In practice, we add a small regularization term $\lambda_r\|\mathbf{s}_\theta\|^2$ to stabilize training and prevent excessive score magnitudes.
	
	Theorem~\ref{thm:dsm_consistency} establishes the consistency of \eqref{eq:dsm_loss} in an idealized setting where samples from the true posterior are available. The theorem provides the foundation for the subsequent Noise2Noise analysis, where training is performed using the EnKF-CGNS surrogate rather than the true posterior.
	
	\begin{theorem}[Ideal DSM Consistency on the Augmented State]
		\label{thm:dsm_consistency}
		
		Let
		$p_{\mathrm{aug}}
		(\mathbf{V}_k\mid\mathbf{Z}_{0:k})$
		be a continuously differentiable density on
		$\mathbb{R}^{d_X+d_Y}$ with finite Fisher information
		\begin{equation}
			\mathcal{I}(p_{\mathrm{aug}})
			:=
			\mathbb{E}_{p_{\mathrm{aug}}}
			\left[
			\left\|
			\nabla_{\mathbf{V}_k}
			\log
			p_{\mathrm{aug}}
			(\mathbf{V}_k\mid\mathbf{Z}_{0:k})
			\right\|^2
			\right]
			<
			\infty.
			\label{eq:fisher_info}
		\end{equation}
		Let
		\[
		p_\sigma(\tilde{\mathbf{V}}_k\mid\mathbf{V}_k)
		=
		\mathcal{N}
		(\mathbf{V}_k,\sigma^2\mathbf{I})
		\]
		and define the smoothed marginal
		\begin{equation}
			p_\sigma(\tilde{\mathbf{V}}_k)
			=
			\int
			p_\sigma
			(\tilde{\mathbf{V}}_k\mid\mathbf{V}_k)
			\,
			p_{\mathrm{aug}}
			(\mathbf{V}_k\mid\mathbf{Z}_{0:k})
			\,
			d\mathbf{V}_k.
			\label{eq:smoothed_marginal}
		\end{equation}
		Suppose
		\[
		\mathbf{s}_{\mathrm{EnKF},k}
		+
		\mathbf{s}_\theta
		\in
		L^2(p_\sigma).
		\]
		Then:
		
		\begin{enumerate}[(i)]
			
			\item The DSM objective satisfies
			\begin{equation}
				\mathcal{L}(\theta)
				=
				\mathbb{E}_{p_\sigma}
				\left[
				\sigma^2
				\left\|
				\mathbf{s}_{\mathrm{EnKF},k}
				+
				\mathbf{s}_\theta
				-
				\nabla\log p_\sigma
				\right\|^2
				\right]
				+
				C,
				\label{eq:dsm_identity}
			\end{equation}
			where
			\begin{equation}
				C
				=
				\mathbb{E}_{p_\sigma}
				\left[
				\sigma^2
				\|\nabla\log p_\sigma\|^2
				\right]
				\label{eq:dsm_constant}
			\end{equation}
			is finite and independent of $\theta$.
			
			\item The unique minimizer satisfies
			\begin{equation}
				\mathbf{s}_{\mathrm{EnKF},k}
				+
				\mathbf{s}_{\theta^*}
				=
				\nabla\log p_\sigma
				\qquad
				p_\sigma\text{-a.e.}
				\label{eq:dsm_minimizer}
			\end{equation}
			
			\item As $\sigma\to0$,
			\begin{equation}
				\nabla\log p_\sigma
				\to
				\nabla\log p_{\mathrm{aug}}
				(\cdot\mid\mathbf{Z}_{0:k})
				\quad
				\text{in }L^2(p_{\mathrm{aug}}),
				\label{eq:score_convergence}
			\end{equation}
			and consequently
			\begin{equation}
				\mathbf{s}_{\theta^*}
				\to
				\nabla\log p_{\mathrm{aug}}
				(\cdot\mid\mathbf{Z}_{0:k})
				-
				\mathbf{s}_{\mathrm{EnKF},k}
				\quad
				\text{in }L^2(p_{\mathrm{aug}}).
				\label{eq:residual_convergence}
			\end{equation}
			
		\end{enumerate}
	\end{theorem}
	
	\begin{proof}[Proof sketch]
		The key step is Stein's identity applied to the perturbed state
		\[
		\tilde{\mathbf{V}}_k
		=
		\mathbf{V}_k
		+
		\sigma\boldsymbol{\epsilon}.
		\]
		Conditioning on $\tilde{\mathbf{V}}_k$ gives
		\[
		\mathbb{E}
		\left[
		-\frac{\boldsymbol{\epsilon}}{\sigma}
		\,\middle|\,
		\tilde{\mathbf{V}}_k
		\right]
		=
		\nabla\log p_\sigma(\tilde{\mathbf{V}}_k),
		\]
		which converts the DSM objective into an
		$L^2(p_\sigma)$ distance between the total learned score and the
		score of the smoothed posterior. The minimizer therefore coincides
		with the posterior score of the smoothed distribution. The small-noise
		limit follows from the continuity of Gaussian smoothing under finite
		Fisher information. The full proof is given in~\ref{app:proof_dsm}.
	\end{proof}
	
	Theorem~\ref{thm:dsm_consistency} shows that DSM provides a rigorous surrogate for learning the posterior score on the augmented state. The training objective is equivalent to matching the score of the Gaussian-smoothed posterior, the minimizer is unique, and the learned residual converges to the true non-Gaussian correction as the noise level tends to zero. This ideal result serves as the foundation for the Noise2Noise analysis in Theorem~\ref{thm:n2n_consistency}, where the true posterior samples are replaced by samples from the EnKF-CGNS surrogate.

	\begin{theorem}[Noise2Noise Surrogate Consistency]
		\label{thm:n2n_consistency}
		
		Let $p_{\mathrm{aug}}(\mathbf{V}_k\mid\mathbf{Z}_{0:k})$ denote the true augmented posterior, and let $q_{\mathrm{aug}}^{(N)} (\mathbf{V}_k\mid\mathbf{Z}_{0:k})$ denote the surrogate distribution induced by the EnKF-CGNS forecast ensemble \eqref{eq:augmented_forecast}. We model this surrogate as the empirical Gaussian mixture
		\begin{equation}
			q_{\mathrm{aug}}^{(N)}
			(\mathbf{V}_k\mid\mathbf{Z}_{0:k})
			=
			\frac{1}{N}
			\sum_{i=1}^{N}
			\mathcal{N}
			\left(
			\mathbf{V}_{f,k}^{(i)},
			\mathbf{P}_{f,k}
			\right),
			\label{eq:surrogate_dist}
		\end{equation}
		where $\mathbf{P}_{f,k}$ is the shared empirical forecast covariance \eqref{eq:enkf_cov}. Define the Gaussian-smoothed distributions
		\begin{align}
			p_\sigma(\tilde{\mathbf{V}}_k)
			&=
			\int
			\mathcal{N}
			\left(
			\tilde{\mathbf{V}}_k;
			\mathbf{V}_k,
			\sigma^2\mathbf{I}
			\right)
			p_{\mathrm{aug}}
			(\mathbf{V}_k\mid\mathbf{Z}_{0:k})
			d\mathbf{V}_k,
			\label{eq:true_smoothed}
			\\
			q_\sigma^{(N)}(\tilde{\mathbf{V}}_k)
			&=
			\int
			\mathcal{N}
			\left(
			\tilde{\mathbf{V}}_k;
			\mathbf{V}_k,
			\sigma^2\mathbf{I}
			\right)
			q_{\mathrm{aug}}^{(N)}
			(\mathbf{V}_k\mid\mathbf{Z}_{0:k})
			d\mathbf{V}_k.
			\label{eq:surrogate_smoothed}
		\end{align}
		Assume that $p_{\mathrm{aug}}$ and $q_{\mathrm{aug}}^{(N)}$ have finite Fisher information and that
		\[
		\mathbf{s}_{\mathrm{EnKF},k}
		+
		\mathbf{s}_\theta
		\in
		L^2(q_\sigma^{(N)}).
		\]
		For $\sigma>0$, both $p_\sigma$ and $q_\sigma^{(N)}$ are strictly positive on $\mathbb{R}^{d_X+d_Y}$ and are mutually absolutely continuous with respect to Lebesgue measure.
		Let
		\begin{equation}
			p_{\mathrm{CGNS}}
			(\mathbf{V}_k\mid\mathbf{Z}_{0:k})
			:=
			\lim_{N\to\infty}
			q_{\mathrm{aug}}^{(N)}
			(\mathbf{V}_k\mid\mathbf{Z}_{0:k})
			\label{eq:cgns_limit}
		\end{equation}
		denote the infinite-ensemble CGNS surrogate posterior, and let
		$p_{\mathrm{CGNS},\sigma}$ denote its Gaussian smoothing. When the true
		dynamics do not exactly satisfy the CGNS structure
		\eqref{eq:cgns_obs}--\eqref{eq:cgns_lat}, define the CGNS
		misspecification error
		\begin{equation}
			\varepsilon_{\mathrm{CGNS}}
			:=
			\left\|
			\nabla\log p_{\mathrm{CGNS},\sigma}
			-
			\nabla\log p_\sigma
			\right\|_{L^2(p_{\mathrm{aug}})}.
			\label{eq:cgns_misspec}
		\end{equation}
		Then the following statements hold.
		
		\begin{enumerate}[(i)]
			
			\item \textbf{Surrogate DSM identity.}
			The DSM objective trained on surrogate samples
			$\mathbf{V}_k\sim q_{\mathrm{aug}}^{(N)}$ satisfies
			\begin{equation}
				\mathcal{L}_q(\theta)
				=
				\mathbb{E}_{q_\sigma^{(N)}}
				\left[
				\sigma^2
				\left\|
				\mathbf{s}_{\mathrm{EnKF},k}(\tilde{\mathbf{V}}_k)
				+
				\mathbf{s}_\theta(\tilde{\mathbf{V}}_k)
				-
				\nabla_{\tilde{\mathbf{V}}_k}
				\log q_\sigma^{(N)}(\tilde{\mathbf{V}}_k)
				\right\|^2
				\right]
				+
				C_q,
				\label{eq:n2n_identity}
			\end{equation}
			where
			\begin{equation}
				C_q
				=
				\mathbb{E}_{q_\sigma^{(N)}}
				\left[
				\sigma^2
				\left\|
				\nabla\log q_\sigma^{(N)}(\tilde{\mathbf{V}}_k)
				\right\|^2
				\right]
				\label{eq:n2n_constant}
			\end{equation}
			is independent of $\theta$.
			
			\item \textbf{Minimizer characterization.}
			The unique minimizer $\theta^*$ of
			$\mathcal{L}_q(\theta)$ over $L^2(q_\sigma^{(N)})$ satisfies
			\begin{equation}
				\mathbf{s}_{\mathrm{EnKF},k}
				(\tilde{\mathbf{V}}_k)
				+
				\mathbf{s}_{\theta^*}
				(\tilde{\mathbf{V}}_k)
				=
				\nabla_{\tilde{\mathbf{V}}_k}
				\log q_\sigma^{(N)}(\tilde{\mathbf{V}}_k)
				\qquad
				q_\sigma^{(N)}\text{-a.e.}
				\label{eq:n2n_minimizer}
			\end{equation}
			
			\item \textbf{Two-source error decomposition.}
			Define the finite-ensemble sampling error
			\begin{equation}
				\varepsilon_{\mathrm{samp}}(N)
				:=
				\left\|
				\nabla\log q_\sigma^{(N)}
				-
				\nabla\log p_{\mathrm{CGNS},\sigma}
				\right\|_{L^2(p_{\mathrm{aug}})}.
				\label{eq:samp_error}
			\end{equation}
			Then the approximation error of the learned residual satisfies
			\begin{equation}
				\left\|
				\mathbf{s}_{\theta^*}
				-
				\left(
				\nabla\log p_\sigma
				-
				\mathbf{s}_{\mathrm{EnKF},k}
				\right)
				\right\|_{L^2(p_{\mathrm{aug}})}
				\leq
				\underbrace{
					\varepsilon_{\mathrm{samp}}(N)
				}_{\text{finite-ensemble sampling}}
				+
				\underbrace{
					\varepsilon_{\mathrm{CGNS}}
				}_{\text{structural misspecification}} .
				\label{eq:surrogate_error}
			\end{equation}
			
			\item \textbf{Elimination of sampling error as $N\to\infty$.}
			Under standard EnKF consistency conditions,
			$q_{\mathrm{aug}}^{(N)}$ converges to
			$p_{\mathrm{CGNS}}$ as $N\to\infty$. Under the regularity assumptions
			used in~\ref{app:n2n_proof}, this implies
			\begin{equation}
				\varepsilon_{\mathrm{samp}}(N)
				\to
				0
				\qquad
				\text{as }N\to\infty.
				\label{eq:sampling_error_vanish}
			\end{equation}
			Consequently,
			\begin{equation}
				\lim_{N\to\infty}
				\left\|
				\mathbf{s}_{\theta^*}
				-
				\left(
				\nabla\log p_\sigma
				-
				\mathbf{s}_{\mathrm{EnKF},k}
				\right)
				\right\|_{L^2(p_{\mathrm{aug}})}
				\leq
				\varepsilon_{\mathrm{CGNS}}.
				\label{eq:n2n_convergence}
			\end{equation}
			Thus the infinite-ensemble residual error is determined entirely by
			the structural mismatch between the CGNS surrogate and the true
			posterior.
			
			\item \textbf{Full convergence when CGNS is exact.}
			If the true system dynamics exactly satisfy the CGNS structure
			\eqref{eq:cgns_obs}--\eqref{eq:cgns_lat}, then
			$p_{\mathrm{CGNS}}=p_{\mathrm{aug}}$ and
			$\varepsilon_{\mathrm{CGNS}}=0$. Therefore
			\begin{equation}
				\mathbf{s}_{\theta^*}
				\to
				\nabla
				\log p_{\mathrm{aug}}
				(\cdot\mid\mathbf{Z}_{0:k})
				-
				\mathbf{s}_{\mathrm{EnKF},k}
				\qquad
				\text{in }L^2(p_{\mathrm{aug}})
				\label{eq:full_conv}
			\end{equation}
			as $N\to\infty$ and $\sigma\to0$.
			
		\end{enumerate}
	\end{theorem}
	
	Theorem~\ref{thm:n2n_consistency} addresses the practical reality that
	training samples are drawn from the EnKF--CGNS surrogate
	$q_{\mathrm{aug}}^{(N)}$ rather than from the unknown true posterior
	$p_{\mathrm{aug}}$. It shows that the Noise2Noise setup is still
	principled: the DSM objective trained on surrogate samples recovers
	the surrogate score exactly, and the discrepancy from the true
	posterior score decomposes into two independent components. The first
	is a finite-ensemble sampling error that vanishes as $N$ grows. The
	second is a structural CGNS misspecification error that persists even
	for infinite ensembles. Therefore, improving $\Omega$ requires
	identifying which error source dominates: increasing $N$ reduces the
	sampling error, while improving the CGNS surrogate reduces the
	irreducible misspecification floor.
	
	\begin{proof}[Proof sketch]
		Parts (i) and (ii) follow by repeating the proof of
		Theorem~\ref{thm:dsm_consistency} with
		$p_{\mathrm{aug}}$ replaced by
		$q_{\mathrm{aug}}^{(N)}$. The Stein identity applies because the
		Gaussian-smoothed surrogate density is continuously differentiable,
		strictly positive, and has finite Fisher information.
		
		For Part (iii), the minimizer identity gives
		\[
		\mathbf{s}_{\theta^*}
		=
		\nabla\log q_\sigma^{(N)}
		-
		\mathbf{s}_{\mathrm{EnKF},k}
		\qquad
		q_\sigma^{(N)}\text{-a.e.}
		\]
		Since $\sigma>0$, the Gaussian-smoothed densities are strictly positive and mutually absolutely continuous, so the identity can be compared on the support relevant to $p_{\mathrm{aug}}$. Subtracting the target residual $\nabla\log p_\sigma-\mathbf{s}_{\mathrm{EnKF},k}$ and inserting the intermediate score $\nabla\log p_{\mathrm{CGNS},\sigma}$ gives
		\[
		\begin{aligned}
			&
			\left\|
			\mathbf{s}_{\theta^*}
			-
			\left(
			\nabla\log p_\sigma
			-
			\mathbf{s}_{\mathrm{EnKF},k}
			\right)
			\right\|_{L^2(p_{\mathrm{aug}})}
			\\
			&\qquad
			=
			\left\|
			\nabla\log q_\sigma^{(N)}
			-
			\nabla\log p_\sigma
			\right\|_{L^2(p_{\mathrm{aug}})}
			\\
			&\qquad
			\leq
			\left\|
			\nabla\log q_\sigma^{(N)}
			-
			\nabla\log p_{\mathrm{CGNS},\sigma}
			\right\|_{L^2(p_{\mathrm{aug}})}
			+
			\left\|
			\nabla\log p_{\mathrm{CGNS},\sigma}
			-
			\nabla\log p_\sigma
			\right\|_{L^2(p_{\mathrm{aug}})}.
		\end{aligned}
		\]
		This proves the two-source decomposition
		\eqref{eq:surrogate_error}.
		
		For Part (iv), EnKF consistency gives convergence of $q_{\mathrm{aug}}^{(N)}$ toward $p_{\mathrm{CGNS}}$. Gaussian smoothing preserves this convergence and the regularity assumptions in~\ref{app:n2n_proof} allow the convergence to be transferred to the corresponding smoothed scores. Hence $\varepsilon_{\mathrm{samp}}(N)\to0$. Part (v) follows immediately: if $p_{\mathrm{CGNS}}=p_{\mathrm{aug}}$, then $\varepsilon_{\mathrm{CGNS}}=0$, so the learned residual converges to the true posterior correction. The full proof is given in~\ref{app:n2n_proof}.
	\end{proof}

	\begin{remark}[Mutual absolute continuity for $\sigma>0$]
		\label{rem:abs_cont}
		
		The mutual absolute continuity assumption used in Theorem~\ref{thm:n2n_consistency} is automatically satisfied for any $\sigma>0$. Indeed, both $p_\sigma$ and $q_\sigma$ are obtained by convolving their respective base measures with the Gaussian kernel $\mathcal N(0,\sigma^2\mathbf I)$, which has full support on $\mathbb R^{d_X+d_Y}$. Consequently, both smoothed distributions admit strictly positive densities on the entire state space and therefore share the same Lebesgue-null sets.
		
		It follows that $p_\sigma \sim q_\sigma$, that is, the two measures are mutually absolutely continuous. Therefore any identity that holds $q_\sigma$-almost everywhere, including the minimizer identity \eqref{eq:n2n_minimizer}, also holds $p_\sigma$-almost everywhere. This allows the learned surrogate score to be compared directly with the score of the smoothed target posterior.
	\end{remark}

	\begin{remark}[Two-source error decomposition: interpretation and diagnostics]
		\label{rem:error_interpretation}
		
		The error bound \eqref{eq:surrogate_error} separates two
		fundamentally different sources of approximation error, each with a
		different practical remedy.
		
		\medskip
		
		\begin{itemize}
			
			\item
			\textbf{Sampling error $\varepsilon_{\mathrm{samp}}(N)$.} This term arises from representing the CGNS surrogate posterior using
			a finite ensemble. By Theorem~\ref{thm:n2n_consistency}(iv),
			\[
			\varepsilon_{\mathrm{samp}}(N)\to0
			\qquad
			\text{as }
			N\to\infty.
			\]
			In practice, finite-ensemble effects enter both the surrogate distribution and the Monte Carlo approximation of the DSM objective. They can therefore be reduced by increasing the ensemble size or by improving the sampling accuracy of the training procedure.
			
			\medskip
			
			\item
			\textbf{CGNS misspecification error
				$\varepsilon_{\mathrm{CGNS}}$.}
			This term arises when the true dynamics do not exactly satisfy the CGNS structure \eqref{eq:cgns_obs}--\eqref{eq:cgns_lat}. Unlike the sampling error, it does not vanish as
			$N\to\infty$ because it reflects a structural limitation of the surrogate model itself.
			
			The residual operator $\mathbf s_\theta$ is trained to recover the surrogate score $\nabla\log q_\sigma$. Consequently, it can correct finite-ensemble effects and nonlinear
			posterior structure represented by the surrogate, but it cannot remove the discrepancy between the surrogate posterior and the true posterior. When $\varepsilon_{\mathrm{CGNS}}$ is large, improving the CGNS surrogate directly is the only way to reduce the irreducible approximation floor.
			
		\end{itemize}
		
		Although $\varepsilon_{\mathrm{CGNS}}$ is not directly observable, the innovation
		\begin{equation}
			\boldsymbol{\nu}_k
			:=
			\mathbf z_k
			-
			\mathbf H
			\bar{\mathbf V}_{f,k}
			\label{eq:innovation}
		\end{equation}
		provides a useful diagnostic. Under a well-specified CGNS surrogate and an approximately Gaussian forecast,
		\[
		\mathbb E[\boldsymbol{\nu}_k]
		=
		0,
		\qquad
		\mathrm{Cov}(\boldsymbol{\nu}_k)
		=
		\mathbf H
		\mathbf P_{f,k}
		\mathbf H^\top
		+
		\mathbf R_{\mathrm{obs}}.
		\]
		Persistent innovation bias or covariance inflation therefore suggests that surrogate misspecification, rather than finite-ensemble error, is the dominant limitation.
	\end{remark}

	\begin{remark}[Component-wise stabilization]
		\label{rem:stabilization}
		
		Learning the residual correction in the $\boldsymbol{\mu}$-subspace is generally more difficult than in the observed $\bo X$-subspace. The observed component $\bo X_k$ receives information directly from the likelihood term, whereas the conditional mean $\boldsymbol{\mu}_k$ is influenced only indirectly through the forecast cross-covariance $\mathbf{P}_{f,k}^{\mu X}$. Consequently, the signal-to-noise ratio of the residual correction is typically weaker in the $\boldsymbol{\mu}$-subspace, making the corresponding score estimates more sensitive to sampling error and optimization noise.
		
		To improve training stability, we employ a weighted component-wise DSM objective:
		\begin{equation}
			\mathcal{L}_{\mathrm w}(\theta)
			=
			\mathbb{E}
			\left[
			\sigma^2
			\left\|
			\mathbf{s}_{\mathrm{tot},k}^{X}
			(\tilde{\mathbf V}_k)
			+
			\frac{\boldsymbol{\epsilon}_X}{\sigma}
			\right\|^2
			+
			\lambda
			\sigma^2
			\left\|
			\mathbf{s}_{\mathrm{tot},k}^{\boldsymbol{\mu}}
			(\tilde{\mathbf V}_k)
			+
			\frac{\boldsymbol{\epsilon}_{\boldsymbol{\mu}}}{\sigma}
			\right\|^2
			\right],
			\label{eq:weighted_loss}
		\end{equation}
		where
		\[
		\boldsymbol{\epsilon}
		=
		(\boldsymbol{\epsilon}_X,
		\boldsymbol{\epsilon}_{\boldsymbol{\mu}})
		\]
		denotes the decomposition of the perturbation noise into observed and
		conditional-mean components,
		\[
		\mathbf{s}_{\mathrm{tot},k}
		=
		\mathbf{s}_{\mathrm{EnKF},k}
		+
		\mathbf{s}_{\theta},
		\]
		and $\lambda\in(0,1]$ controls the relative strength of the $\boldsymbol{\mu}$-subspace correction.
		
		This weighting reflects the design philosophy of $\Omega$. The EnKF-CGNS baseline already provides a stable first-order update for $\boldsymbol{\mu}_k$ through the cross-covariance $\mathbf{P}_{f,k}^{\mu X}$. The role of the residual operator is therefore not to replace this baseline correction, but to refine it by learning the remaining
		non-Gaussian component of the posterior score. Taking $\lambda<1$ encourages the neural residual to act as a higher-order correction while preserving the robustness of the Gaussian baseline during the early stages of training.
	\end{remark}

	\subsection{Annealed Langevin Analysis on the Augmented State}
	\label{sec:sampling}

	Given the trained residual operator $\mathbf{s}_\theta$, the analysis augmented ensemble is generated by annealed Langevin dynamics on the augmented state
	$\mathbf V_k
	=
	(\bo X_k,\boldsymbol{\mu}_k)$, initialized from the forecast ensemble $\{\mathbf V_{f,k}^{(i)}\}_{i=1}^{N}$.
	
	A key design choice in $\Omega$ is that the score correction is performed on the augmented state $(\bo X_k,\boldsymbol{\mu}_k)$ rather than on the full state $(\bo X_k,\bo Y_k)$.
	The variable $\boldsymbol{\mu}_k$ parametrizes the CGNS conditional posterior of the latent state. Correcting $\boldsymbol{\mu}_k$ therefore modifies the entire conditional distribution of $\bo Y_k$ while preserving the closed-form Gaussian structure of the CGNS representation.
	
	For each ensemble member, the Langevin update at iteration $\ell$ is given by
	\begin{equation}
		\mathbf V_k^{(i,\ell+1)}
		=
		\mathbf V_k^{(i,\ell)}
		+
		\frac{\alpha_\ell}{2}
		\left[
		\mathbf s_{\mathrm{EnKF},k}
		\!\left(
		\mathbf V_k^{(i,\ell)}
		\right)
		+
		\mathbf s_\theta
		\!\left(
		\mathbf V_k^{(i,\ell)},
		\sigma_\ell,
		\mathcal C_k
		\right)
		\right]
		+
		\sqrt{\alpha_\ell}
		\,
		\boldsymbol{\eta}_\ell^{(i)},
		\qquad
		\boldsymbol{\eta}_\ell^{(i)}
		\sim
		\mathcal N(0,\mathbf I),
		\label{eq:langevin}
	\end{equation}
	where $\alpha_\ell>0$ is the step-size schedule and $\sigma_\ell\searrow0$ is the annealed noise level.
	
	Under the assumptions of Theorem~\ref{thm:langevin_surrogate}, the resulting chain converges to the smoothed surrogate posterior $q_\sigma$. In the exact-CGNS and infinite-ensemble limits discussed in Theorems~\ref{thm:n2n_consistency} and \ref{thm:langevin_surrogate}, this target converges further to the true augmented posterior
	$p_{\mathrm{aug}} (\mathbf V_k\mid\mathbf Z_{0:k})$.
	
	After $L$ Langevin iterations, the analysis augmented ensemble is
	\begin{equation}
		\left\{
		\mathbf V_{a,k}^{(i)}
		\right\}_{i=1}^{N}
		=
		\left\{
		\left(
		\bo X_{a,k}^{(i)},
		\boldsymbol{\mu}_{a,k}^{(i)}
		\right)
		\right\}_{i=1}^{N}.
		\label{eq:analysis_augmented}
	\end{equation}
	
	\begin{remark}[Why $\bo Y$ is never updated by Langevin]
		\label{rem:no_langevin_y}
		
		The latent variable $\bo Y_k$ is never evolved directly by the Langevin chain. Instead, the chain updates only the augmented state $(\bo X_k,\boldsymbol{\mu}_k)$, after which $\bo Y_k$ is recovered analytically from the CGNS conditional posterior.
		
		This strategy exploits the conditional Gaussian structure of the CGNS representation. Direct Langevin sampling in the full space $(\bo X_k,\bo Y_k)$ would require exploring the potentially high-dimensional latent space and would destroy the closed-form conditional structure. By contrast, updating only $(\bo X_k,\boldsymbol{\mu}_k)$ allows the nonlinear score correction to be applied where it is most informative while retaining an analytic representation of latent uncertainty.
		
		The resulting Gaussian-mixture reconstruction therefore combines the flexibility of score-based sampling with the dimensional scalability of conditional Gaussian filtering.
	\end{remark}

	\begin{theorem}[Langevin Convergence under Learned Surrogate Score]
		\label{thm:langevin_surrogate}
		
		Let $q_\sigma(\mathbf{V}_k)$ denote the Gaussian-smoothed surrogate
		augmented posterior defined in Theorem~\ref{thm:n2n_consistency},
		with $\sigma>0$ fixed. Assume:
		\begin{enumerate}[(i)]
			\item $q_\sigma$ admits a continuously differentiable, strictly positive density on $\mathbb{R}^{d_X+d_Y}$;
			
			\item the score $\nabla_{\mathbf{V}_k}\log q_\sigma(\mathbf{V}_k)$ is $L$-Lipschitz for some $L>0$;
			
			\item the potential $U_\sigma(\mathbf{V})
			:=-\log q_\sigma(\mathbf{V})$ satisfies the dissipativity condition
			\begin{equation}
				\bigl\langle
				\nabla U_\sigma(\mathbf{V}),\mathbf{V}
				\bigr\rangle
				\geq
				m\|\mathbf{V}\|^2-b,
				\qquad
				m>0,\quad b\geq0,
				\label{eq:dissipativity}
			\end{equation}
			and the associated Langevin diffusion is geometrically ergodic;
			
			\item the learned residual operator $\mathbf{s}_{\theta^*}$     satisfies the Noise2Noise minimizer identity
			\begin{equation}
				\mathbf{s}_{\mathrm{EnKF},k}(\mathbf{V}_k)
				+
				\mathbf{s}_{\theta^*}(\mathbf{V}_k)
				=
				\nabla_{\mathbf{V}_k}\log q_\sigma(\mathbf{V}_k)
				\qquad q_\sigma\text{-a.e.}
				\label{eq:score_identity}
			\end{equation}
			and, by mutual absolute continuity of Gaussian-smoothed densities, also $p_{\mathrm{aug}}$-a.e.
		\end{enumerate}
		
		Let $\{\mathbf{V}_k^{(i,\ell)}\}_{\ell\geq0}$ be generated by the Langevin recursion \eqref{eq:langevin}, initialized at $\mathbf{V}_k^{(i,0)}=\mathbf{V}_{f,k}^{(i)}$. Denote
		\[
		\mu^\ell
		:=
		\mathrm{Law}(\mathbf{V}_k^{(i,\ell)}).
		\]
		Then:
		
		\begin{enumerate}[(i)]
			
			\item \textbf{Weak convergence to the surrogate posterior.} For fixed $\sigma_\ell\equiv\sigma$ and sufficiently small constant step size $\alpha_\ell\equiv\alpha$, the Langevin chain is the unadjusted Langevin algorithm targeting $q_\sigma$. Consequently, under the above regularity assumptions,
			\begin{equation}
				\mu^\ell
				\Rightarrow
				q_\sigma
				\qquad
				\text{as }\ell\to\infty
				\text{ and }\alpha\to0.
				\label{eq:langevin_weak_conv}
			\end{equation}
			
			\item \textbf{Quantitative Wasserstein bound.} If, in addition, the target satisfies the standard contractivity condition used in quantitative ULA theory, then for a constant step size $0<\alpha\leq m/L^2$,
			\begin{equation}
				W_2(\mu^\ell,q_\sigma)
				\leq
				\rho^\ell W_2(\mu^0,q_\sigma)
				+
				B_\alpha,
				\label{eq:wasserstein_bound}
			\end{equation}
			where
			\begin{equation}
				\rho
				=
				1-m\alpha+\frac{L^2\alpha^2}{2}
				\in(0,1),
				\label{eq:rho_def}
			\end{equation}
			and $B_\alpha$ is the ULA discretization bias satisfying
			\begin{equation}
				B_\alpha
				=
				\mathcal{O}(\alpha)
				\quad
				\text{or}
				\quad
				\mathcal{O}(\sqrt{\alpha}),
				\label{eq:bias_rate}
			\end{equation}
			depending on the regularity assumptions used in the underlying ULA estimate. In particular, the first term in \eqref{eq:wasserstein_bound} gives geometric convergence toward the smoothed surrogate posterior, while the second term vanishes as $\alpha\to0$.
			
			\item \textbf{Convergence toward the true posterior in the exact-CGNS limit.} Suppose additionally that the true dynamics exactly satisfy the CGNS structure \eqref{eq:cgns_obs}--\eqref{eq:cgns_lat}. Then $p_{\mathrm{CGNS}}=p_{\mathrm{aug}}$ and $\varepsilon_{\mathrm{CGNS}}=0$. By Theorem~\ref{thm:n2n_consistency},
			\begin{equation}
				\mathbf{s}_{\mathrm{EnKF},k}
				+
				\mathbf{s}_{\theta^*}
				\to
				\nabla\log p_{\mathrm{aug}}
				\quad
				\text{in }L^2(p_{\mathrm{aug}})
				\quad
				\text{as }N\to\infty.
				\label{eq:score_true_limit}
			\end{equation}
			Moreover,
			\begin{equation}
				q_\sigma
				=
				p_{\mathrm{aug}}
				*
				\mathcal{N}(0,\sigma^2\mathbf{I})
				\Rightarrow
				p_{\mathrm{aug}}
				\qquad
				\text{as }\sigma\to0.
				\label{eq:qsigma_to_paug}
			\end{equation}
			Therefore, taking the limits sequentially,
			\begin{equation}
				W_2
				\left(
				\mu^\ell,
				p_{\mathrm{aug}}(\mathbf{V}_k\mid\mathbf{Z}_{0:k})
				\right)
				\to0
				\quad
				\text{as }
				\ell\to\infty,\;
				N\to\infty,\;
				\alpha\to0,\;
				\sigma\to0.
				\label{eq:langevin_true_conv}
			\end{equation}
			Under CGNS misspecification, the Langevin chain still targets $q_\sigma$, and the residual gap to the true posterior is controlled by the surrogate misspecification error $\varepsilon_{\mathrm{CGNS}}$.
		\end{enumerate}
	\end{theorem}
	
	Theorem~\ref{thm:langevin_surrogate} establishes that the annealed Langevin step in $\Omega$ is mathematically well posed. The theorem says three things. First, once the learned score satisfies the Noise2Noise minimizer identity, the $\Omega$ analysis update is exactly the unadjusted Langevin algorithm targeting the smoothed surrogate posterior $q_\sigma$. Second, quantitative ULA theory gives a Wasserstein bound with two terms: a geometrically decaying mixing term and a step-size-dependent discretization bias. Third, when the CGNS surrogate exactly represents the true dynamics, the infinite-ensemble and small-noise limits make the surrogate posterior converge to the true augmented posterior. Under surrogate misspecification, the chain still converges, but to the best posterior represented by the EnKF-CGNS surrogate.
	
	\begin{proof}[Proof sketch]
		The key observation is that, by the Noise2Noise minimizer identity in
		Theorem~\ref{thm:n2n_consistency}, the total learned score satisfies
		\[
		\mathbf{s}_{\mathrm{EnKF},k}
		+
		\mathbf{s}_{\theta^*}
		=
		\nabla\log q_\sigma
		\]
		almost everywhere. Substituting this identity into the Langevin update \eqref{eq:langevin} shows that $\Omega$'s analysis step is precisely the unadjusted Langevin algorithm for the target density $q_\sigma$. The weak convergence and the Wasserstein estimate then follow from standard convergence results for ULA under Lipschitz and dissipativity conditions; see, for example, \cite{durmus2019analysis}. The exact-CGNS claim follows by combining the score convergence from Theorem~\ref{thm:n2n_consistency}, the Gaussian smoothing limit $q_\sigma\Rightarrow p_{\mathrm{aug}}$ as $\sigma\to0$, and the triangle inequality in Wasserstein distance. The full proof is given in~\ref{app:langevin_proof}.
	\end{proof}

	\subsection{Analytic Gaussian Mixture Recovery of $\bo Y$}
	\label{sec:y_sampling}
	
	The Langevin correction is performed on the augmented state $\mathbf{V}_k=(\bo X_k,\boldsymbol{\mu}_k)$ rather than directly on the latent variable $\bo Y_k$. This is a key design feature of $\Omega$. The conditional mean $\boldsymbol{\mu}_k$ parametrizes the CGNS conditional posterior of $\bo Y_k$, so correcting $\boldsymbol{\mu}_k$ modifies the entire latent distribution while avoiding direct sampling in the potentially high-dimensional latent space.
	
	Given the analysis augmented ensemble $\left\{(\bo X_{a,k}^{(i)}, \boldsymbol{\mu}_{a,k}^{(i)})\right\}_{i=1}^{N}$, the latent variable is recovered analytically through the CGNS
	conditional posterior:
	\begin{equation}
		\bo Y_{a,k}^{(i)}
		\sim
		\mathcal{N}
		\left(
		\boldsymbol{\mu}_{a,k}^{(i)},
		\mathbf{R}_{f,k}^{(i)}
		\right),
		\qquad
		i=1,\ldots,N,
		\label{eq:y_sampling}
	\end{equation}
	where the conditional covariance $\mathbf{R}_{f,k}^{(i)}$ is obtained from the CGNS forecast step \eqref{eq:forecast_conditional}. Consequently,
	\begin{equation}
		p(\bo Y_k\mid\mathbf Z_{0:k})
		\approx
		\frac1N
		\sum_{i=1}^{N}
		\mathcal N
		\left(
		\boldsymbol{\mu}_{a,k}^{(i)},
		\mathbf{R}_{f,k}^{(i)}
		\right),
		\label{eq:analysis_mixture}
	\end{equation}
	yielding a non-Gaussian Gaussian-mixture posterior. The non-Gaussianity arises from the ensemble diversity and from the nonlinear Langevin correction of the conditional means $\{\boldsymbol{\mu}_{a,k}^{(i)}\}$. Unlike importance-sampling approaches, all ensemble members carry equal weight and latent uncertainty is incorporated analytically within each Gaussian component.
	
	\begin{remark}[Practical step-size choice]
		\label{rem:stepsize}
		
		The Wasserstein estimate \eqref{eq:wasserstein_bound} reveals the classical bias-mixing tradeoff of Langevin sampling. Smaller step sizes reduce the discretization bias but slow geometric convergence through
		\[
		\rho
		=
		1-m\alpha+\frac{L^2\alpha^2}{2}.
		\]
		Conversely, larger step sizes accelerate mixing but increase the asymptotic discretization error.
		
		A natural constant-step-size choice is $\alpha=\mathcal O(m/L^2)$, which balances contraction and discretization effects.
		In practice, however, $\Omega$ employs an annealed schedule $\alpha_\ell\searrow0$, and $\sigma_\ell\searrow0$, so that early iterations explore the smoothed posterior while later iterations refine the ensemble near the target posterior. The convergence guarantees remain valid provided
		\[
		\sum_{\ell=1}^{\infty}\alpha_\ell=\infty,
		\qquad
		\sum_{\ell=1}^{\infty}\alpha_\ell^2<\infty.
		\]
	\end{remark}
	
	\begin{remark}[Beating the curse of dimensionality]
		\label{rem:cod}
		
		A key theoretical advantage of $\Omega$ is that the Monte Carlo error of the Gaussian-mixture approximation does not grow with the latent dimension $d_Y$.
		
		The reason is that the integration over $\bo Y_k$ is performed analytically within each conditional Gaussian component rather than by direct sampling in the latent space. Proposition~2.2 of~\cite{moser2026mechanisms} shows that when the conditional posterior $p(\bo Y_k\mid \bo X_{0:k}^{(i)})$, $i=1,\ldots,N$, is available in closed form, the resulting empirical mixture estimator is unbiased and has Monte Carlo error $\mathcal O(N^{-1/2})$, with variance determined entirely by the variability of the observed trajectories $\{\bo X_{0:k}^{(i)}\}_{i=1}^{N}$. Importantly, this variance is independent of the latent dimension $d_Y$.
		
		Consequently, increasing the dimension of the latent state does not degrade the convergence rate of the Gaussian-mixture approximation \eqref{eq:analysis_mixture}. The ensemble size $N$ controls the approximation error, while the latent dimension does not appear in the Monte Carlo rate. This provides a theoretical justification for performing Langevin dynamics only on the augmented state $(\bo X_k,\boldsymbol{\mu}_k)$ and recovering $\bo Y_k$ analytically. We provide the supporting numerical reuslts in~\ref{app:latent_dimension}.
	\end{remark}
	
	Finally, we show that the residual operator enriches the latent posterior beyond the affine EnKF-CGNS approximation and enables recovery of non-Gaussian posterior structures.
	
	\begin{theorem}[Non-Gaussian Recovery and Gaussian-Mixture Enrichment]
		\label{thm:nongaussian_recovery}
		
		Let
		\[
		\mathbf{s}_{\mathrm{post},k}
		=
		\nabla_{\mathbf V_k}
		\log p_{\mathrm{aug}}
		(\mathbf V_k\mid\mathbf Z_{0:k})
		\]
		denote the true augmented posterior score and let $p^*(\bo Y_k\mid\mathbf Z_{0:k})$ denote the true latent posterior.
		
		Let $\mathcal G$ denote the class of affine functions on $\mathbb R^{d_X+d_Y}$ and let $\mathcal F_\theta$ denote the hypothesis class of the residual operator
		$\mathbf s_\theta$. Assume that $p_{\mathrm{aug}}$ has finite second moment. Define the CGNS baseline mixture
		\begin{equation}
			p_{\mathrm{CGNS}}
			(\bo Y_k\mid\mathbf Z_{0:k})
			=
			\frac1N
			\sum_{i=1}^{N}
			\mathcal N
			\left(
			\boldsymbol{\mu}_{\mathrm{CGNS},k}^{(i)},
			\mathbf R_{f,k}^{(i)}
			\right)
			\label{eq:cgns_mixture_thm}
		\end{equation}
		and the $\Omega$ mixture
		\begin{equation}
			p_{\Omega}
			(\bo Y_k\mid\mathbf Z_{0:k})
			=
			\frac1N
			\sum_{i=1}^{N}
			\mathcal N
			\left(
			\boldsymbol{\mu}_{a,k}^{(i)},
			\mathbf R_{f,k}^{(i)}
			\right).
			\label{eq:omega_mixture_thm}
		\end{equation}
		
		Then:
		
		\begin{enumerate}[(i)]
			
			\item
			\textbf{Fundamental limitation of the affine baseline.} If $\mathbf s_{\mathrm{post},k}\notin\mathcal G$, then
			\begin{equation}
				\inf_{g\in\mathcal G}
				\|
				\mathbf s_{\mathrm{post},k}-g
				\|_{L^2(p_{\mathrm{aug}})}
				>0.
				\label{eq:affine_gap}
			\end{equation}
			Hence no affine score correction can exactly recover a non-Gaussian
			posterior score.
			
			\item
			\textbf{Residual operator approximation.} For every $\varepsilon>0$ there exists $\mathbf s_{\theta^*}\in\mathcal F_\theta$ such that
			\begin{equation}
				\|
				\mathbf s_{\mathrm{post},k}
				-
				\mathbf s_{\mathrm{EnKF},k}
				-
				\mathbf s_{\theta^*}
				\|_{L^2(p_{\mathrm{aug}})}
				<
				\varepsilon.
				\label{eq:residual_approx}
			\end{equation}
			Moreover, the approximation error tends to zero as the network capacity increases.
			
			\item
			\textbf{Necessity of latent correction.} Let
			\[
			\delta_\mu
			=
			\|
			\mathbf s_{\mathrm{post},k}^{\mu}
			-
			\mathbf s_{\mathrm{EnKF},k}^{\mu}
			\|_{L^2(p_{\mathrm{aug}})}.
			\]
			If
			$\delta_\mu>0$,
			then any residual operator satisfying
			\eqref{eq:residual_approx}
			with
			$\varepsilon<\delta_\mu$
			must satisfy
			\begin{equation}
				\|
				\mathbf s_{\theta^*}^{\mu}
				\|_{L^2(p_{\mathrm{aug}})}
				\ge
				\delta_\mu-\varepsilon
				>
				0.
				\label{eq:latent_correction_bound}
			\end{equation}
			Thus accurate posterior recovery requires a nontrivial correction in the latent conditional-mean subspace.
			
			\item
			\textbf{Gaussian-mixture enrichment.}
			
			Suppose the residual operator is consistent in the sense of Theorems~\ref{thm:n2n_consistency} and \ref{thm:langevin_surrogate}. Then, in the asymptotic regime
			\[
			L\to\infty,
			\qquad
			N\to\infty,
			\qquad
			\alpha\to0,
			\qquad
			\sigma\to0,
			\]
			the corrected mixture satisfies
			\begin{equation}
				p_{\Omega}
				(\bo Y_k\mid\mathbf Z_{0:k})
				\Rightarrow
				p^*
				(\bo Y_k\mid\mathbf Z_{0:k}),
				\label{eq:omega_to_true}
			\end{equation}
			whereas the CGNS baseline remains limited by its structural misspecification. Consequently,
			\begin{equation}
				D_{\mathrm{KL}}
				\left(
				p^*
				\,\|\,p_{\Omega}
				\right)
				\to0,
				\label{eq:kl_improvement}
			\end{equation}
			while
			$D_{\mathrm{KL}}
			\left(
			p^*
			\,\|\,p_{\mathrm{CGNS}}
			\right)$
			remains bounded below by the residual CGNS approximation error whenever the baseline mixture does not coincide with the true posterior.
			
		\end{enumerate}
	\end{theorem}
	
	Theorem~\ref{thm:nongaussian_recovery} provides the main expressiveness result for $\Omega$. Part~(i) shows that affine score corrections are fundamentally incapable of representing non-Gaussian posterior structure. Part~(ii) establishes that the residual operator can approximate the missing nonlinear score correction arbitrarily well. Part~(iii) shows that recovering latent posterior information generically requires a nontrivial correction of the conditional mean $\boldsymbol{\mu}_k$. Finally, Part~(iv) connects the learned score correction to the output of the method: the Langevin-corrected conditional means produce a Gaussian-mixture posterior that converges toward the true latent posterior and is asymptotically richer than the CGNS baseline.
	
	\begin{proof}[Proof sketch]
		
		Part~(i) follows from the orthogonal projection theorem in $L^2(p_{\mathrm{aug}})$. Since $\mathcal G$ is a closed finite-dimensional subspace, a non-affine posterior score has a strictly positive distance from $\mathcal G$. Part~(ii) follows from the universal approximation theorem together with a standard tail-truncation argument using the finite second moment of $p_{\mathrm{aug}}$. Part~(iii) follows by projecting the approximation error onto the $\boldsymbol{\mu}$ component and applying the triangle inequality. Part~(iv) follows from the convergence of the learned score to the posterior score, the Langevin convergence result of Theorem~\ref{thm:langevin_surrogate}, and the analytic Gaussian-mixture representation \eqref{eq:analysis_mixture}. The full proof is given in~\ref{app:nong_proof}.
		
	\end{proof}

	\subsection{Summary of Information Flow}
	\label{sec:summary}
	
	Table~\ref{tab:info_flow} summarizes the main information pathways in
	one $\Omega$ assimilation cycle.
	
	\begin{table}[h]
		\centering
		\caption{Information propagation pathways in the $\Omega$ framework
			at assimilation time $t_k$.
			$\mathbf{V}_k = (\bo X_k,\boldsymbol{\mu}_k)$ is the augmented
			state;
			$\mathbf{R}_{f,k}^{(i)}$ is CGNS forecast covariance for latent variable;
			$\mathbf{P}_{f,k}$ is the empirical augmented forecast covariance \eqref{eq:enkf_cov};
			$\mathbf{R}_{\mathrm{obs}}$ is the observation noise covariance;
			$\mathcal{C}_k = (\bar{\mathbf{V}}_{f,k},
			\mathrm{diag}(\mathbf{P}_{f,k}),\mathbf{Z}_k)$ is the context vector \eqref{eq:context}.
			Notation from
			Sections~\ref{sec:cgns_enkf}--\ref{sec:joint_score}.}
		\label{tab:info_flow}
		\renewcommand{\arraystretch}{1.6}
		\resizebox{\textwidth}{!}{%
			\begin{tabular}{p{4.5cm} p{4.0cm} p{9.5cm} l}
				\hline
				\textbf{Pathway}
				& \textbf{Mechanism}
				& \textbf{Form}
				& \textbf{Behavior} \\
				\hline
				
				$\mathbf{Z}_k \to \bo X_k$
				& EnKF likelihood score
				\eqref{eq:lik_score}
				& $\mathbf{H}_X^*\mathbf{R}_{\mathrm{obs}}^{-1}
				(\mathbf{Z}_k - \mathbf{H}_X\bo X_k)$
				& Linear (Gaussian) \\
				
				$\mathbf{Z}_k \to \bo X^{(i)} \to \boldsymbol{\mu}_k^{(i)}$
				via CGNS mean ODE
				& CGNS filtering ODE \eqref{eq:cgns_mean}
				driven by $d\bo X^{(i)}$
				& $d\boldsymbol{\mu}^{(i)}
				=\bigl(\mathbf{a}_0(\bo X^{(i)},t)
				+\mathbf{a}_1(\bo X^{(i)},t)\,\boldsymbol{\mu}^{(i)}\bigr)\,dt$
				$+\,\mathbf{R}^{(i)}\mathbf{A}_1^*(\bo X^{(i)},t)
				\bigl(\boldsymbol{\Sigma}_X\boldsymbol{\Sigma}_X^*\bigr)^{-1}
				\!\bigl[d\bo X^{(i)}
				-\bigl(\mathbf{A}_0(\bo X^{(i)},t)
				+\mathbf{A}_1(\bo X^{(i)},t)\,\boldsymbol{\mu}^{(i)}\bigr)dt\bigr]$
				& Linear in $\boldsymbol{\mu}$ \\
				
				$\mathbf{R}_{f,k}^{(i)}$ evolution
				& CGNS Riccati ODE \eqref{eq:cgns_cov};
				conditionally deterministic
				(obs.-path-independent)
				& $d\mathbf{R}^{(i)}
				=\bigl[\mathbf{a}_1\mathbf{R}^{(i)}
				+\mathbf{R}^{(i)}\mathbf{a}_1^*
				+\boldsymbol{\Sigma}_Y\boldsymbol{\Sigma}_Y^*
				-\mathbf{R}^{(i)}\mathbf{A}_1^*
				(\boldsymbol{\Sigma}_X\boldsymbol{\Sigma}_X^*)^{-1}
				\mathbf{A}_1\mathbf{R}^{(i)}\bigr]\,dt$
				& Deterministic \\
				
				$\mathbf{Z}_k \to \bo X_k \to \boldsymbol{\mu}_k$
				via EnKF cross-covariance
				& Kalman gain on the
				$\boldsymbol{\mu}$-block of
				$\mathbf{P}_{f,k}$
				\eqref{eq:enkf_update}
				& $\boldsymbol{\mu}_{a,k}^{(i)}
				= \boldsymbol{\mu}_{f,k}^{(i)}
				+ \mathbf{K}_k^{\boldsymbol{\mu}}
				\!\bigl(\mathbf{Z}_k
				- \mathbf{H}_X\bo X_{f,k}^{(i)}
				+ \boldsymbol{\epsilon}_k^{(i)}\bigr)$,
				\quad
				$\mathbf{K}_k^{\boldsymbol{\mu}}
				= \mathbf{P}_{f,k}^{\boldsymbol{\mu}X}\mathbf{H}_X^*
				\!\bigl(\mathbf{H}_X\mathbf{P}_{f,k}^{XX}\mathbf{H}_X^*
				+\mathbf{R}_{\mathrm{obs}}\bigr)^{-1}$
				& Linear (Gaussian) \\
				
				$\mathbf{V}_k \to \bo X_k$
				via residual operator
				& Nonlinear score correction
				on the $X$-component
				\eqref{eq:score_components}
				& $\mathbf{s}_\theta^X
				\!\bigl(\mathbf{V}_k,\,\sigma,\,\mathcal{C}_k\bigr)$
				& Nonlinear \\
				
				$\mathbf{V}_k \to \boldsymbol{\mu}_k$
				via residual operator
				& Nonlinear score correction
				on the $\boldsymbol{\mu}$-component
				\eqref{eq:score_components}
				& $\mathbf{s}_\theta^{\boldsymbol{\mu}}
				\!\bigl(\mathbf{V}_k,\,\sigma,\,\mathcal{C}_k\bigr)$
				& Nonlinear \\
				
				Joint Langevin update of
				$(\bo X_k,\boldsymbol{\mu}_k)$
				& Annealed Langevin MCMC
				on the augmented state
				\eqref{eq:langevin}
				& $\mathbf{V}_k^{(i,\ell+1)}
				= \mathbf{V}_k^{(i,\ell)}
				+ \dfrac{\alpha_\ell}{2}
				\!\bigl[\mathbf{s}_{\mathrm{EnKF},k}
				\!\bigl(\mathbf{V}_k^{(i,\ell)}\bigr)
				+\mathbf{s}_\theta
				\!\bigl(\mathbf{V}_k^{(i,\ell)},\,
				\sigma_\ell,\,\mathcal{C}_k\bigr)\bigr]
				+\sqrt{\alpha_\ell}\,\boldsymbol{\eta}_\ell^{(i)}$,
				\quad $\boldsymbol{\eta}_\ell^{(i)}\sim\mathcal{N}(0,\mathbf{I})$
				& Nonlinear \\
				
				$\boldsymbol{\mu}_{a,k}^{(i)}
				\to \bo Y_{a,k}^{(i)}$
				& Analytic Gaussian draw
				from corrected mean
				\eqref{eq:y_sampling}
				& $\bo Y_{a,k}^{(i)}
				\sim\mathcal{N}\!\bigl(
				\boldsymbol{\mu}_{a,k}^{(i)},\,
				\mathbf{R}_{f,k}^{(i)}\bigr)$
				& Exact (closed-form) \\
				
				$p(\bo Y_k\mid\mathbf{Z}_{0:k})$
				& Gaussian mixture over
				the corrected ensemble
				\eqref{eq:analysis_mixture}
				& $\dfrac{1}{N}\displaystyle\sum_{i=1}^{N}
				\mathcal{N}\!\bigl(
				\boldsymbol{\mu}_{a,k}^{(i)},\,
				\mathbf{R}_{f,k}^{(i)}\bigr)$
				& Non-Gaussian (mixture) \\
				
				\hline
			\end{tabular}%
		}
	\end{table}
	
	The EnKF-CGNS baseline provides a stable Gaussian first-order approximation, the residual operator learns the missing nonlinear posterior score, the Langevin step transports the augmented ensemble toward the corrected posterior, and the latent variable is recovered analytically as a Gaussian mixture. This is the core mechanism by which $\Omega$ combines the stability of CGNS filtering with the flexibility of score-based non-Gaussian correction.

	\section{Numerical Validation in Nonlinear Systems with Intermittency and Extreme Events}
	\label{sec:experiments}
	
We validate $\Omega$ on three nonlinear dynamical systems that span a range of complexity, intermittency, and model structure. The dyad model serves as a canonical testbed for strongly non-Gaussian filtering and allows detailed examination of posterior recovery in a low-dimensional setting. We next consider a topographic mean-flow interaction model, a physically motivated geophysical system that exhibits intermittent bursts, extreme events, and highly non-Gaussian statistics. Both systems satisfy the CGNS structure \eqref{eq:cgns_obs}--\eqref{eq:cgns_lat} exactly, enabling us to assess the performance of $\Omega$ when the analytical conditional-Gaussian framework is fully available for unresolved modes. Finally, we examine the stochastic Lorenz 96 system, a moderately high-dimensional nonlinear model that does not admit a natural CGNS representation. In this case, a surrogate CGNS is constructed, providing a stringent test of the robustness of $\Omega$ under model approximation.

Across all experiments, we compare $\Omega$ against two baseline methods: the standard EnKF and the CGNS-EnKF. The latter applies the EnKF analysis step on the augmented CGNS state without any learned residual correction. These comparisons isolate the benefits of the analytical CGNS representation from those of the learned non-Gaussian correction. For the dyad model, we additionally include a particle filter (PF) benchmark, whose ability to accurately resolve the posterior distribution provides a reference for assessing how well $\Omega$ captures strongly non-Gaussian posterior structures.

\subsection{Recovering Non-Gaussian Latent Dynamics in an Intermittent Dyad System}
\label{sec:exp_dyad}

The dyad model is a canonical stochastic system for studying intermittency, nonlinear energy transfer, and non-Gaussian fluctuations in dynamical systems \cite{majda2018model}. Despite its low dimensionality, the model exhibits many phenomena commonly associated with significantly more complex turbulent systems, including intermittent bursts, heavy-tailed statistics, and strong nonlinear interactions between observed and hidden variables \cite{majda2013physics}. These features make the dyad model an ideal testbed for assessing whether a data assimilation method can accurately reconstruct latent dynamics and characterize non-Gaussian posterior distributions beyond Gaussian approximations.

The model describes the nonlinear interaction between an observed variable $x$ and an unobserved variable $y$:
\begin{align}
	dx &= \left(-d_x x + \gamma x y + f_x\right) dt
	+ \eta_x\, dW_x,
	\label{eq:dyad_x} \\
	dy &= \left(-d_y y - \gamma x^2 + f_y\right) dt
	+ \eta_y\, dW_y,
	\label{eq:dyad_y}
\end{align}
where $d_x,d_y>0$ are damping coefficients, $\gamma$ controls the nonlinear coupling strength, $f_x$ and $f_y$ are external forcings, and $W_x$ and $W_y$ are independent Wiener processes with amplitudes $\eta_x$ and $\eta_y$. The hidden variable $y$ feeds back into the observed dynamics through the bilinear term $\gamma xy$, while the observed variable drives the hidden dynamics through the quadratic term $-\gamma x^2$. This bidirectional nonlinear coupling generates intermittent excursions and strongly non-Gaussian joint statistics.

Observations are available only for the observed variable $x$:
\begin{equation}
	Z_k = x(t_k) + \epsilon_k,
	\qquad
	\epsilon_k \sim \mathcal{N}(0,R_{\mathrm{obs}}).
	\label{eq:dyad_obs}
\end{equation}
Importantly, the dyad model satisfies the CGNS structure \eqref{eq:cgns_obs}--\eqref{eq:cgns_lat} exactly, with $\mathbf{A}_1(x,t)=\gamma x$ and $\mathbf{a}_1(x,t)=-d_y$. This allows us to isolate the effect of the learned residual correction without introducing surrogate-model error. The physical parameters are chosen as
$d_x = 0.5$, $d_y = 0.5$, $\gamma = 2$, $f_x = 0.5$, $f_y = 1$, with noise amplitudes $\eta_x = 0.5$ and $\eta_y = 1$. The truth trajectory is generated using the Euler-Maruyama discretization. The initial condition is $x_0=2$ and $y_0=0$, while the integration time step is $\Delta t = 0.005$, and noisy observations of $x$ are provided every 100 model time steps with observational noise standard deviation $0.2$.

We evaluate two annealed Langevin noise schedules: a fine-grained schedule with
$\sigma \in [0.01,0.7]$, which emphasizes accurate posterior refinement at small noise levels, and a coarser schedule with
$\sigma \in [0.5,1.0]$, which prioritizes stronger regularization. Both schedules produce qualitatively similar posterior estimates. The fine-grained schedule yields slightly lower latent-state RMSE at the expense of additional sampling steps and is therefore adopted throughout this subsection unless stated otherwise.

Figure~\ref{fig:tracking_x} shows the assimilation results for the observed variable $x$. Since $x$ is directly observed, all methods are able to track the truth trajectory accurately. Differences become visible primarily during rapid transitions, where the EnKF exhibits a slightly more conservative response. 

The more stringent test concerns the hidden variable $y$, shown in Figure~\ref{fig:tracking_y}. Unlike $x$, the variable $y$ is never observed directly and must be inferred solely through its nonlinear dynamical coupling with the observed state. Recovering $y$ therefore requires the assimilation method to correctly propagate observational information across the latent-observed interface rather than simply fitting the available measurements. This setting provides a direct assessment of the ability of $\Omega$ to reconstruct hidden intermittent dynamics. The results in figure~\ref{fig:tracking_y} show that $\Omega$ closely follows the true latent trajectory and achieves an accuracy comparable to the particle filter, while the EnKF exhibits noticeably larger deviations, particularly during intermittent excursions where the posterior departs most strongly from Gaussianity.

\begin{figure}[htbp]
	\centering
	\includegraphics[width=.8\textwidth]{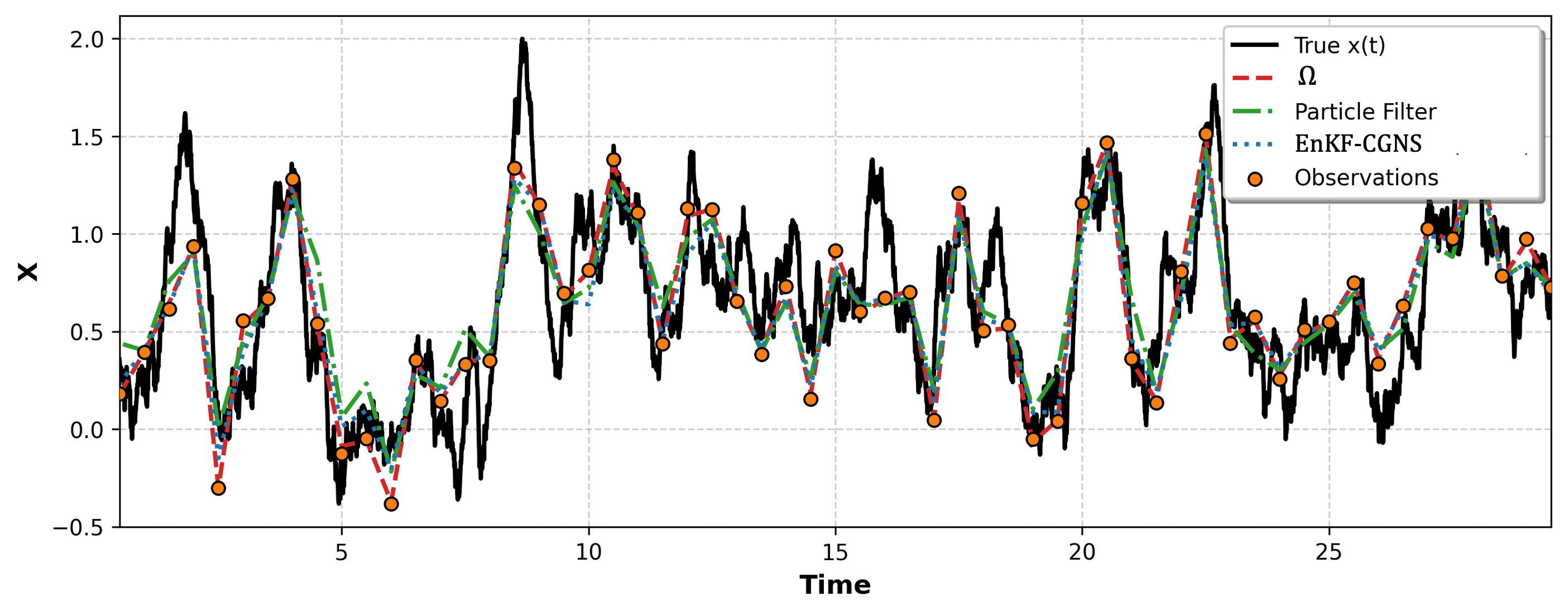}
	\caption{
		Closed-loop assimilation of the observed variable $x$ for the dyad model.
		Shown are the true trajectory (black), the $\Omega$ posterior mean (red), the EnKF posterior mean (blue), and the particle filter posterior mean (green).
	}
	\label{fig:tracking_x}
\end{figure}

\begin{figure}[h]
	\centering
	\includegraphics[width=.8\textwidth]{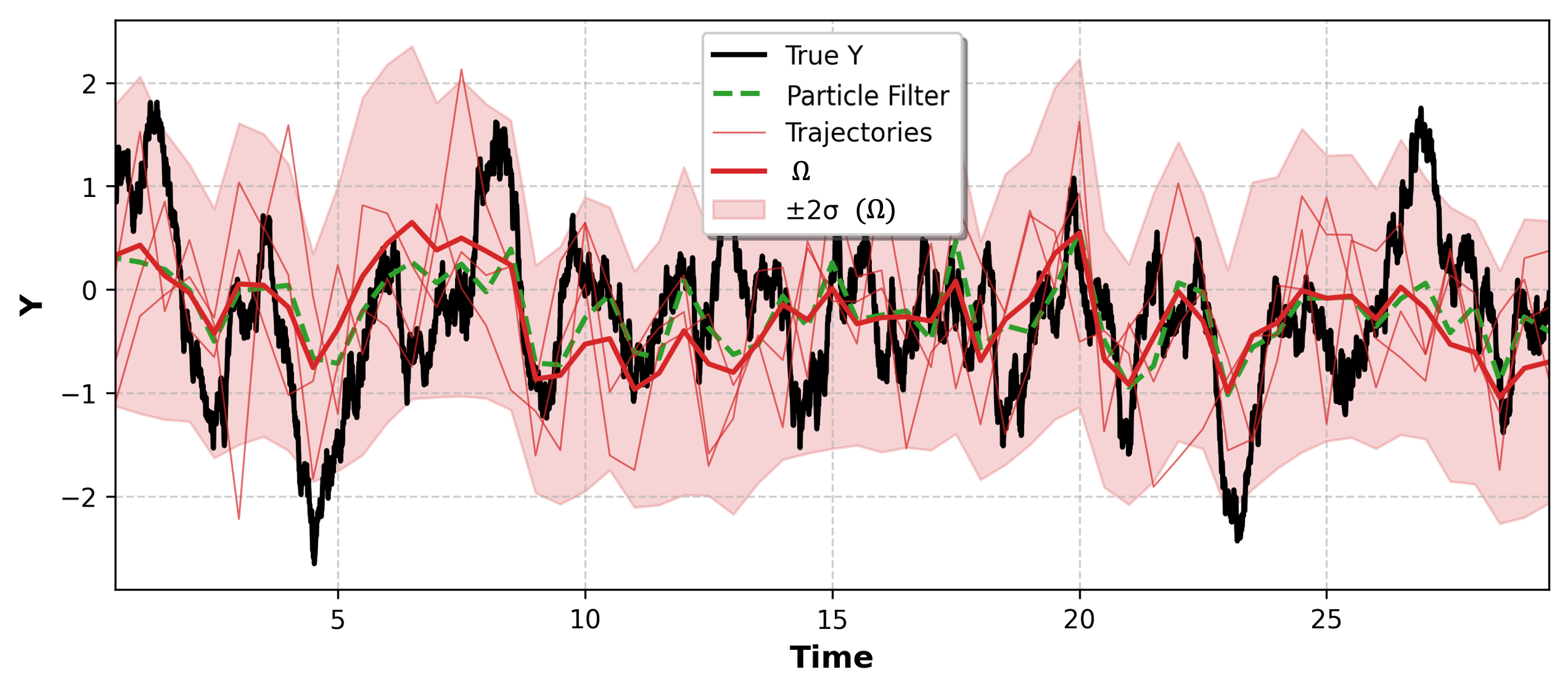}
\caption{Reconstruction of the hidden variable $y$ for the dyad model.  The $\Omega$ posterior mean (red) and particle filter mean (green) both closely track the latent trajectory.  Shading shows $\pm 2$ posterior standard
    deviations for $\Omega$.}
	\label{fig:tracking_y}
\end{figure}
The principal advantage of $\Omega$ is not reflected solely in posterior mean accuracy, but in its ability to recover the full posterior distribution. This distinction is particularly important in intermittent systems, where rare but dynamically significant events are encoded in the tails of the distribution rather than in its mean.

Figure~\ref{fig:pdf_y} compares the marginal probability density functions of $x$ and $y$ accumulated over the entire assimilation window. While the EnKF reproduces the central bulk of the distribution reasonably well, it substantially underestimates the probability of extreme excursions and fails to capture the asymmetric heavy-tailed structure generated by the nonlinear dynamics. This deficiency is particularly evident for the hidden variable $y$, whose statistics are strongly influenced by intermittent bursts and nonlinear energy exchange. In contrast, the Gaussian-mixture representation underlying $\Omega$ accurately reproduces both the central region and the tails of the particle-filter reference distribution, demonstrating its ability to preserve non-Gaussian information throughout the assimilation cycle.

\begin{figure}[h]
	\centering

    \includegraphics[width=0.49\textwidth]{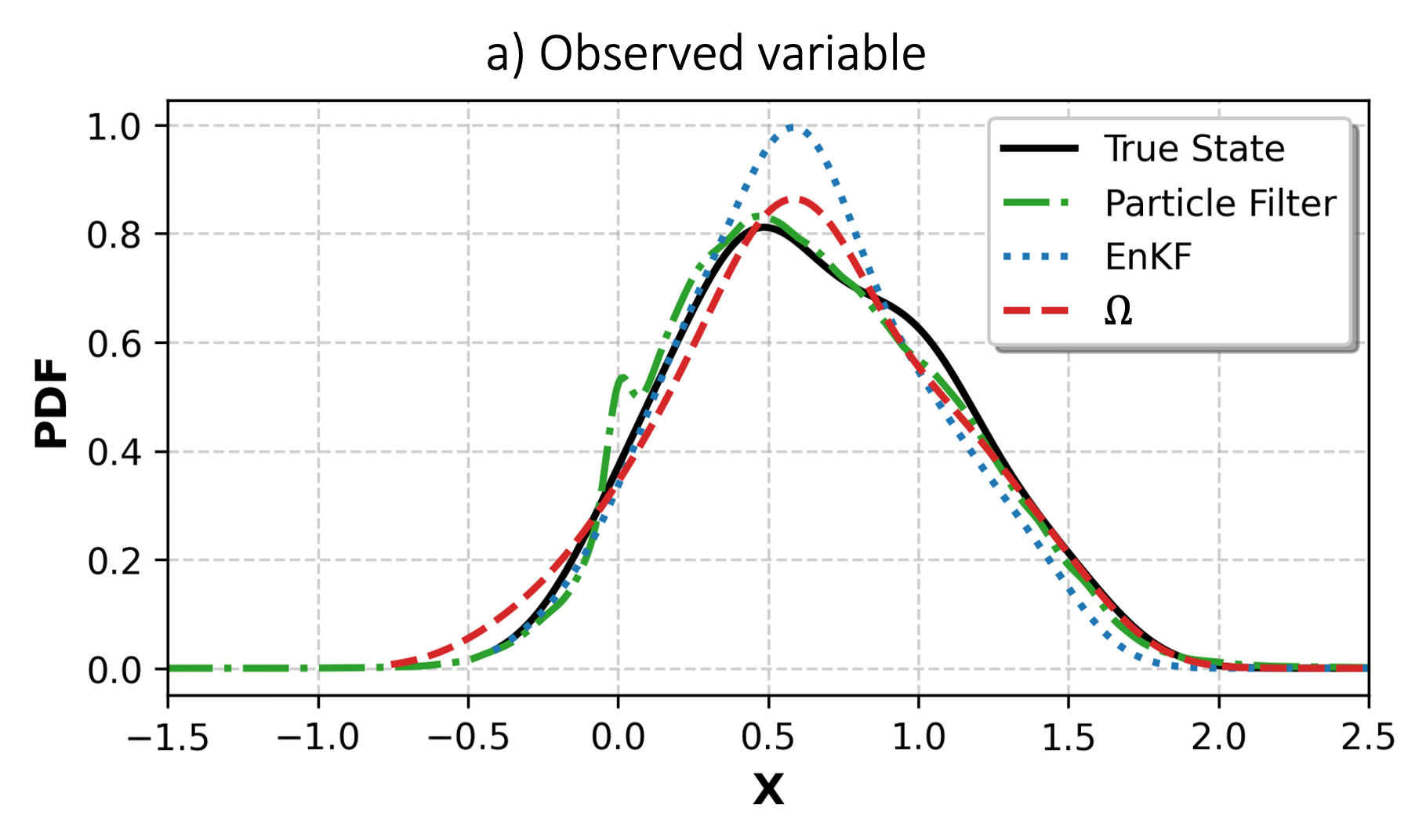}
	\includegraphics[width=0.49\textwidth]{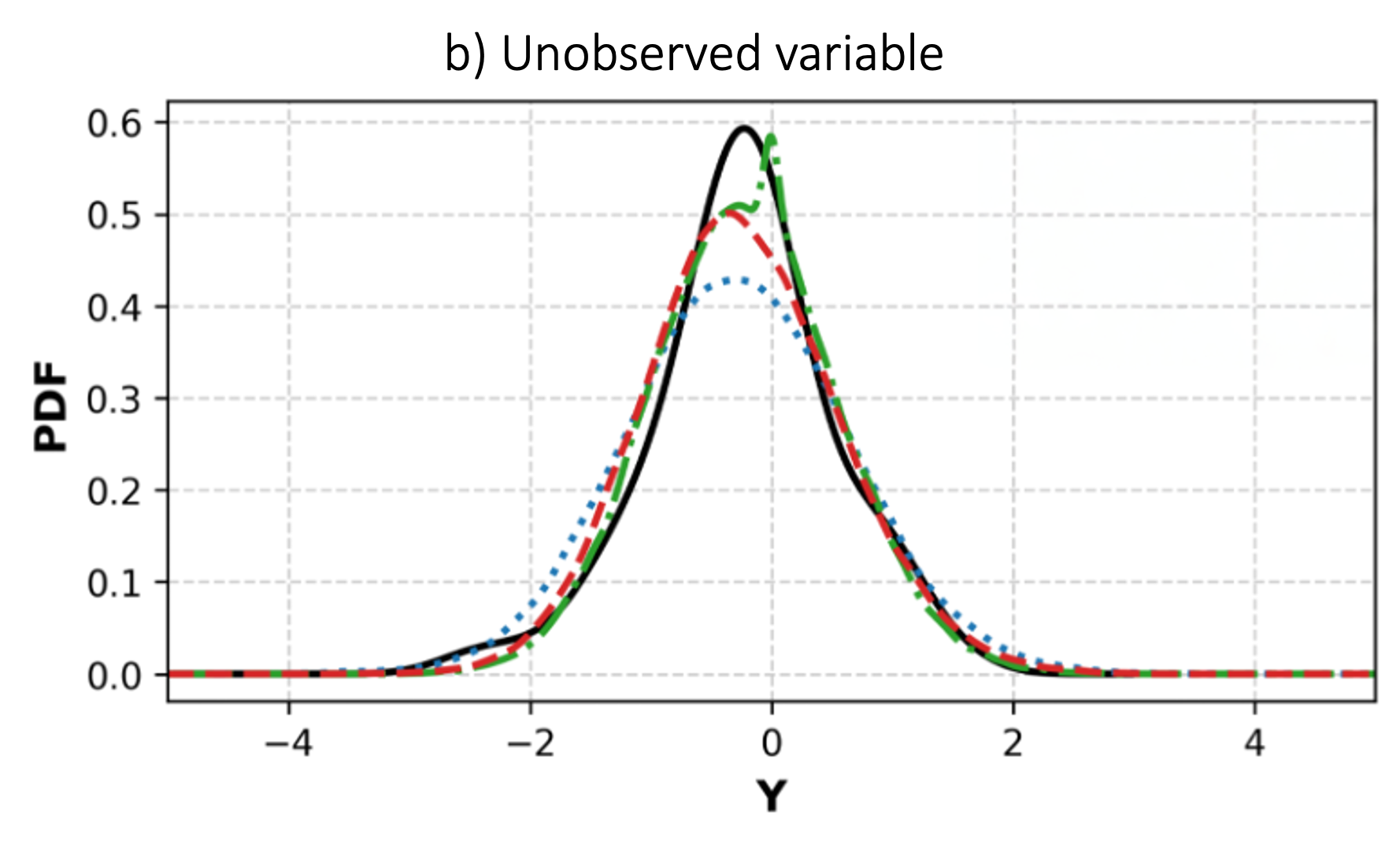}
	\caption{
		Marginal probability density functions of $x$ (a) and $y$ (b) accumulated over all assimilation times.
		The Gaussian-mixture posterior generated by $\Omega$ reproduces the asymmetric and heavy-tailed structure of the particle-filter reference distribution, whereas the EnKF shows larger errors in estimating the probability of extreme events.
	}
	\label{fig:pdf_y}
\end{figure}

Figure~\ref{fig:pdf_y_time} provides an even more stringent assessment by examining the instantaneous posterior distribution of the hidden variable at selected assimilation times. Visible errors appear in the EnKF solutions compared to the benchmark particle filter solution, not only in the variance but also in estimating the mean state (panel (b)). In contrast, $\Omega$ successfully recovers both probability peaks and closely matches the particle-filter reference solution. This result highlights the ability of the learned residual operator to restore posterior structures that cannot be represented through covariance corrections alone. More broadly, these results demonstrate that the benefit of $\Omega$ extends beyond improving the mean state estimates. The framework provides a substantially more faithful characterization of posterior uncertainty, including heavy tails and asymmetry. Such features play a central role in quantifying risk, diagnosing hidden dynamical regimes, and predicting intermittent extreme events, all of which are difficult to capture using Gaussian filtering approaches.

\begin{figure}[htbp]
	\centering
	\includegraphics[width=.6\textwidth]{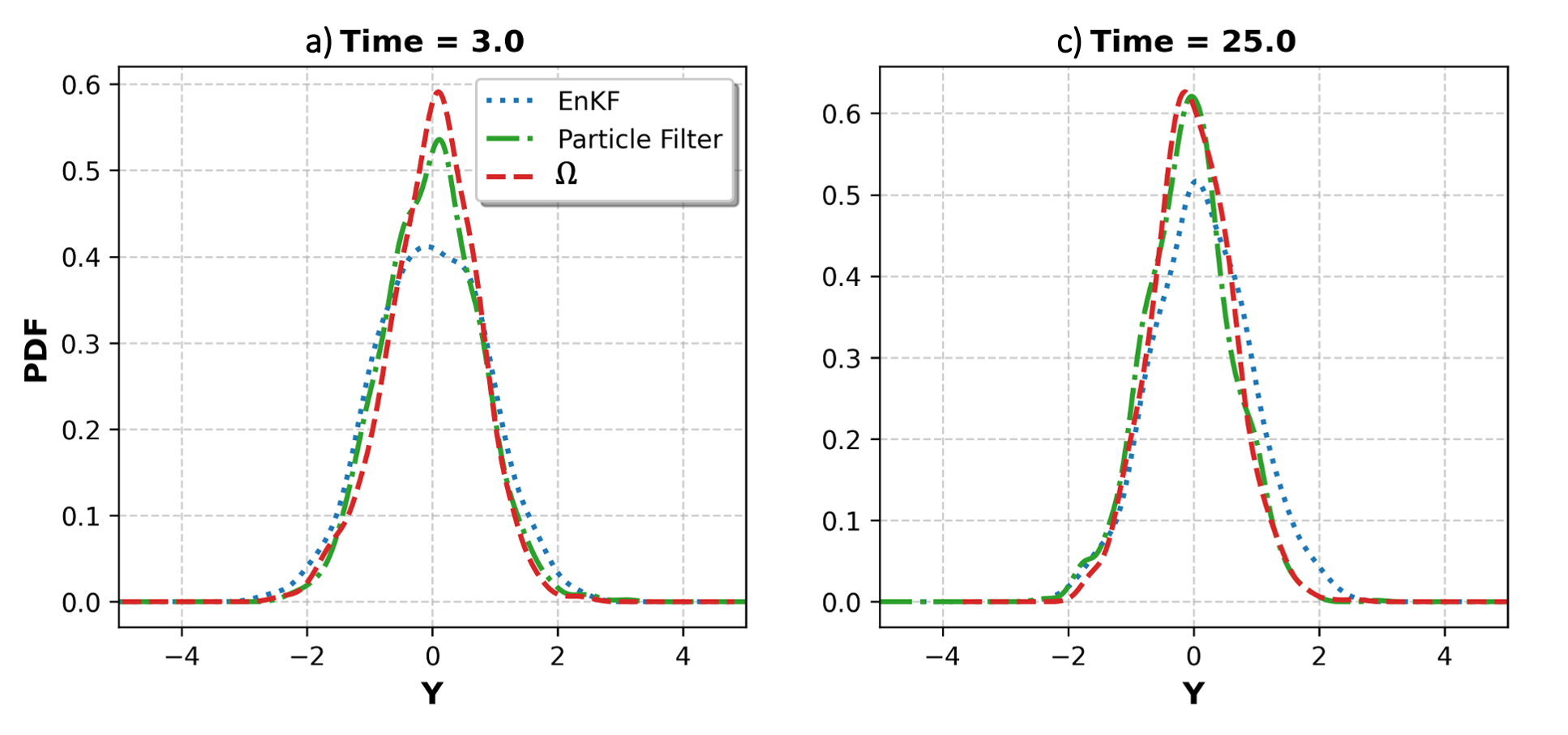}
	\caption{
		Conditional posterior distributions of the hidden variable $y$ at three selected assimilation times.
	}
	\label{fig:pdf_y_time}
\end{figure}

\subsection{Recovering Hidden Intermittent Dynamics in a Topographic Geophysical Flow}
\label{sec:exp_topo}

The topographic model describes barotropic quasi-geostrophic flow over a rough bottom boundary and serves as a physically motivated prototype for large-scale atmosphere-ocean dynamics \cite{majda2006nonlinear,chen2020predicting, vallis2017atmospheric}. Unlike the dyad model, which is primarily used as a canonical test problem, the topographic model captures key mechanisms underlying geophysical variability, including topographic form stress, nonlinear energy exchange across spatial scales, and intermittent extreme events. The model has been extensively used to study predictability, uncertainty quantification, and rare-event dynamics in geophysical fluid systems. These features make it an ideal testbed for assessing whether $\Omega$ can accurately recover hidden intermittent dynamics in a realistic Earth-system setting.

The governing topographic barotropic flow equation reads~\cite{majda2006nonlinear}
\begin{equation}
\begin{aligned}
d\omega&=
\left[
-\nabla^\perp \psi \cdot \nabla \omega 
+f+D(\Delta)\psi\right]dt, \\
dU&=\left[-\fint \frac{\partial h}{\partial x} \psi'
-\gamma_U U+f_U\right]dt,
\end{aligned}
\label{eq:vorticity_pde_app}
\end{equation}
which is defined in a two-dimensional domain $\mathcal{D} : \mathbf{x} = (x,y) \in [-\pi, \pi]^2$ with double periodic boundary conditions. In \eqref{eq:vorticity_pde_app}, $D(\Delta)$ is the dissipation operator, $h$ is the topographic effect, $\gamma_U$ is the dissipation coefficient, and $f$ and $f_U$ are external forcings. The state variable $U$ represents the large-scale zonal flow velocity, while $\omega$ and $\psi$ are the potential vorticity and the stream function, respectively. These two fields are related using $\omega = \omega' + f = \nabla^2 \psi' + h + \beta y$ and $ \psi = -U(t)y + \psi'$ with the prime terms denoting the fluctuation related to the large-scale mean flow. The averaged integration in the mean dynamics $U$ is defined as $\fint p\, d\mathbf{x} = \frac{1}{|\mathcal{D}|}\int_\mathcal{D}\, p\, d\mathbf{x}$, with $|\mathcal{D}|$ being the domain's total area. Therefore, the term $\fint \frac{\partial h}{\partial x} \psi'$ mediates energy exchange between the large-scale flow and unresolved eddies and is responsible for generating strongly non-Gaussian fluctuations and intermittent bursts.

Additionally, a passive tracer model that characterizes the advection-diffusion of the transport of a thermodynamic density field $T(\mathbf{x}, t)$ evolves according to
\begin{equation}
dT=\left[
 -\mathbf{u} \cdot \nabla T-\gamma_T T + \kappa\Delta T\right]dt,
\label{eq:T_pde_app}
\end{equation}
completing the coupled system. $\mathbf{u} = \nabla^\perp \psi$, $d_T$ is the drag term and $\kappa$ is the diffusion coefficient. A particularly interesting case of the barotropic model is the one with layered topography, where the topography and stream function have the following expansion form $h(x,y) = \sum_{k} \hat{h}_k e^{i k \mathbf{l} \cdot \mathbf{x}}, $ and $ \psi(x,y,t) = \sum_{k} \hat{\psi}_k(t) e^{i k \mathbf{l} \cdot \mathbf{x}}$.
The expansion is along one characteristic wavenumber direction $\mathbf{l} = (l_x, l_y)$ with $|\mathbf{l}| = 1$. The full spectral formulation is summarized in~\ref{app:topo_equations}. After projecting onto four Fourier modes for both the vorticity and thermodynamic fields, the reduced-order state becomes
\begin{equation}
\mathbf{U}
=
\bigl(
U,
\omega_1,\ldots,\omega_4,
T_1,\ldots,T_4
\bigr)^\top.
\label{eq:topo_state}
\end{equation}
The mean flow $U$ constitutes the observed component, while all spectral coefficients ${\omega_k,T_k}$ remain hidden. From a data-assimilation perspective, this setting is considerably more challenging than the dyad model because the observed variable represents only a coarse large-scale flow, whereas the hidden variables correspond to unresolved eddy and thermodynamic modes that directly govern the intermittent dynamics.

Importantly, the reduced system satisfies the CGNS structure exactly. The drift of $U$ is linear in the hidden modes through the topographic form stress, while the drifts of $\omega_k$ and $T_k$ remain conditionally linear given $U$; see ~\ref{app:topo_cgns}. This allows us to isolate the impact of non-Gaussian posterior correction without introducing surrogate-model error.

The model is integrated using RK4 with time step $\Delta t = 5\times10^{-4}$, and noisy observations of the mean flow are provided every 200 model time steps:
\begin{equation}
Z_k
=
U(t_k)
+
\epsilon_k,
\qquad
\epsilon_k
\sim
\mathcal{N}(0,R_{\mathrm{obs}}).
\label{eq:topo_obs}
\end{equation}
All spectral modes remain completely unobserved.
The parameters are chosen as
\begin{equation}
\beta=1.0,
\qquad
\gamma_{\omega,k}=0.0125,\quad k=1,\ldots,4,
\qquad
\gamma_U=0.0125,
\end{equation}
and
\begin{equation}
f_T=f_U=0,
\qquad
\kappa_T=0.001,
\qquad
\alpha=1.0,
\qquad
\sigma_U=\frac{1}{\sqrt{2}}.
\end{equation}
All Fourier modes are initialized at unity.

Figure~\ref{fig:topo_omega_traj} compares the reconstructed trajectories of the observed variable $U$ and several representative hidden modes. Since the hidden variables are never observed directly, successful recovery requires the assimilation algorithm to infer the slow-fast multiscale evolution of unresolved eddies and thermodynamic fluctuations solely from measurements of the large-scale mean flow. This problem closely mirrors practical geophysical applications, where sparse observations must constrain a much larger hidden state.

The results demonstrate that $\Omega$ substantially improves the reconstruction of hidden dynamics relative to the EnKF. The improvement is particularly pronounced during intermittent bursts, where the EnKF systematically underestimates fluctuation amplitudes due to its Gaussian analysis assumption. In contrast, the learned residual correction in $\Omega$ allows observational information to propagate nonlinearly into the hidden modes, yielding more accurate reconstruction of extreme excursions and energetic events.

\begin{figure}[htbp]
\centering
\includegraphics[width=0.9\textwidth]{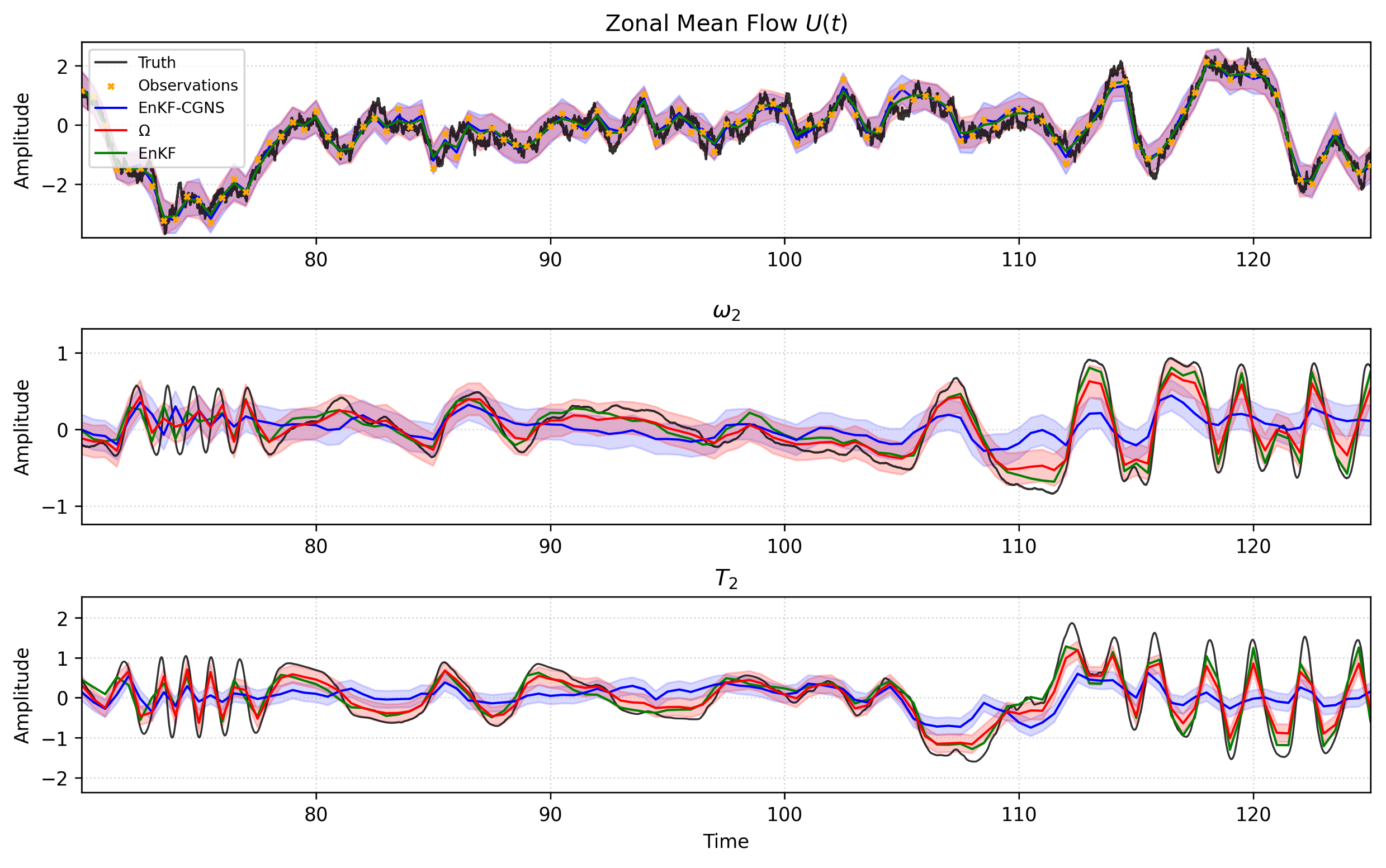}
\caption{
Posterior mean trajectories of representative hidden spectral modes.
$\Omega$ more accurately reconstructs intermittent fluctuations and extreme excursions than the EnKF, particularly during periods of enhanced nonlinear activity. The EnKF uncertainty band is omitted for visual clarity and overall easier interpretation. 
}
\label{fig:topo_omega_traj}
\end{figure}

A more revealing comparison is provided by the posterior distributions. Figure~\ref{fig:topo_pdf} shows the posterior distribution of a representative hidden vorticity mode. The reference distribution exhibits clear asymmetry and heavy-tailed behavior arising from intermittent energy exchange between the mean flow and unresolved eddies. Such structures are important because they govern the probability of rare but dynamically significant events.

The EnKF reproduces only the central bulk of the distribution and substantially underestimates the probability of extreme excursions. In contrast, the Gaussian-mixture representation underlying $\Omega$ more accurately captures both the asymmetric tail behavior and the overall shape of the reference distribution. This demonstrates that the benefits of $\Omega$ extend beyond improved state estimation and include substantially improved uncertainty quantification in intermittent geophysical systems.

The large-scale flow structures can be visualized through the streamfunction field. Given the zonal mean flow $U(t)$ and the Fourier mode amplitudes, the streamfunction is reconstructed as
\begin{equation}
\psi(x,y,t)=
-U(t),y+
\sum_{\mathbf{k}\in\mathcal{K}}
\phi_{\mathbf{k}}(t)
e^{i(k_x x+k_y y)},
\end{equation}
where $\mathbf{k}=(k_x,k_y)$ denotes the wavevector and $\phi_{\mathbf{k}}$ are the corresponding streamfunction Fourier coefficients. The resulting streamfunction field provides a spatial representation of the large-scale circulation and coherent structures generated by the model. Figure~\ref{fig:streamfunction} displays two representative snapshots of the reconstructed streamfunction field obtained from the simulated dynamics. The top row shows a blocking-like pattern, while the bottom row illustrates a strong zonal jet corresponding to a nearly unblocked field.

\begin{figure}[t]
\centering
\includegraphics[width=0.85\textwidth]{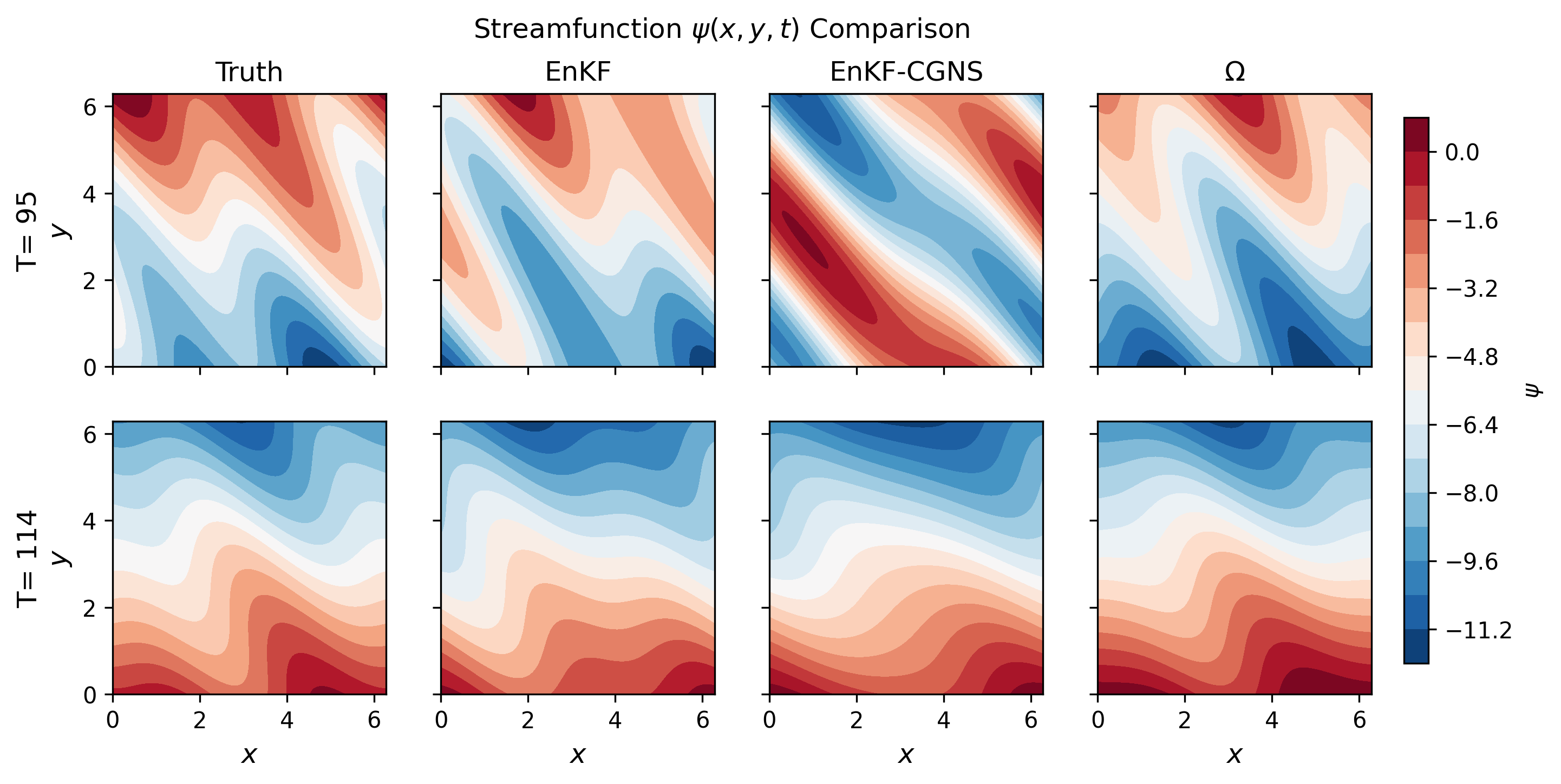}
\caption{Snapshot of the reconstructed streamfunction field $\psi(x,y,t)$ at a representative time instance. The field is obtained from the Fourier-mode representation of the model and illustrates the spatial organization of the large-scale circulation and wave structures.}
\label{fig:streamfunction}
\end{figure}

\begin{figure}[htbp]
\centering
\includegraphics[width=0.7\textwidth]{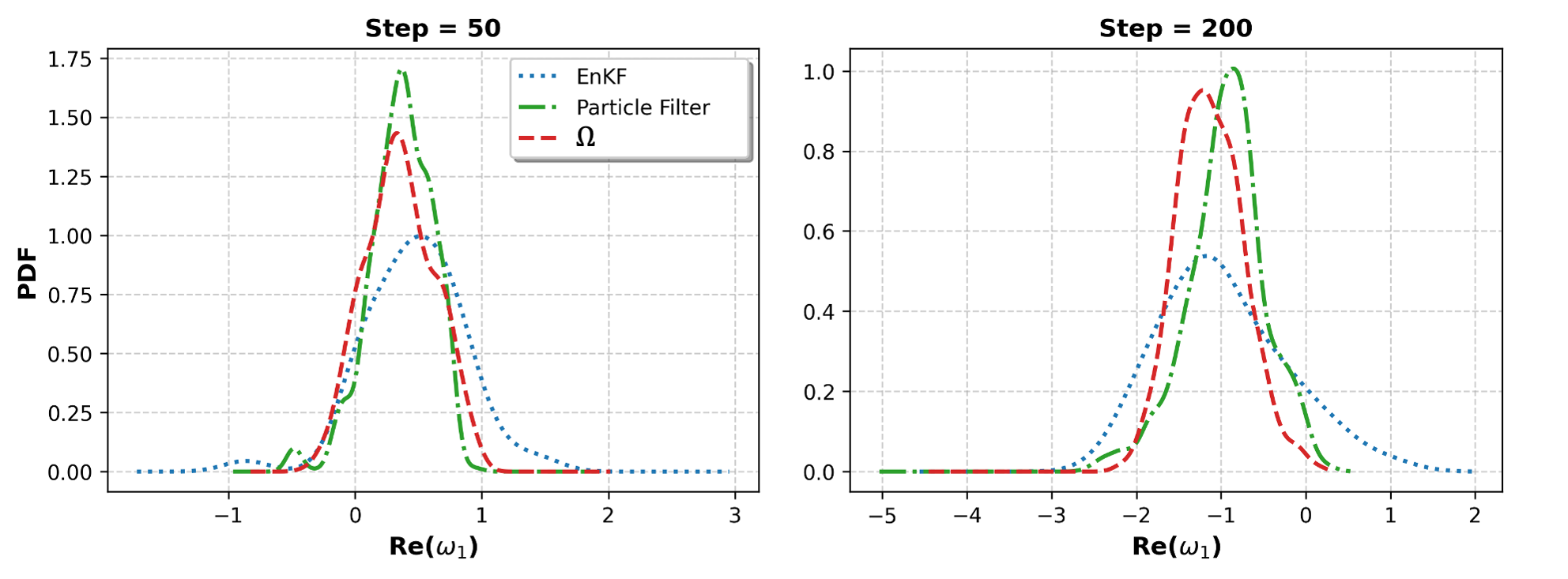}
\caption{
Conditional distribution of a representative hidden vorticity mode.
The asymmetric heavy-tailed structure associated with intermittent dynamics is accurately captured by $\Omega$ but is significantly underestimated by the EnKF.
}
\label{fig:topo_pdf}
\end{figure}

Table~\ref{tab:topo} summarizes the reconstruction errors for representative hidden variables. Improvements are observed consistently across both vorticity and thermodynamic modes. While the reduction in RMSE is moderate, the distributional improvements shown in Figure~\ref{fig:topo_pdf} indicate a substantially more accurate representation of posterior uncertainty. This distinction is particularly important in geophysical applications, where accurate characterization of rare events and hidden regime transitions is often more important than small improvements in posterior mean estimates.

\begin{table}[htbp]
\centering
\caption{RMSE for representative hidden variables of the topographic model.}
\label{tab:topo}
\renewcommand{\arraystretch}{1.2}
\begin{tabular}{lcccc}
\toprule
Method & $\omega_1$ & $T_1$ & $\omega_2$ & $T_2$ \\
\midrule
CGNS-EnKF      & 0.170 & 0.162 & 0.110 & 0.089 \\
Standard EnKF  & 0.121 & 0.123 & 0.073 & 0.061 \\
$\Omega$       & 0.110 & 0.102 & 0.067 & 0.053 \\
\bottomrule
\end{tabular}
\end{table}

These results demonstrate that the advantages of $\Omega$ persist beyond low-dimensional benchmark systems and extend to physically meaningful geophysical models with intermittent extreme events. The framework not only improves hidden-state reconstruction but also provides a substantially richer representation of posterior uncertainty, enabling more reliable quantification of rare events and unresolved dynamical variability.

\subsection{Robustness Under Model Misspecification: Stochastic Lorenz 96}
\label{sec:exp_l96}

The stochastic Lorenz 96 system \cite{lorenz1996predictability, arnold2013stochastic} is a widely used benchmark for studying chaos, predictability, and data assimilation in high-dimensional nonlinear systems. Unlike the dyad and topographic models considered previously, the Lorenz 96 system does not satisfy the conditional Gaussian structure required by the CGNS framework. It therefore provides a stringent test of the robustness of $\Omega$ when the forecast model is only an approximate surrogate of the true dynamics.

The stochastic Lorenz 96 model is given by
\begin{equation}
du_j
=
\left[
(u_{j+1}-u_{j-2})u_{j-1}
-u_j
+F
\right]dt
+
\sigma dW_j(t),
\qquad
j=1,\ldots,N,
\label{eq:L96}
\end{equation}
where indices are interpreted modulo $N$, $F=8$ is the forcing parameter that triggers strongly chaotic behavior of the system, and $\sigma=0.5$ is the stochastic forcing amplitude.

To mimic a partially observed setting, we partition the state into observed and hidden components:
\begin{equation}
\mathbf{x}
=
(u_1,u_3,u_5,\ldots),
\qquad
\mathbf{y}
=
(u_2,u_4,u_6,\ldots),
\end{equation}
where only the odd-indexed variables are observed at discrete times.

A major challenge arises because the hidden-state dynamics contain nonlinear interactions among multiple hidden variables. In particular, terms involving products such as $u_{2k+2}u_{2k-1}$ appear in the latent-state drift, violating the CGNS requirement that the hidden dynamics remain conditionally linear given the observed variables. Consequently, the Lorenz 96 system lies outside the exact CGNS class and cannot be treated analytically using the conditional Gaussian framework.

To construct a tractable baseline model, we replace the hidden-hidden quadratic interactions with local linear approximations and introduce an additional inflation noise term to account for unresolved nonlinear effects. This leads to the surrogate CGNS model
\begin{align}
dX_k
&=
\left[
-X_{k-1}Y_k
-(1+\gamma)X_k
+F
\right]dt
+
(\sigma+\sigma_s)dW_k^X,
\label{eq:surr_obs}
\\
dY_k
&=
\left[
-X_{k-1}Y_{k-1}
-Y_k
+X_kX_{k-1}
+F
\right]dt
+
\sigma dW_k^Y,
\label{eq:surr_hid}
\end{align}
where
\begin{equation}
\gamma = 0.1,
\qquad
F = 8,
\qquad
\sigma = 0.5,
\qquad
\sigma_s = 5.
\end{equation}

The surrogate model satisfies the CGNS structure because the observed-state drift is linear in the hidden variables and the hidden-state drift is linear in $(Y_{k-1},Y_k)$. Details of the derivation, the verification of the CGNS conditions, and the role of the inflation noise are provided in ~\ref{app:l96_surrogate}.

This example is particularly important because it evaluates the applicability of $\Omega$ beyond exact CGNS systems. In contrast to the previous examples, the goal is not to demonstrate performance under ideal conditions, but rather to assess whether the framework remains effective when the conditional Gaussian model is only an approximation of the true dynamics.

Figure~\ref{fig:l96_compare} compares posterior mean trajectories for representative observed and hidden variables. The standard EnKF is implemented in a perfect-model setting and therefore has access to the true Lorenz 96 dynamics. In contrast, both the CGNS-EnKF and $\Omega$ deliberately employ the approximate surrogate model \eqref{eq:surr_obs}--\eqref{eq:surr_hid}. The CGNS-EnKF uses the surrogate forecast model together with the standard linear Kalman analysis step, whereas $\Omega$ augments the same surrogate with the learned residual correction.
Visual inspection reveals that the CGNS-EnKF suffers from systematic biases caused by the surrogate approximation, particularly in the hidden variables, where unresolved nonlinear interactions play a dominant role. The residual correction in $\Omega$ significantly reduces these biases and produces trajectories that remain consistently closer to the truth than those obtained using the surrogate model alone. This result demonstrates that the learned operator is capable of correcting a substantial fraction of the structural model error introduced by the approximate CGNS representation. This behavior remains the same across all other variables. Finally, in Figure~\ref{fig:l96_Hov_compare}, we present the Hovmöller diagrams of $\Omega$, the EnKF using the perfect model, and the truth. $\Omega$ closely captures the system’s chaotic and wave propagating behavior and agrees well with the truth.

\begin{figure}[htbp]
\centering
\includegraphics[width=\linewidth]{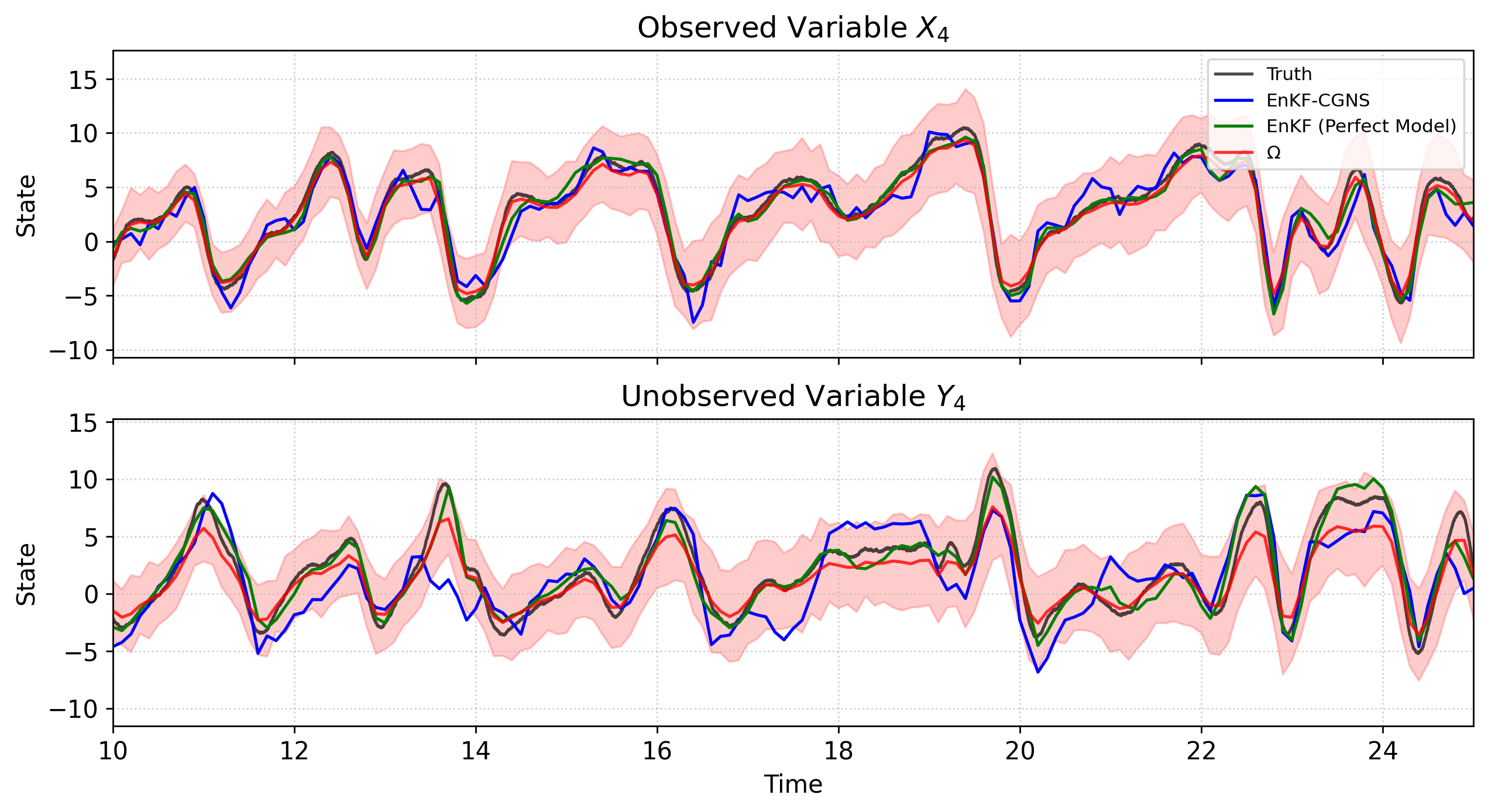}
\caption{
Posterior mean trajectories for representative observed and hidden variables of the stochastic Lorenz 96 system.
The residual correction in $\Omega$ substantially improves the reconstruction of hidden variables relative to the surrogate-only CGNS-EnKF and is slosely comparble with the EnKF applied on the on the perfect model. The behavior is the same across all other variables $X_i, \,Y_i, \,i=1,\cdots,10$. 
}
\label{fig:l96_compare}
\end{figure}
\begin{figure}[h]
\centering
\includegraphics[width=.5\linewidth]{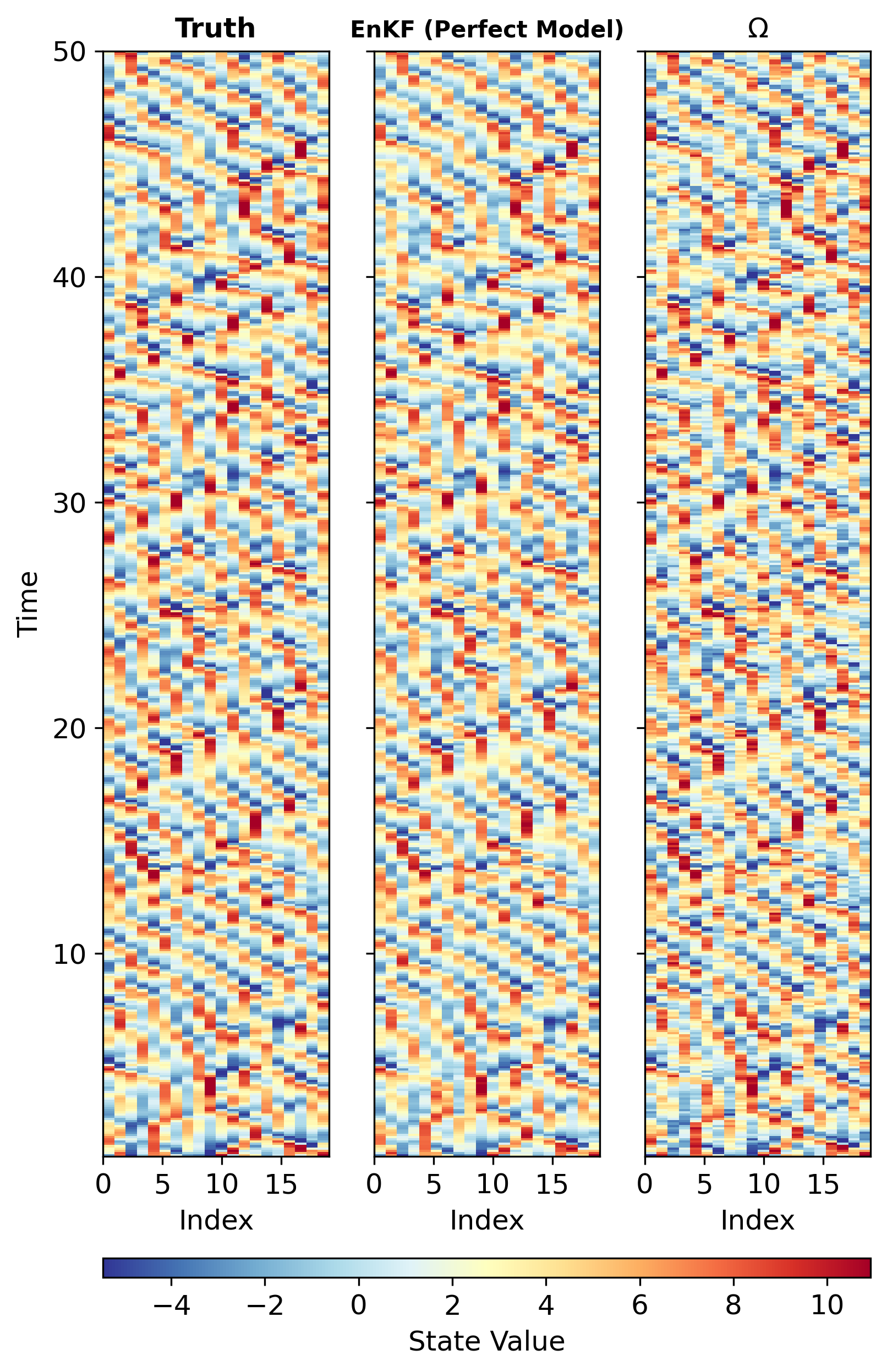}
\caption{
 Hovm\H{o}ller diagram of mean trajectories of the stochastic Lorenz 96 system. The $\Omega$ results are very close to the truth and the EnKF applied on the perfect model. 
}
\label{fig:l96_Hov_compare}
\end{figure}

Table~\ref{tab:nrmse} reports RMSE values for representative hidden variables. Relative to the CGNS-EnKF baseline, $\Omega$ consistently improves the reconstruction of all hidden variables, reducing the error by approximately $8$--$15\%$. Although the perfect-model EnKF achieves the lowest RMSE overall, this comparison should be interpreted in light of the fundamentally different forecast assumptions. The EnKF benefits from exact knowledge of the governing equations, whereas $\Omega$ relies only on a simplified surrogate model constructed to satisfy the CGNS structure. From this perspective, the relatively small performance gap between $\Omega$ and the perfect-model EnKF highlights the robustness of the framework under substantial model misspecification.

\begin{table}[htbp]
\centering
\caption{RMSE for representative hidden variables of the stochastic Lorenz--96 system.}
\label{tab:nrmse}
\renewcommand{\arraystretch}{1.2}
\begin{tabular}{lccc}
\toprule
Method & $Y_1$ & $Y_4$ & $Y_9$ \\
\midrule
CGNS-EnKF                  & 0.525 & 0.608 & 0.633 \\
Standard EnKF (Perfect Model) & 0.451 & 0.459 & 0.440 \\
$\Omega$                   & 0.474 & 0.548 & 0.540 \\
\bottomrule
\end{tabular}
\end{table}

More importantly, the Lorenz 96 experiment highlights several practical advantages of the $\Omega$ framework. First, the surrogate CGNS model provides an analytically tractable representation of the hidden variables and avoids propagating the full nonlinear latent dynamics within the assimilation procedure. This significantly reduces computational cost relative to methods that require repeated integration of the complete nonlinear system. Second, unlike many operational implementations of the EnKF, which often rely on empirically tuned covariance inflation and localization strategies to maintain filter performance, the $\Omega$ framework does not require such ad hoc corrections. Instead, the learned residual operator provides a systematic mechanism for correcting both non-Gaussian posterior structures and structural model deficiencies.

The Lorenz 96 results demonstrate that the applicability of $\Omega$ extends well beyond exact CGNS systems. Even when the conditional Gaussian representation is only a crude approximation of the true dynamics, the learned residual correction remains effective and systematically improves posterior recovery. Since the surrogate model employed here is intentionally simple, substantial additional gains are likely achievable through more sophisticated CGNS constructions, such as local higher-order approximations or Koopman-based lifting approaches. This suggests that $\Omega$ provides a flexible framework that combines the efficiency of analytical conditional-Gaussian surrogates with data-driven correction of model error, making it particularly attractive for high-dimensional nonlinear systems where perfect forecast models are unavailable and computational efficiency is a primary concern.

	\section{Conclusions}
	\label{sec:conclud}
	We introduced $\Omega$, an operator-based mixture ensemble framework for non-Gaussian data assimilation in partially observed nonlinear dynamical systems. The framework combines three ingredients: a CGNS surrogate forecast, a learned residual score operator, and annealed Langevin sampling in an augmented state space. Together, these components provide a practical mechanism for enriching the Gaussian EnKF analysis while avoiding the weight degeneracy of particle filters and the supervised training requirements commonly associated with score-based data assimilation methods.
	
	The central idea of $\Omega$ is to reformulate data assimilation in an augmented state space consisting of the observed variables and the conditional means of the latent variables under a CGNS surrogate. This lifting preserves the essential posterior information needed for non-Gaussian correction while avoiding direct sampling in the potentially high-dimensional latent space. As a result, nonlinear score corrections can be learned and applied in a lower-dimensional representation, while latent uncertainty is retained analytically through the conditional Gaussian structure. The resulting Gaussian-mixture analysis naturally captures non-Gaussian posterior features in both observed and latent variables.
	
	On the theoretical side, we established a sequence of consistency and convergence results linking denoising score matching, Noise2Noise training on surrogate ensembles, Langevin sampling, and Gaussian-mixture posterior recovery. These results clarify the role of the residual operator, quantify the effects of finite-ensemble sampling and surrogate misspecification, and provide conditions under which the recovered posterior converges to the true posterior distribution. The analysis also explains why the framework remains scalable with respect to the latent dimension: the nonlinear correction is performed in the augmented filtering variables, while integration over the latent state is carried out analytically within each conditional Gaussian component.
	
	Numerical experiments on three nonlinear systems demonstrate that $\Omega$ consistently improves posterior reconstruction and latent-state recovery relative to both the standard EnKF and the CGNS baseline. The framework remains effective not only when the underlying dynamics possess an exact CGNS structure, but also when the surrogate must be constructed through local linearization and model inflation, suggesting a broader range of applicability than exact conditional Gaussian systems alone.
	
	Several important challenges remain. The residual operator is currently trained offline, motivating future work on online and adaptive learning strategies capable of responding to evolving dynamics and observation networks. Surrogate misspecification introduces an irreducible approximation floor, and developing adaptive mechanisms for surrogate refinement is a natural next step. In addition, extending $\Omega$ to truly high-dimensional geophysical applications will require combining the augmented-state formulation with localization and scalable operator-learning architectures.
	
	More broadly, $\Omega$ illustrates how analytic structure and learned corrections can be integrated within a single data assimilation framework. Rather than replacing model-based inference with machine learning, the framework uses learning to correct precisely those aspects of the posterior that are inaccessible to the Gaussian approximation while retaining the computational efficiency and interpretability of the underlying dynamical model. This perspective offers a promising pathway toward scalable non-Gaussian data assimilation in complex multiscale systems.

    \section{Acknowledgment} N.C. is supported by the Office of Naval Research under Award No. N00014-24-1-2244 and by the Army Research Office under Award No. W911NF-23-1-0118. P.B. is partially supported by the aforementioned ONR award and partially by the University of Wisconsin-Madison Office of the Vice Chancellor for Research, with funding from the Wisconsin Alumni Research Foundation.

	\bibliography{ref}
	
	\appendix
	
	\section{Proof of Theorem~\ref{thm:dsm_consistency} }
	\label{app:proof_dsm}
	
	\begin{proof}
		\textbf{DSM identity.}
		Write $\tilde{\mathbf{V}}_k = \mathbf{V}_k + \sigma\boldsymbol{\epsilon}$
		with $\boldsymbol{\epsilon} \sim \mathcal{N}(0, \mathbf{I})$ and
		let $\mathbf{s}_{\mathrm{tot},k} = \mathbf{s}_{\mathrm{EnKF},k} +
		\mathbf{s}_\theta$. The DSM objective \eqref{eq:dsm_loss} is
		\begin{equation}
			\mathcal{L}(\theta) = \mathbb{E}_{\mathbf{V}_k,
				\boldsymbol{\epsilon}}\!\left[\sigma^2 \left\|
			\mathbf{s}_{\mathrm{tot},k}(\tilde{\mathbf{V}}_k) +
			\frac{\boldsymbol{\epsilon}}{\sigma}\right\|^2\right].
			\label{eq:proof_loss}
		\end{equation}
		Since $\nabla_{\tilde{\mathbf{V}}_k} \log
		p_\sigma(\tilde{\mathbf{V}}_k \mid \mathbf{V}_k) =
		-\boldsymbol{\epsilon}/\sigma$, taking the conditional expectation
		given $\tilde{\mathbf{V}}_k$ and applying Stein's
		identity~\cite{stein1981estimation} $\left(\mathbb{E}_{\tilde{\bo V} \mid \bo  V} \left[(\tilde{\bo V} - \bo V)\, h(\tilde{\bo V})\right] = \sigma^2 \mathbb{E}_{\tilde{\bo V} \mid \bo V} \left[\nabla_{\tilde{\bo V}} h(\tilde{\bo V})\right] \text{ with $h$ being weakly differentiable}\right)$ yields
		\begin{equation}
			\mathbb{E}\!\left[-\frac{\boldsymbol{\epsilon}}{\sigma}
			\;\middle|\; \tilde{\mathbf{V}}_k\right] =
			\mathbb{E}\!\left[\nabla_{\tilde{\mathbf{V}}_k} \log
			p_\sigma(\tilde{\mathbf{V}}_k \mid \mathbf{V}_k)
			\;\middle|\; \tilde{\mathbf{V}}_k\right] =
			\nabla_{\tilde{\mathbf{V}}_k} \log
			p_\sigma(\tilde{\mathbf{V}}_k).
			\label{eq:stein_identity}
		\end{equation}
		Expanding the squared norm in \eqref{eq:proof_loss} and using
		\eqref{eq:stein_identity} to evaluate the cross-term:
		\begin{align}
			\mathcal{L}(\theta)
			&= \mathbb{E}_{p_\sigma}\!\left[\sigma^2\left\| \mathbf{s}_{\mathrm{tot},k} -  \nabla\log p_\sigma\right\|^2\right] \notag \\
			&\quad + 2\,\mathbb{E}_{p_\sigma}\!\left[\sigma^2\left\langle \mathbf{s}_{\mathrm{tot},k}(\tilde{\mathbf{V}}_k) - \nabla\log p_\sigma(\tilde{\mathbf{V}}_k),\,
			\nabla\log p_\sigma(\tilde{\mathbf{V}}_k) +
			\frac{\boldsymbol{\epsilon}}{\sigma}\right\rangle\right]
			\notag \\
			&\quad + \mathbb{E}_{p_\sigma}\!\left[\sigma^2\left\| \nabla\log p_\sigma(\tilde{\mathbf{V}}_k)\right\|^2\right].
			\label{eq:expand}
		\end{align}
		Since $s_{\rm tot}(\tilde V) - \nabla\log p_\sigma(\tilde V)$ is measurable w.r.t. $\tilde V$, conditioning on $\tilde{\mathbf{V}}_k$ results in vanishing the cross-term as:
		\begin{equation}
			\mathbb{E}\!\left[\nabla\log p_\sigma(\tilde{\mathbf{V}}_k) +
			\frac{\boldsymbol{\epsilon}}{\sigma} \;\middle|\;
			\tilde{\mathbf{V}}_k\right] =
			\nabla\log p_\sigma(\tilde{\mathbf{V}}_k) -
			\nabla\log p_\sigma(\tilde{\mathbf{V}}_k) = \mathbf{0},
			\label{eq:cross_term}
		\end{equation}
		by \eqref{eq:stein_identity}. Setting $C = \mathbb{E}_{p_\sigma}[\sigma^2\|\nabla\log
		p_\sigma(\tilde{\mathbf{V}}_k)\|^2]$, which is finite by the
		finite Fisher information assumption \eqref{eq:fisher_info} and
		the smoothing inequality $\mathcal{I}(p_\sigma) \leq
		\mathcal{I}(p_{\mathrm{aug}})$, establishes
		\eqref{eq:dsm_identity}.
		
		\textbf{ Uniqueness and characterization of the minimizer.}
		The first term in \eqref{eq:expand} is the squared $L^2(p_\sigma)$
		norm of $\mathbf{s}_{\mathrm{tot},k} - \nabla\log p_\sigma$, which
		is a strictly convex functional of $\mathbf{s}_\theta$ (as an
		element of $L^2(p_\sigma)$). Its unique minimizer satisfies
		$\mathbf{s}_{\mathrm{tot},k} = \nabla\log p_\sigma$
		$p_\sigma$-a.e., establishing \eqref{eq:dsm_minimizer}.
		
		\textbf{Small-noise limit.}
		The limit follows from standard properties of Gaussian mollification. Under finite Fisher information \eqref{eq:fisher_info}, the Gaussian-smoothed densities
		\[
		p_\sigma =
		p_{\mathrm{aug}}* \mathcal N(0,\sigma^2I)
		\]
		satisfy the score continuity property
		\[
		\nabla\log p_\sigma \to \nabla\log p_{\mathrm{aug}} \quad
		\text{in }L^2(p_{\mathrm{aug}})
		\]
		as \(\sigma\to0\); see, e.g., \cite{chen2023score}. Combining this convergence with the minimizer identity \eqref{eq:dsm_minimizer},
		\[
		\mathbf s_{\mathrm{EnKF},k} + \mathbf s_{\theta^*}
		= \nabla\log p_\sigma,
		\]
		and applying the triangle inequality yields \eqref{eq:residual_convergence}.
	\end{proof}
	
	\section{Proof of Theorem~\ref{thm:n2n_consistency}}
    \label{app:n2n_proof}
	\begin{proof}
		
		\textbf{Parts (i) and (ii).}
		The argument is identical to the proof of
		Theorem~\ref{thm:dsm_consistency}, with $p_{\mathrm{aug}}$ and $p_\sigma$ replaced throughout by $q_{\mathrm{aug}}$ and $q_\sigma$. Applicability of the Stein identity requires $h = \mathbf{s}_{\mathrm{EnKF},k} + \mathbf{s}_\theta \in L^2(q_\sigma)$, which is given by hypothesis. Since $q_{\mathrm{aug}}$ is a finite Gaussian mixture with finite Fisher information, $q_\sigma$ satisfies all regularity conditions needed in the proof of Theorem~\ref{thm:dsm_consistency}.
		
		\textbf{Part (iii): two-source error decomposition.}
		By Remark~\ref{rem:abs_cont}, $p_{\mathrm{aug}} \ll q_\sigma$ for
		$\sigma > 0$, so the minimizer identity \eqref{eq:n2n_minimizer}
		holds $p_{\mathrm{aug}}$-almost everywhere.  Therefore, for $p_{\mathrm{aug}}$-almost every $\tilde{\mathbf{V}}_k$, \begin{equation}
			\mathbf{s}_{\theta^*}(\tilde{\mathbf{V}}_k)
			- \bigl(\nabla\log p_\sigma(\tilde{\mathbf{V}}_k) - \mathbf{s}_{\mathrm{EnKF},k}(\tilde{\mathbf{V}}_k)\bigr)\;=\;
			\nabla\log q_\sigma(\tilde{\mathbf{V}}_k) - \nabla\log p_\sigma(\tilde{\mathbf{V}}_k).
			\label{eq:pf_iii_eq}
		\end{equation}
		Taking the $L^2(p_{\mathrm{aug}})$ norm of \eqref{eq:pf_iii_eq} and applying the triangle inequality with the intermediate term $\nabla\log p_{\mathrm{CGNS},\sigma}$:
		\begin{align}
			&\bigl\|\mathbf{s}_{\theta^*}
			- \bigl(\nabla\log p_\sigma
			- \mathbf{s}_{\mathrm{EnKF},k}\bigr)
			\bigr\|_{L^2(p_{\mathrm{aug}})}
			\notag \\[4pt]
			&=\; \bigl\|\nabla\log q_\sigma
			- \nabla\log p_\sigma \bigr\|_{L^2(p_{\mathrm{aug}})}
			\label{eq:pf_iii_norm} \\[4pt]
			&\leq\; \underbrace{
				\bigl\|\nabla\log q_\sigma
				- \nabla\log p_{\mathrm{CGNS},\sigma}
				\bigr\|_{L^2(p_{\mathrm{aug}})}
			}_{\varepsilon_{\mathrm{samp}}(N)}
			\;+\; \underbrace{
				\bigl\|\nabla\log p_{\mathrm{CGNS},\sigma} - \nabla\log p_\sigma \bigr\|_{L^2(p_{\mathrm{aug}})}
			}_{\varepsilon_{\mathrm{CGNS}}},
			\label{eq:pf_iii_triangle}
		\end{align}
		which is \eqref{eq:surrogate_error}.  The equality
		\eqref{eq:pf_iii_norm} is an exact equality (not merely an
		inequality) because the absolute continuity established in
		Remark~\ref{rem:abs_cont} ensures that \eqref{eq:pf_iii_eq}
		holds $p_{\mathrm{aug}}$-a.e.
		\textbf{Part (iv): Elimination of sampling error as $N \to \infty$.}
		
		We introduce two standing assumptions for this part, which hold
		under standard EnKF regularity conditions and are natural given
		the finite-ensemble setting.
		
		\begin{assumption}[Bounded ensemble]\label{ass:bounded_ensemble}
			The ensemble members satisfy
			$\sup_{N,i}\|\mathbf{V}_{f,k}^{(i)}\| \leq R < \infty$
			uniformly in $N$.
		\end{assumption}
		
		\begin{assumption}[Bounded density ratio]\label{ass:bounded_ratio}
			The Radon--Nikodym derivative
			$w_N := dp_{\mathrm{aug}}/dq_{\sigma}^{(N)}$
			satisfies $\sup_N \|w_N\|_{L^\infty} \leq M < \infty$.
		\end{assumption}
		
		\textit{Step 1: KL convergence.}
		By standard EnKF consistency~\cite{mandel2011convergence}, the
		empirical mixture $q_{\mathrm{aug}}^{(N)}$ converges weakly to
		$p_{\mathrm{CGNS}}$ as $N\to\infty$. Under
		Assumption~\ref{ass:bounded_ensemble}, every ensemble member lies
		in the ball of radius $R$, so both $q_{\mathrm{aug}}^{(N)}$ and
		$p_{\mathrm{CGNS}}$ are Gaussian mixtures with components bounded
		in $\ell^2$-norm. Hence the log-likelihood ratio
		$\log(q_{\mathrm{aug}}^{(N)}/p_{\mathrm{CGNS}})$ is bounded
		uniformly in $N$, and weak convergence implies KL convergence:
		\begin{equation}
			D_{\mathrm{KL}}\!\left(
			q_{\mathrm{aug}}^{(N)} \,\|\, p_{\mathrm{CGNS}}
			\right) \to 0
			\quad \text{as } N\to\infty.
			\label{eq:pf_kl_conv}
		\end{equation}
		
		\textit{Step 2: Contraction under Gaussian smoothing.}
		Gaussian convolution is a Markov kernel, so the data-processing
		inequality gives
		\begin{equation}
			D_{\mathrm{KL}}\!\left(
			q_{\sigma}^{(N)} \,\|\, p_{\mathrm{CGNS},\sigma}
			\right)
			\;\leq\;
			D_{\mathrm{KL}}\!\left(
			q_{\mathrm{aug}}^{(N)} \,\|\, p_{\mathrm{CGNS}}
			\right)
			\;\to\; 0.
			\label{eq:pf_data_proc}
		\end{equation}
		\textit{Step 3: Score convergence in $L^2(p_{\mathrm{aug}})$.}
		Since $\phi_\sigma := \mathcal{N}(0,\sigma^2\mathbf{I})$ is a
		$C^\infty$ kernel with rapidly decaying derivatives, and
		$q_{\mathrm{aug}}^{(N)} \Rightarrow p_{\mathrm{CGNS}}$ weakly
		(Step~1), standard properties of convolution imply that
		$q_\sigma^{(N)} = q_{\mathrm{aug}}^{(N)} * \phi_\sigma$ and
		$p_{\mathrm{CGNS},\sigma} = p_{\mathrm{CGNS}} * \phi_\sigma$
		are both $C^\infty$ densities, and that
		\begin{equation}
			q_\sigma^{(N)} \to p_{\mathrm{CGNS},\sigma},
			\qquad
			\nabla q_\sigma^{(N)} \to \nabla p_{\mathrm{CGNS},\sigma}
			\label{eq:conv_grad_unif}
		\end{equation}
		uniformly on compact subsets of $\mathbb{R}^{d_X + d_Y}$.
		Since both densities are strictly positive everywhere
		(Remark~\ref{rem:abs_cont}), dividing by the density is
		continuous, and \eqref{eq:conv_grad_unif} gives pointwise
		convergence of the scores:
		\begin{equation}
			\nabla\log q_\sigma^{(N)}(\mathbf{v})
			\;\to\;
			\nabla\log p_{\mathrm{CGNS},\sigma}(\mathbf{v})
			\quad \text{for every } \mathbf{v}.
			\label{eq:score_ptwise}
		\end{equation}
		By the Tweedie--Miyasawa formula~\cite{efron2011tweedie,
			raphan2011least}, the score of any $\sigma$-smoothed Gaussian
		mixture with components bounded by $R$
		(Assumption~\ref{ass:bounded_ensemble}) satisfies the uniform
		bound
		\begin{equation}
			\bigl\|\nabla\log q_{\sigma}^{(N)}
			- \nabla\log p_{\mathrm{CGNS},\sigma}
			\bigr\|_{L^\infty}
			\;\leq\; 2C_\sigma,
			\qquad
			C_\sigma := \frac{R}{\sigma^2}
			+ \frac{(d_X+d_Y)^{1/2}}{\sigma},
			\label{eq:score_diff_Linfty}
		\end{equation}
		uniformly in $N$. Combining pointwise convergence
		\eqref{eq:score_ptwise} with the uniform bound
		\eqref{eq:score_diff_Linfty} and applying dominated convergence
		to the $L^2(q_\sigma^{(N)})$ norm yields
		\begin{equation}
			\bigl\|\nabla\log q_\sigma^{(N)}
			- \nabla\log p_{\mathrm{CGNS},\sigma}
			\bigr\|_{L^2(q_\sigma^{(N)})}^2
			=
			\int \bigl\|\nabla\log q_\sigma^{(N)}
			- \nabla\log p_{\mathrm{CGNS},\sigma}
			\bigr\|^2 dq_\sigma^{(N)}
			\;\to\; 0,
			\label{eq:score_conv_q}
		\end{equation}
		where dominated convergence applies because the integrand is bounded by $4C_\sigma^2$ uniformly in $N$ and the measures $q_\sigma^{(N)}$ converge weakly to $p_{\mathrm{CGNS},\sigma}$.
		
		\textit{Step 4: Change of measure.}
		Using Assumption~\ref{ass:bounded_ratio},
		\begin{align}
			\varepsilon_{\mathrm{samp}}(N)^2
			&=
			\bigl\|\nabla\log q_\sigma^{(N)}
			- \nabla\log p_{\mathrm{CGNS},\sigma}
			\bigr\|_{L^2(p_{\mathrm{aug}})}^2
			\notag \\
			&=
			\int \bigl\|\nabla\log q_\sigma^{(N)}
			- \nabla\log p_{\mathrm{CGNS},\sigma}
			\bigr\|^2 w_N\, dq_\sigma^{(N)}
			\notag \\
			&\leq\;
			M \cdot
			\bigl\|\nabla\log q_\sigma^{(N)}
			- \nabla\log p_{\mathrm{CGNS},\sigma}
			\bigr\|_{L^2(q_\sigma^{(N)})}^2
			\;\to\; 0,
			\label{eq:change_of_measure}
		\end{align}
		where the convergence uses \eqref{eq:score_conv_q}.
		Substituting $\varepsilon_{\mathrm{samp}}(N)\to 0$ into
		\eqref{eq:surrogate_error} yields \eqref{eq:n2n_convergence}.
		
		\begin{remark}[Verification of assumptions]
			Assumption~\ref{ass:bounded_ensemble} holds in any finite-time EnKF implementation where the ensemble spread is controlled, and is standard in practice. Assumption~\ref{ass:bounded_ratio}
			follows from Assumption~\ref{ass:bounded_ensemble} whenever
			$p_{\mathrm{aug}}$ has sub-Gaussian tails: since all smoothed densities are then bounded away from zero on compact sets uniformly in $N$, the density ratio $w_N$ remains bounded. Both assumptions can be verified explicitly for a given forecast model.
		\end{remark}
		
		\textbf{Part (v): full convergence when CGNS is exact.}
		If the true system exactly satisfies
		\eqref{eq:cgns_obs}--\eqref{eq:cgns_lat}, then
		$p_{\mathrm{CGNS}} = p_{\mathrm{aug}}$ by definition, so
		$\varepsilon_{\mathrm{CGNS}} = 0$.  Substituting into
		\eqref{eq:n2n_convergence} gives \eqref{eq:full_conv}.
		
	\end{proof}
	
	\section {Conditions for
		Theorem~\ref{thm:langevin_surrogate}}
	\label{rem:dissipativity}
	Hypothesis (iii) in Theorem~\ref{thm:langevin_surrogate}, strong log-concavity of $q_\sigma$, is a genuine condition on the surrogate, not an automatic consequence
	of the Gaussian mixture structure. Three practically verifiable sufficient conditions are given below. Note that the implication runs as follows: strong log-concavity ($\nabla^2(-\log q_\sigma) \geq m\mathbf{I}$) implies a log-Sobolev inequality with constant $\rho_{\mathrm{LS}} \geq m$ via the Bakry--\'{E}mery criterion~\cite{Carlen1986}, not the converse.
	
	\paragraph{Large smoothing noise.}
	For fixed $q_{\mathrm{aug}}$, the Gaussian kernel
	$\mathcal{N}(0,\sigma^2\mathbf{I})$ is itself strongly log-concave with parameter $1/\sigma^2$. By the Pr\'{e}kopa--Leindler inequality, the convolution $q_\sigma = q_{\mathrm{aug}} * \mathcal{N}(0,\sigma^2\mathbf{I})$ is log-concave whenever $q_{\mathrm{aug}}$ is log-concave~\cite{prekopa1973logarithmic}. Moreover, for any $q_{\mathrm{aug}}$, the convolved density $q_\sigma$ satisfies
	\begin{equation}
		\nabla^2(-\log q_\sigma)(\mathbf{v})
		\;\geq\;
		\frac{1}{\sigma^2}\mathbf{I}
		\quad\text{pointwise},
		\label{eq:hessian_lower}
	\end{equation}
	for $\sigma$ large enough relative to the spread of
	$q_{\mathrm{aug}}$ (see, e.g.,~\cite{wibisono2019proximal}),
	so hypothesis (iii) holds with $m = 1/\sigma^2$.
	This is practically the regime of early annealing steps
	(large $\sigma_\ell$), where the Langevin chain is initialized.
	
	\paragraph{Well-separated Gaussian components.}
	If the ensemble means $\{\mathbf{V}_{f,k}^{(i)}\}$ satisfy the
	separation condition
	$\|\mathbf{V}_{f,k}^{(i)} - \mathbf{V}_{f,k}^{(j)}\|
	\leq c\,\sigma$ for all $i \neq j$ and a numerical constant
	$c > 0$, then $q_\sigma$ is approximately a single Gaussian and
	is strongly log-concave~\cite{wibisono2019proximal}. In the
	ensemble Kalman setting this holds whenever the ensemble spread
	is small relative to $\sigma$.

	\section{Proof of Theorem~\ref{thm:langevin_surrogate}}
	\label{app:langevin_proof}
	\begin{proof}
		
		\textbf{Part (i): Weak convergence.}
		This follows immediately from Part (ii) by taking $\ell \to \infty$
		in \eqref{eq:wasserstein_bound}: the first term $\rho^\ell
		W_2(\mu^0, q_\sigma) \to 0$ since $\rho \in (0,1)$, and the
		bias $C_2\alpha/(1-\rho)$ is driven to zero by taking
		$\alpha \to 0$.
		
		\textbf{Part (ii): Wasserstein-2 bound.}
		By hypothesis (iv) and Remark~\ref{rem:abs_cont}, the total score
		satisfies $\mathbf{s}_{\mathrm{EnKF},k} + \mathbf{s}_{\theta^*}
		= \nabla\log q_\sigma$ both $q_\sigma$-a.e.\ and
		$p_{\mathrm{aug}}$-a.e.\ Substituting into the Langevin recursion
		\eqref{eq:langevin} gives the Unadjusted Langevin Algorithm (ULA)
		targeting $q_\sigma$:
		\begin{equation}
			\mathbf{V}_k^{(i,\ell+1)}
			= \mathbf{V}_k^{(i,\ell)}
			+ \frac{\alpha}{2}\,
			\nabla_{\mathbf{V}_k}\log q_\sigma\!
			\left(\mathbf{V}_k^{(i,\ell)}\right)
			+ \sqrt{\alpha}\,\boldsymbol{\eta}_\ell^{(i)},
			\label{eq:ula}
		\end{equation}
		which is the Euler--Maruyama discretization of the overdamped
		Langevin SDE with invariant measure $q_\sigma$. Under hypotheses
		(ii) and (iii), Durmus \& Moulines (2019),
		Theorem~6~\cite{durmus2019analysis} gives the one-step
		$W_2$-contraction:
		\begin{equation}
			W_2\!\left(\mu^{\ell+1},\, q_\sigma\right)
			\;\leq\;
			\rho\,W_2\!\left(\mu^\ell,\, q_\sigma\right)
			+ C_2\,\alpha,
			\label{eq:w2_onestep}
		\end{equation}
		with contraction factor $\rho = 1 - m\alpha + L^2\alpha^2/2
		\in (0,1)$ for $\alpha < 2m/L^2$, and constant $C_2$ depending
		on $L$, $m$, and dimension. Iterating \eqref{eq:w2_onestep}
		$\ell$ times gives \eqref{eq:wasserstein_bound}. For the bias
		rate \eqref{eq:bias_rate}, note that for $\alpha \leq m/L^2$:
		\begin{equation}
			1 - \rho
			= m\alpha - \tfrac{1}{2}L^2\alpha^2
			\;\geq\; \tfrac{1}{2}m\alpha,
			\label{eq:rho_lower}
		\end{equation}
		so
		\begin{equation}
			\frac{C_2\,\alpha}{1 - \rho}
			\;\leq\;
			\frac{C_2\,\alpha}{\tfrac{1}{2}m\alpha}
			= \frac{2C_2}{m},
			\label{eq:bias_bounded}
		\end{equation}
		which is $\mathcal{O}(1/m)$ and independent of $\alpha$. The bias therefore does not shrink as $\mathcal{O}(\sqrt{\alpha})$ but rather remains bounded by $2C_2/m$ for all small $\alpha$, and the overall $W_2$ error is driven to zero by first taking $\ell \to \infty$ (exponential decay of the transient term) and then $\alpha \to 0$ (which drives $C_2\alpha \to 0$ in the one-step bound \eqref{eq:w2_onestep}, and hence the steady-state bias to zero).
		
		\textbf{Part (iii): Convergence toward the true posterior.}
		Under exact CGNS ($\varepsilon_{\mathrm{CGNS}} = 0$),
		Theorem~\ref{thm:n2n_consistency}(v) gives
		\begin{equation}
			\bigl\|\mathbf{s}_{\mathrm{EnKF},k}
			+ \mathbf{s}_{\theta^*}
			- \nabla\log p_{\mathrm{aug}}
			\bigr\|_{L^2(p_{\mathrm{aug}})}
			\to 0
			\quad\text{as }N \to \infty.
			\label{eq:score_conv_true}
		\end{equation}
		Let $\tilde{q}_\sigma^{(N)}$ denote the invariant measure of the
		ULA using the $N$-approximate total score. By
		Assumption~\ref{ass:bounded_ratio} and mutual absolute continuity
		(Remark~\ref{rem:abs_cont}), convergence in $L^2(p_{\mathrm{aug}})$
		implies convergence in $L^2(q_\sigma)$, so the $W_2$ perturbation
		bound for ULA (Durmus \& Moulines 2019,
		Proposition~4~\cite{durmus2019analysis}) gives:
		\begin{equation}
			W_2\!\left(\tilde{q}_\sigma^{(N)},\, q_\sigma\right)
			\;\leq\;
			\frac{1}{m}\,
			\bigl\|\mathbf{s}_{\mathrm{EnKF},k}
			+ \mathbf{s}_{\theta^*}^{(N)}
			- \nabla\log q_\sigma
			\bigr\|_{L^2(q_\sigma)}
			\;\to\; 0
			\quad\text{as }N \to \infty.
			\label{eq:w2_perturbation}
		\end{equation}
		Therefore $\tilde{q}_\sigma^{(N)} \xrightarrow{w} q_\sigma$ as
		$N \to \infty$. Separately, since
		$q_\sigma = p_{\mathrm{aug}} *
		\mathcal{N}(0,\sigma^2\mathbf{I})$ and
		$\mathcal{N}(0,\sigma^2\mathbf{I}) \xrightarrow{w} \delta_0$ as
		$\sigma \to 0$, standard properties of weak convergence under
		convolution give $q_\sigma \xrightarrow{w} p_{\mathrm{aug}}$.
		Combining with the $W_2$ bound \eqref{eq:wasserstein_bound}
		(with $\alpha \to 0$ driving the one-step bias $C_2\alpha \to 0$)
		and applying the triangle inequality for $W_2$ gives
		\eqref{eq:langevin_true_conv}.
		
	\end{proof}

	\section{On the affinity of CGNS and recovering the true joint}
	
	\begin{lemma}[Affinity of the EnKF--CGNS Baseline Score]
		\label{lem:enkf_affine}
		The Gaussian baseline score $\mathbf{s}_{\mathrm{EnKF},k}$ defined
		in \eqref{eq:enkf_score} is affine in $\mathbf{V}_k$:
		\begin{equation}
			\mathbf{s}_{\mathrm{EnKF},k}(\mathbf{V}_k) =
			\mathbf{A}_k \mathbf{V}_k + \mathbf{b}_k,
			\label{eq:enkf_affine_form}
		\end{equation}
		where
		\begin{align}
			\mathbf{A}_k &= -\mathbf{P}_{f,k}^{-1} -
			{\mathbf{H}}^\top\mathbf{R}_{\mathrm{obs}}^{-1}
			{\mathbf{H}}
			\in \mathbb{R}^{(d_X+d_Y)\times(d_X+d_Y)},
			\label{eq:Ak} \\
			\mathbf{b}_k &= \mathbf{P}_{f,k}^{-1}\bar{\mathbf{V}}_{f,k} +
			{\mathbf{H}}^\top\mathbf{R}_{\mathrm{obs}}^{-1}
			\mathbf{Z}_k
			\in \mathbb{R}^{d_X+d_Y}.
			\label{eq:bk}
		\end{align}
		In particular, $\mathbf{s}_{\mathrm{EnKF},k} \in \mathcal{G}$,
		where $\mathcal{G}$ is the class of affine functions on
		$\mathbb{R}^{d_X + d_Y}$, and $\mathbf{A}_k$ is negative
		definite since both $\mathbf{P}_{f,k}^{-1}$ and
		${\mathbf{H}}^\top\mathbf{R}_{\mathrm{obs}}^{-1}
		{\mathbf{H}}$ are positive semi-definite.
	\end{lemma}
	
	\begin{proof}
		Direct substitution of \eqref{eq:enkf_score}:
		\begin{align*}
			\mathbf{s}_{\mathrm{EnKF},k}(\mathbf{V}_k)
			&= -\mathbf{P}_{f,k}^{-1}(\mathbf{V}_k -
			\bar{\mathbf{V}}_{f,k}) +
			{\mathbf{H}}^\top\mathbf{R}_{\mathrm{obs}}^{-1}
			(\mathbf{Z}_k - {\mathbf{H}}\mathbf{V}_k) \\
			&= \underbrace{\left(-\mathbf{P}_{f,k}^{-1} -
				{\mathbf{H}}^\top\mathbf{R}_{\mathrm{obs}}^{-1}
				{\mathbf{H}}\right)}_{\mathbf{A}_k}\mathbf{V}_k +
			\underbrace{\mathbf{P}_{f,k}^{-1}\bar{\mathbf{V}}_{f,k} +
				{\mathbf{H}}^\top\mathbf{R}_{\mathrm{obs}}^{-1}
				\mathbf{Z}_k}_{\mathbf{b}_k}.
		\end{align*}
		The coefficients $\mathbf{A}_k$ and $\mathbf{b}_k$ depend on
		the forecast statistics $(\bar{\mathbf{V}}_{f,k},
		\mathbf{P}_{f,k})$ and the observation $\mathbf{Z}_k$ but
		not on $\mathbf{V}_k$, confirming affinity. Negative
		definiteness of $\mathbf{A}_k$ follows since $\mathbf{P}_{f,k}^{-1}
		\succ 0$ and ${\mathbf{H}}^\top\mathbf{R}_{\mathrm{obs}}^{-1}
		{\mathbf{H}} \succeq 0$.
	\end{proof}

	\begin{corollary}[Recovery of the True Joint Posterior]
		\label{cor:posterior_recovery}
		Under the assumptions of
		Theorems~\ref{thm:dsm_consistency}--\ref{thm:langevin_surrogate},
		suppose the step sizes $\{\alpha_\ell\}$ satisfy
		$\alpha_\ell \equiv \alpha$ with $0 < \alpha < 2m/L^2$,
		and the noise schedule satisfies $\sigma_\ell \searrow 0$.
		Suppose further that the CGNS exactly represents the true
		system dynamics, so that $p_{\mathrm{CGNS}} = p_{\mathrm{aug}}$
		and $\varepsilon_{\mathrm{CGNS}} = 0$
		\eqref{eq:cgns_misspec}. Then:
		
		\begin{enumerate}[(i)]
			
			\item \textbf{Score convergence.}
			As $N \to \infty$ and $\sigma \to 0$ sequentially,
			Theorem~\ref{thm:n2n_consistency}(v) gives
			\begin{equation}
				\mathbf{s}_{\mathrm{EnKF},k} + \mathbf{s}_{\theta^*}
				\to \nabla_{\mathbf{V}_k}\log
				p_{\mathrm{aug}}(\mathbf{V}_k \mid \mathbf{Z}_{0:k})
				\quad \text{in } L^2(p_{\mathrm{aug}}).
				\label{eq:cor_total_score}
			\end{equation}
			
			\item \textbf{Langevin convergence.}
			For $N$ sufficiently large and $\alpha$ sufficiently small,
			the annealed Langevin chain \eqref{eq:langevin} satisfies,
			by Theorem~\ref{thm:langevin_surrogate}(iii):
			\begin{equation}
				W_2\!\left(
				\mathrm{Law}\!\left(\mathbf{V}_k^{(i,\ell)}\right),\,
				p_{\mathrm{aug}}(\mathbf{V}_k \mid \mathbf{Z}_{0:k})
				\right)
				\to 0
				\quad\text{as }
				\ell \to \infty,\;
				N \to \infty,\;
				\alpha \to 0,\;
				\sigma \to 0,
				\label{eq:cor_langevin_conv}
			\end{equation}
			where $W_2$ denotes the Wasserstein-$2$ distance.
			
			\item \textbf{Joint posterior recovery.}
			As $N, \ell \to \infty$ and $\alpha, \sigma \to 0$,
			the particles $\{(X_{a,k}^{(i)}, Y_{a,k}^{(i)})\}$
			become asymptotically exchangeable and their empirical
			measure converges weakly to $p^*$. Specifically, for any
			bounded continuous test function $\phi$,
			\begin{equation}
				\frac{1}{N}\sum_{i=1}^N
				\phi\!\left(X_{a,k}^{(i)},\,
				Y_{a,k}^{(i)}\right)
				\xrightarrow{N,\,\ell\to\infty}
				\mathbb{E}_{p^*}\!\left[\phi(X_k, Y_k)\right],
				\label{eq:cor_joint_conv}
			\end{equation}
			where $p^*(X_k, Y_k \mid \mathbf{Z}_{0:k})$ is the true
			joint posterior. The marginal posterior of $\mathbf{Y}_k$ is recovered as the Gaussian mixture
			\begin{equation}
				\frac{1}{N}\sum_{i=1}^N
				\mathcal{N}\!\left(
				\boldsymbol{\mu}_{a,k}^{(i)},\,
				\mathbf{R}_{f,k}^{(i)}\right)
				\xrightarrow{w}
				p^*(Y_k \mid \mathbf{Z}_{0:k})
				\quad\text{as } N,\,\ell \to \infty,
				\label{eq:cor_mixture_conv}
			\end{equation}
			where $\mathbf{R}_{f,k}^{(i)} \to \mathbf{R}_k(X_k)$
			as $N \to \infty$ by CGNS consistency~\cite{majda2018strategies}.
			
			\item \textbf{Approximate recovery under surrogate
				misspecification.}
			When $\varepsilon_{\mathrm{CGNS}} > 0$, Theorem~\ref{thm:langevin_surrogate}(i)--(ii) gives
			$\mathrm{Law}(\mathbf{V}_k^{(i,\ell)})
			\xrightarrow{w} q_\sigma$ as $\ell \to \infty$,
			and the $W_2$ distance to the true posterior satisfies
			\begin{equation}
				W_2\!\left(
				\mathrm{Law}\!\left(\mathbf{V}_k^{(i,\ell)}\right),\,
				p_{\mathrm{aug}}\right)
				\;\leq\;
				W_2\!\left(
				\mathrm{Law}\!\left(\mathbf{V}_k^{(i,\ell)}\right),\,
				q_\sigma\right)
				+
				W_2\!\left(q_\sigma,\, p_{\mathrm{aug}}\right),
				\label{eq:cor_approx_w2}
			\end{equation}
			where the first term satisfies
			\begin{equation}
				W_2\!\left(
				\mathrm{Law}\!\left(\mathbf{V}_k^{(i,\ell)}\right),\,
				q_\sigma\right)
				\;\leq\;
				\rho^\ell\,
				W_2\!\left(\mu^0,\, q_\sigma\right)
				+ \frac{C_2\,\alpha}{1 - \rho}
				\;\to\; 0
				\label{eq:cor_first_term}
			\end{equation}
			as $\ell \to \infty$ and $\alpha \to 0$ by
			Theorem~\ref{thm:langevin_surrogate}(ii), and the second term is the irreducible surrogate quality gap:
			\begin{equation}
				W_2\!\left(q_\sigma,\, p_{\mathrm{aug}}\right)
				\;\leq\;
				\frac{1}{m}\,\varepsilon_{\mathrm{CGNS}},
				\label{eq:cor_second_term}
			\end{equation}
			controlled entirely by the CGNS misspecification and
			independent of $N$, $\ell$, and $\alpha$. Consequently
			\begin{equation}
				\lim_{\ell\to\infty,\,\alpha\to 0}\;
				W_2\!\left(
				\mathrm{Law}\!\left(\mathbf{V}_k^{(i,\ell)}\right),\,
				p_{\mathrm{aug}}\right)
				\;\leq\;
				\frac{\varepsilon_{\mathrm{CGNS}}}{m},
				\label{eq:cor_residual_bound}
			\end{equation}
			so that the residual approximation error is determined
			entirely by the structural surrogate gap, not by $N$,
			$\ell$, or $\alpha$.
			
		\end{enumerate}
	\end{corollary}
	
	\begin{proof}
		
		\textbf{Part (i).}
		Immediate from Theorem~\ref{thm:n2n_consistency}(v) under
		$\varepsilon_{\mathrm{CGNS}} = 0$.
		
		\textbf{Part (ii).}
		By \eqref{eq:cor_total_score}, the total learned score
		converges to $\nabla\log p_{\mathrm{aug}}$ in
		$L^2(p_{\mathrm{aug}})$ as $N \to \infty$. The Langevin
		chain \eqref{eq:langevin} therefore targets $p_{\mathrm{aug}}$ in the limit. Convergence in $W_2$ then follows from Theorem~\ref{thm:langevin_surrogate}(iii), which gives $W_2(\mathrm{Law}(\mathbf{V}_k^{(i,\ell)}), p_{\mathrm{aug}}) \to 0$ as $\ell \to \infty$, $N \to \infty$, $\alpha \to 0$, and $\sigma \to 0$ sequentially.
		
		\textbf{Part (iii).}
		By \eqref{eq:cor_langevin_conv}, the single-particle law
		$\mathrm{Law}(\mathbf{V}_k^{(i,\ell)}) \xrightarrow{w}
		p_{\mathrm{aug}}$ as $\ell, N \to \infty$. Since the
		particles $\{(X_{a,k}^{(i)}, Y_{a,k}^{(i)})\}_{i=1}^N$
		are constructed from an exchangeable ensemble with shared
		forecast statistics $(\bar{\mathbf{V}}_{f,k},
		\mathbf{P}_{f,k})$, they form an exchangeable sequence.
		Under EnKF consistency~\cite{mandel2011convergence}, the
		ensemble is asymptotically i.i.d.\ as $N \to \infty$ in the
		sense of propagation of chaos~\cite{sznitman1991topics,
			moral2004feynman}: for any fixed $k$, the $k$-particle
		marginal of the joint ensemble law converges to the
		$k$-fold product $p_{\mathrm{aug}}^{\otimes k}$.
		The convergence \eqref{eq:cor_joint_conv} then follows
		from the weak law of large numbers applied to the
		asymptotically i.i.d.\ particles.
		
		For \eqref{eq:cor_mixture_conv}, marginalize \eqref{eq:cor_joint_conv} over $X_k$: since the sampling
		map $\boldsymbol{\mu}_{a,k}^{(i)} \mapsto Y_{a,k}^{(i)}
		\sim \mathcal{N}(\boldsymbol{\mu}_{a,k}^{(i)}, \mathbf{R}_{f,k}^{(i)})$ is continuous in $\boldsymbol{\mu}_{a,k}^{(i)}$, and $\mathbf{R}_{f,k}^{(i)} \to \mathbf{R}_k(X_k)$ as $N \to \infty$ by CGNS consistency~\cite{majda2018strategies}, the continuous mapping theorem gives weak convergence of the empirical mixture to $p^*(Y_k \mid \mathbf{Z}_{0:k})$.
		
		\textbf{Part (iv).}
		The triangle inequality for $W_2$ gives
		\eqref{eq:cor_approx_w2} directly. For the first term \eqref{eq:cor_first_term}, apply Theorem~\ref{thm:langevin_surrogate}(ii) with the contraction factor $\rho = 1 - m\alpha + L^2\alpha^2/2
		\in (0,1)$: the transient term $\rho^\ell W_2(\mu^0, q_\sigma) \to 0$ exponentially in $\ell$, and the discretization bias $C_2\alpha/(1-\rho) \leq 2C_2/m \to 0$ as $\alpha \to 0$ by
		\eqref{eq:bias_bounded}.
		
		For the second term \eqref{eq:cor_second_term}, note
		that $q_\sigma$ is the invariant measure of the ULA
		targeting the surrogate score $\nabla\log q_\sigma = \mathbf{s}_{\mathrm{EnKF},k} + \mathbf{s}_{\theta^*}$, while $p_{\mathrm{aug}}$ has score $\nabla\log p_{\mathrm{aug}}$. By the $W_2$ perturbation bound for ULA (Durmus \& Moulines 2019, Proposition~4 \cite{durmus2019analysis}) and the definition of $\varepsilon_{\mathrm{CGNS}}$ \eqref{eq:cgns_misspec}:
		\begin{equation}
			W_2\!\left(q_\sigma,\, p_{\mathrm{aug}}\right)
			\;\leq\;
			\frac{1}{m}\,
			\bigl\|\nabla\log q_\sigma
			- \nabla\log p_{\mathrm{aug}}
			\bigr\|_{L^2(p_{\mathrm{aug}})}
			\;\leq\;
			\frac{\varepsilon_{\mathrm{CGNS}}}{m},
			\label{eq:pf_w2_surrogate}
		\end{equation}
		which is \eqref{eq:cor_second_term}. Combining \eqref{eq:cor_first_term} and \eqref{eq:cor_second_term} in \eqref{eq:cor_approx_w2} and taking $\ell \to \infty$, $\alpha \to 0$ gives \eqref{eq:cor_residual_bound}.
	\end{proof}

	\section{Proof of Theorem~\ref{thm:nongaussian_recovery}}
	\label{app:nong_proof}
	\begin{proof}
		
		\textbf{Parts (i) and (ii).}
		Since $\mathcal{G}$ is a closed subspace of $L^2(p_{\mathrm{aug}})$, the orthogonal projection theorem gives the infimum in \eqref{eq:affine_gap} as
		$\|(\mathbf{I}-\Pi_{\mathcal{G}})\mathbf{s}_{\mathrm{post},k}
		\|_{L^2}$, which is strictly positive by hypothesis
		$\mathbf{s}_{\mathrm{post},k} \notin \mathcal{G}$. Since
		$\mathbf{s}_{\mathrm{EnKF},k} \in \mathcal{G}$ by
		Lemma~\ref{lem:enkf_affine}, for any $g \in \mathcal{G}$:
		\begin{equation}
			\bigl\|\mathbf{s}_{\mathrm{post},k}
			- \mathbf{s}_{\mathrm{EnKF},k} - g
			\bigr\|_{L^2}
			\;\geq\; \bigl\|(\mathbf{I}-\Pi_{\mathcal{G}})
			\mathbf{s}_{\mathrm{post},k}\bigr\|_{L^2} > 0,
			\label{eq:pf_affine_lb}
		\end{equation}
		since $\mathbf{s}_{\mathrm{EnKF},k} + g \in \mathcal{G}$
		and the infimum over $\mathcal{G}$ is the projection distance. This establishes the necessity of the nonlinear residual.
		
		For Part~(ii), fix $\varepsilon > 0$. Choose $R$ large enough so that
		$\int_{\|\mathbf{v}\|>R}
		\|\mathbf{s}_{\mathrm{post},k} - \mathbf{s}_{\mathrm{EnKF},k}
		\|^2\,dp_{\mathrm{aug}} < \varepsilon^2/2$,
		which is possible since $p_{\mathrm{aug}}$ has finite second moment. On the compact ball $B_R$, apply the universal approximation theorem~\cite{leshno1993multilayer} to find
		$\mathbf{s}_{\theta^*} \in \mathcal{F}_\theta$ with
		$\|\mathbf{s}_{\mathrm{post},k} - \mathbf{s}_{\mathrm{EnKF},k}
		- \mathbf{s}_{\theta^*}\|_{L^\infty(B_R)} < \delta$,
		where $\delta$ is chosen so that $\delta^2 p_{\mathrm{aug}}(B_R) < \varepsilon^2/2$. Combining the two contributions gives \eqref{eq:residual_approx}.
		
		\textbf{Part (iii).}
		Suppose for contradiction that $\mathbf{s}_{\theta^*}^{\boldsymbol{\mu}} \equiv \mathbf{0}$. Since $\mathbf{V}_k = (X_k, \boldsymbol{\mu}_k)$ are coordinate blocks and the $L^2(p_{\mathrm{aug}})$ inner product decomposes over coordinate subspaces:
		\begin{equation}
			\bigl\|\mathbf{s}_{\mathrm{post},k}
			- \mathbf{s}_{\mathrm{EnKF},k}
			- \mathbf{s}_{\theta^*}\bigr\|_{L^2}^2
			=
			\bigl\|\mathbf{s}_{\mathrm{post},k}^X
			- \mathbf{s}_{\mathrm{EnKF},k}^X
			- \mathbf{s}_{\theta^*}^X\bigr\|_{L^2}^2
			+
			\bigl\|\mathbf{s}_{\mathrm{post},k}^{\boldsymbol{\mu}}
			- \mathbf{s}_{\mathrm{EnKF},k}^{\boldsymbol{\mu}}
			\bigr\|_{L^2}^2,
			\label{eq:pf_block_decomp}
		\end{equation}
		where the second term equals $\delta_\mu^2$ when
		$\mathbf{s}_{\theta^*}^{\boldsymbol{\mu}} \equiv \mathbf{0}$. Hence $\|\mathbf{s}_{\mathrm{post},k} -
		\mathbf{s}_{\mathrm{EnKF},k} - \mathbf{s}_{\theta^*}\|_{L^2} \geq \delta_\mu$, contradicting \eqref{eq:residual_approx} for $\varepsilon < \delta_\mu$. The bound \eqref{eq:latent_correction_bound} follows from the reverse triangle inequality applied to the $\boldsymbol{\mu}$-block:
		\begin{equation}
			\delta_\mu
			\;\leq\;
			\bigl\|\mathbf{s}_{\mathrm{post},k}^{\boldsymbol{\mu}}
			- \mathbf{s}_{\mathrm{EnKF},k}^{\boldsymbol{\mu}}
			- \mathbf{s}_{\theta^*}^{\boldsymbol{\mu}}
			\bigr\|_{L^2}
			+ \bigl\|\mathbf{s}_{\theta^*}^{\boldsymbol{\mu}}
			\bigr\|_{L^2}
			\;\leq\; \varepsilon
			+ \bigl\|\mathbf{s}_{\theta^*}^{\boldsymbol{\mu}}
			\bigr\|_{L^2}.
			\label{eq:pf_reverse_triangle}
		\end{equation}
		
		\textbf{Part (iv).}
		
        The Langevin-corrected means satisfy
		\begin{equation}
			\boldsymbol{\mu}_{a,k}^{(i)}
			= \boldsymbol{\mu}_{\mathrm{CGNS},k}^{(i)}
			+ \Delta\boldsymbol{\mu}_k^{(i)},
			\label{eq:mu_decomp}
		\end{equation}
		where $\Delta\boldsymbol{\mu}_k^{(i)}$ is the net displacement accumulated by the Langevin chain \eqref{eq:langevin} over $L$ steps. When $\delta_\mu > 0$, Part~(iii) guarantees $\|\mathbf{s}_{\theta^*}^{\boldsymbol{\mu}}
		\|_{L^2(p_{\mathrm{aug}})} \geq \delta_\mu - \varepsilon
		> 0$, so $\mathbf{s}_{\theta^*}^{\boldsymbol{\mu}}$ is
		nonzero on a set of positive $p_{\mathrm{aug}}$-measure. Since the Langevin chain converges to $q_\sigma$ by Theorem~\ref{thm:langevin_surrogate}(i)--(ii), and $q_\sigma$ is mutually absolutely continuous with respect to $p_{\mathrm{aug}}$ for $\sigma > 0$ (Remark~\ref{rem:abs_cont}), the score $\mathbf{s}_{\theta^*}^{\boldsymbol{\mu}}$ is also nonzero on a set of positive $q_\sigma$-measure. Therefore, \eqref{eq:mu_decomp} follows from the Langevin recursion initialized at $\mathbf{V}_k^{(i,0)} = \mathbf{V}_{f,k}^{(i)}$.
		
		{Establishing $p_{\mathrm{\Omega}} \neq
			p_{\mathrm{CGNS}}$ for large $L$.}
		By Part~(iii), $\mathbf{s}_{\theta^*}^{\boldsymbol{\mu}}$
		is nonzero in $L^2(p_{\mathrm{aug}})$, hence nonzero on a
		set $A$ with $p_{\mathrm{aug}}(A) > 0$. By mutual absolute
		continuity of $q_\sigma$ and $p_{\mathrm{aug}}$ (Remark~\ref{rem:abs_cont}), $q_\sigma(A) > 0$ as well.
		Since the Langevin chain converges weakly to $q_\sigma$
		by Theorem~\ref{thm:langevin_surrogate}(i)--(ii),
		the time-averaged occupation measure of the chain converges to $q_\sigma$~\cite{durmus2019analysis}, so the accumulated displacement
		\begin{equation}
			\Delta\boldsymbol{\mu}_k^{(i)}
			= \sum_{\ell=0}^{L-1}\frac{\alpha_\ell}{2}
			\mathbf{s}_{\theta^*}^{\boldsymbol{\mu}}\!
			\left(\mathbf{V}_k^{(i,\ell)}\right)
			+ \text{(EnKF and noise terms)}
			\label{eq:pf_drift_sum}
		\end{equation}
		satisfies $\|\Delta\boldsymbol{\mu}_k^{(i)}\|_{L^2(q_\sigma)} > 0$ as $L \to \infty$, since the drift integrand is nonzero on a $q_\sigma$-positive set. Hence $\boldsymbol{\mu}_{a,k}^{(i)} \neq \boldsymbol{\mu}_{\mathrm{CGNS},k}^{(i)}$ in distribution for some $i$, and since Gaussian distributions with the same covariance differ if and only if their means differ, $p_{\mathrm{\Omega}} \neq p_{\mathrm{CGNS}}$ for  $L$ sufficiently large.
		
		{Asymptotic KL improvement.}
		Under exact CGNS and as $L, N \to \infty$,
		$\alpha, \sigma \to 0$, Corollary~\ref{cor:posterior_recovery}(iii) gives $p_{\mathrm{\Omega}} \xrightarrow{w} p^*$ and hence $D_{\mathrm{KL}}(p^* \| p_{\mathrm{\Omega}}) \to 0$. When $p^* \neq p_{\mathrm{CGNS}}$, the right-hand side of \eqref{eq:kl_improvement} is strictly positive since KL divergence vanishes if and only if the two measures coincide. This gives \eqref{eq:kl_improvement}.
	\end{proof}

	\begin{remark}[When does $\delta_\mu > 0$ hold?]
		\label{rem:latent_nongaussian}
		The condition $\delta_\mu > 0$ holds whenever the $\boldsymbol{\mu}$-component of the true posterior score $\mathbf{s}_{\mathrm{post},k}^{\boldsymbol{\mu}}
		= \nabla_{\boldsymbol{\mu}_k}\log p_{\mathrm{aug}}(\mathbf{V}_k \mid \mathbf{Z}_{0:k})$ is non-affine in $\boldsymbol{\mu}_k$. Since $\mathbf{s}_{\mathrm{EnKF},k}^{\boldsymbol{\mu}}$ is
		affine by Lemma~\ref{lem:enkf_affine}, their $L^2$ difference is nonzero whenever this non-affinity holds. This is implied by, but strictly weaker than, the true posterior being non-Gaussian: the degenerate case $\delta_\mu = 0$ occurs when the $\boldsymbol{\mu}$-score component happens to be affine, which holds in particular when the true posterior is Gaussian, but may also hold for certain non-Gaussian posteriors whose non-Gaussianity is concentrated entirely in the $X$-marginal.
	\end{remark}
	
	\section{Numerical details of experiments}
	\label{app:experiments}

	\subsection{Topographic Model: Spectral Equations and CGNS Structure}
	\label{app:topo_equations}

	To discretize the system~\eqref{eq:vorticity_pde_app}-\eqref{eq:T_pde_app} in space, we project the system using the Fourier spectral modes.
	
	\subsubsection*{Spectral projection}
	Expanding all fields in Fourier modes $e^{ikx}$ and using $\psi_k = -\omega_k/\lambda_m^2$, we reformulate the system\eqref{eq:vorticity_pde_app}-\eqref{eq:T_pde_app} for the layered topographic model as
	\begin{align}
		d\omega_k &=\left[ -i\ell\!\left(Uk - \frac{\beta}{k}\right)\!\omega_k
		- i\ell k h_k U- \gamma_k\omega_k + f_k\right] dt + \sigma_k\,dW_k,
		\label{eq:spectral_vort_app} \\[4pt]
		dU &= \left[2\ell\,\mathrm{Im}\!\left(\sum_k \frac{\bar{h}_k\,\omega_k}{k} \right) - \gamma_U U + f_U \right]dt+ \sigma_U\,dW_U,  \label{eq:spectral_mean_app} \\[4pt]
		dT_k &= \left[\left(-\gamma_T - ikU\right) T_k
		+ \frac{i\alpha k}{\lambda_m^2}\,\omega_k\right] dt,
		\label{eq:spectral_T_app}
	\end{align}
	where $\ell$ is a geometric scaling factor. $dW_k$ and $dW_U$ are Wiener processes with diffusion coefficients $\sigma_k$ and $\sigma_U$, respectively.
	
	\subsubsection*{CGNS structure}
	\label{app:topo_cgns}
	With $\bo X = U$ and
	$\bo Y = (\omega_1,\ldots,\omega_K, T_1,\ldots,T_K)$, equations
	\eqref{eq:spectral_mean_app}--\eqref{eq:spectral_T_app} constitute a CGNS.  The drift of $U$ in \eqref{eq:spectral_mean_app} is linear in $\omega_k$ (with coefficient $2\ell\,\mathrm{Im}(\bar{h}_k/k)$), satisfying \eqref{eq:cgns_obs}.  The drifts of $\omega_k$ and $T_k$ in \eqref{eq:spectral_vort_app}--\eqref{eq:spectral_T_app} are linear in their respective components (with coefficients depending on $U$), satisfying \eqref{eq:cgns_lat}.

	\subsubsection*{CGNS filter and EnKF baseline}
Figure~\ref{fig:top_enkf_baseline} shows the EnKF posterior with discrete observations of $U$ every 200 time steps, using the CGNS as the forecast model.  These serve as the Gaussian baseline \eqref{eq:enkf_analysis_ensemble} against which $\Omega$'s residual correction is evaluated in
	Section~\ref{sec:exp_topo}.
	
	\begin{figure}[h]
		\centering
		\includegraphics[width=0.8\linewidth]{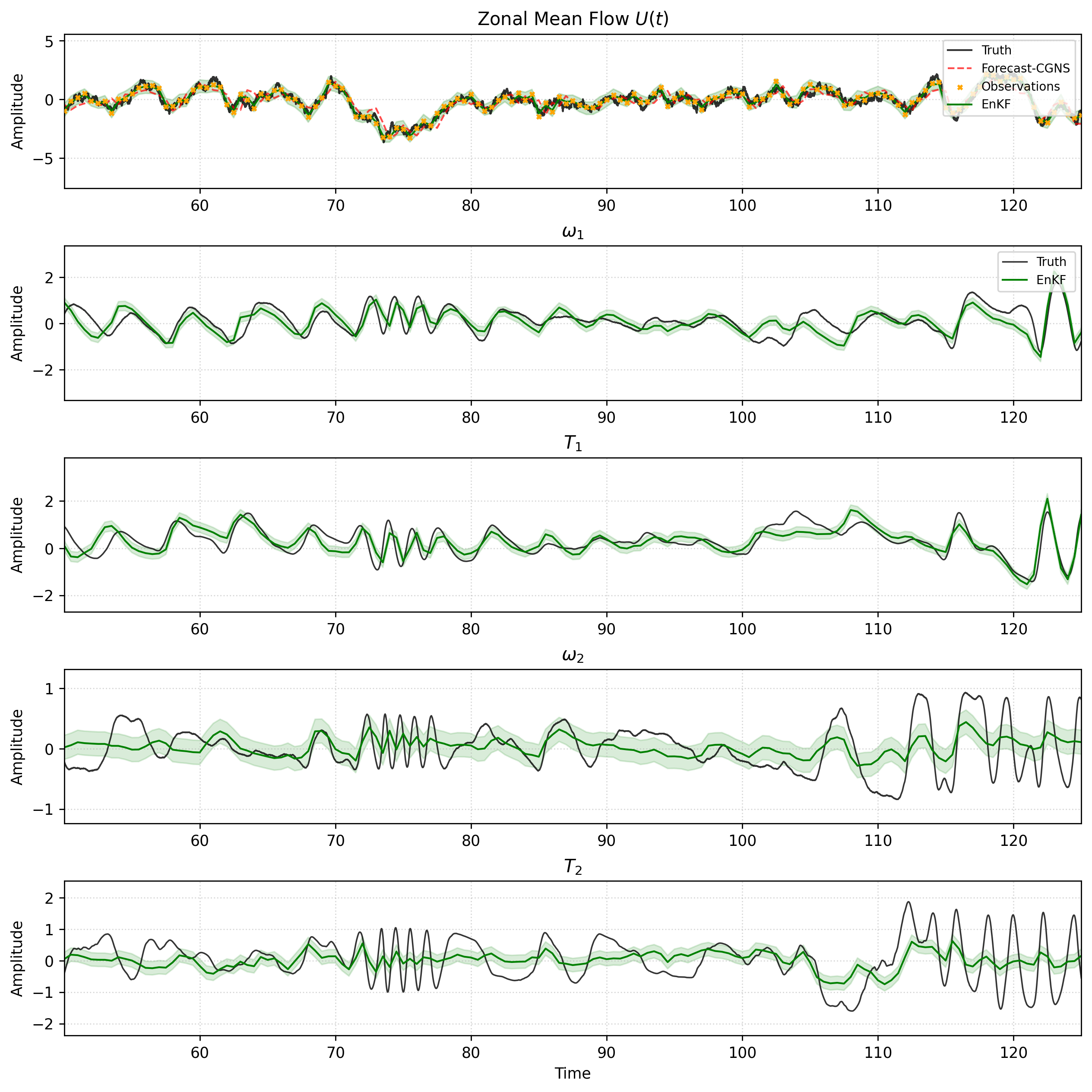}
		\caption{EnKF posterior for the topographic model with discrete observations of $U$ every 200 time steps, using the CGNS as the forecast model.  This is the Gaussian baseline \eqref{eq:enkf_analysis_ensemble} that $\Omega$ corrects in Section~\ref{sec:exp_topo}.}
		\label{fig:top_enkf_baseline}
	\end{figure}

	\section{Dimensional Robustness of Conditional Gaussian Inference}
	\label{app:latent_dimension}
	
	A central advantage of the CGNS framework is that the Monte Carlo error in recovering the latent posterior does not grow with the dimension of $\mathbf{Y}$. As discussed in \ref{rem:cod}, the factorization $p(\mathbf{X}, \mathbf{Y}) = p(\mathbf{Y} \mid \mathbf{X})\,p(\mathbf{X})$ and the closed form of $p(\mathbf{Y}_k \mid \mathbf{X}_{0:k})$ through the CGNS filtering equations
	\eqref{eq:cgns_mean}--\eqref{eq:cgns_cov} imply that  the integration over $\mathbf{Y}$ is performed analytically inside each ensemble component, and the sampling error is governed entirely by the observed process $\mathbf{X}$. This appendix provides numerical evidence for this dimensional robustness using the barotropic topographic turbulence model at increasing Fourier truncation levels.

	We consider the topographic model~\eqref{eq:spectral_vort_app} with Fourier truncation level $k_{\max}$, giving a latent state consisting of the Fourier coefficients
	\begin{equation}
		\mathbf{Y} = (v_k, \bar{v}_k, T_k, \bar{T}_k),\, k=1,\cdots, k_{\max}
		\qquad d_Y = 4\,k_{\max},
		\label{eq:app_latent_dim}
	\end{equation}
	where $d_Y $ is the latent dimension. For each value of $k_{\max}$, truth trajectories are
	generated and the CGNS filter is applied using observations  of the large-scale variable $X = U$.
	Two complementary experiments are reported in~Figure~\ref{fig:CGNS_dimension_MC}. The first varies the number of independent realizations $M$ of the process for fixed latent dimensions $d_Y \in \{8, 16, 32, 48, 64\}$, measuring the convergence rate  of the error beteewn the sampled $Y$using the CGNS and the truth with $M$. The second varies $d_Y$ from $4$ to $64$, measuring how the reconstruction quality changes as the latent space grows. The quality of the recovered latent statistics is measured by the normalized error between the true latent state and the CGNS conditional posterior mean  $\boldsymbol{\mu}_{f,k}$ as $\frac{\|\mathbf{Y}_{\mathrm{truth}} - \boldsymbol{\mu}_{f}\|_2} {\|\mathbf{Y}_{\mathrm{truth}}\|_2}$.
	
	In Figure~\ref{fig:CGNS_dimension_MC} (a), all curves closely follow the reference Monte Carlo rate $M^{-1/2}$ and remain nearly close across latent dimensions ranging from $d_Y = 8$ to $d_Y = 64$. This indicates that increasing the latent dimension does not increase the number of ensemble members required to achieve a given accuracy.

	Figure~\ref{fig:CGNS_dimension_MC}(b) shows  the error between the conditional posterior mean  and latent variable does not exhibit systematic growth as $d_Y$ increases from $4$ to $64$, more than an order of magnitude increase in the latent space size. The absence of dimensional degradation confirms that the conditional Gaussian closure \eqref{eq:cgns_mean}--\eqref{eq:cgns_cov} efficiently absorbs the latent complexity, and later leaving the $\Omega$ ensemble approximation quality unaffected by $d_Y$.

	\begin{figure}[t]
		\centering
		\begin{subfigure}[t]{0.48\textwidth}
			\centering
			\includegraphics[width=\linewidth]{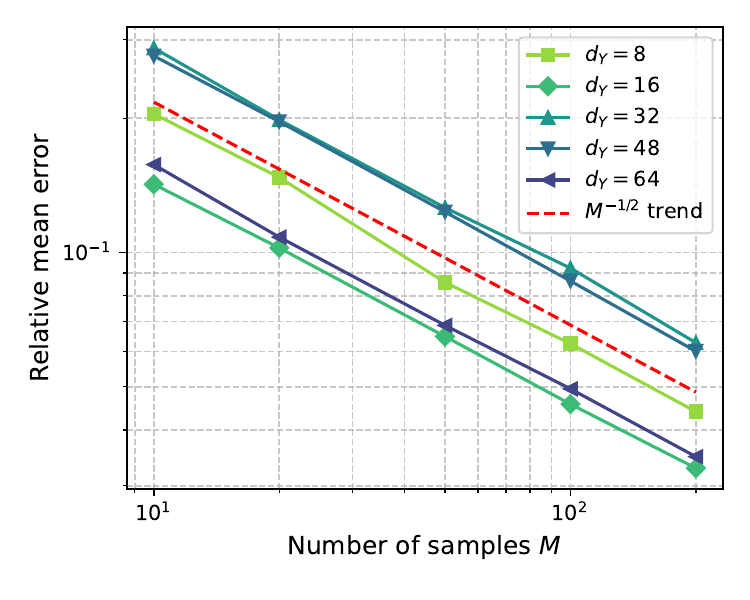}
			\caption{Relative error between the sampled and true latent variable as a function of the number of 	independent realizations $M$ for latent dimensions $d_Y \in \{8, 16, 32, 48, 64\}$. All curves closely follow the reference rate $M^{-1/2}$ and nearly overlap, confirming that convergence is insensitive to the latent dimension.}
		\end{subfigure}
		\hfill
		\begin{subfigure}[t]{0.48\textwidth}
			\centering
			\includegraphics[width=\linewidth]{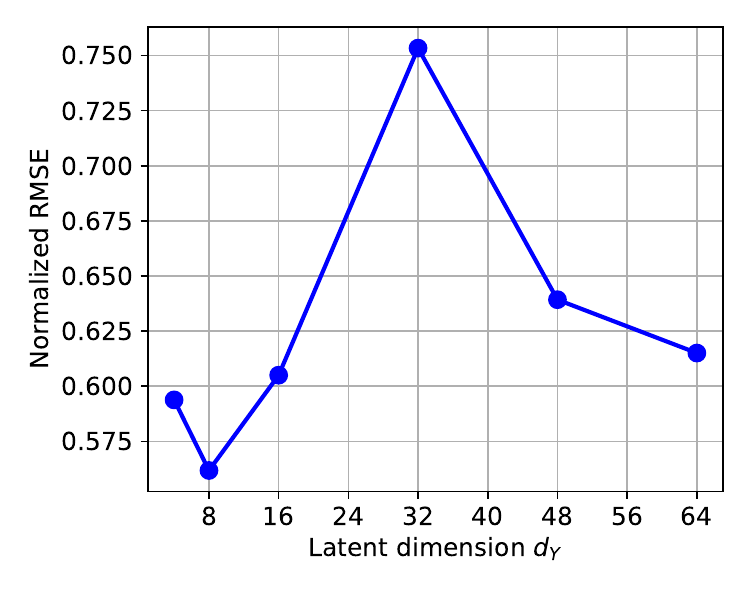}
			\caption{Relative error  as a function of latent dimension $d_Y = 4k_{\max}$. The reconstruction error does not grow systematically as $d_Y$ 	increases from $4$ to $64$, demonstrating  robustness of the CGNS conditional Gaussian closure in high-dimensional latent spaces.}
		\end{subfigure}
		\caption{Numerical investigation of dimensional
			robustness in CGNS-based latent inference.
			\textbf{(a)} The $\mathcal{O}(M^{-1/2})$ Monte Carlo
			convergence rate is maintained across all tested latent dimensions. \textbf{(b)} Reconstruction  accuracy remains stable as the latent dimension grows. Together, these results confirm that the
			dominant sampling burden is associated with the observed process $\mathbf{X}$, while the latent variables $\mathbf{Y}$ are efficiently recovered through the conditional Gaussian closure.}
		\label{fig:CGNS_dimension_MC}
	\end{figure}

	\subsection{Lorenz-96: CGNS Structure Failure and Surrogate Construction}
	\label{app:l96_surrogate}
	
	\label{app:l96_cgns_failure}
	Partitioning the L96 state into odd variables $\mathbf{x} = (u_1, u_3, \ldots)$ and even variables $\mathbf{y} = (u_2, u_4, \ldots)$, the equation for a representative pair centered on $j = 2k$ is
	\begin{align}
		du_{2k-1}
		&= \bigl[(u_{2k+1} - u_{2k-3})\,u_{2k-2}
		- u_{2k-1} + F\bigr]\,dt
		+ \sigma\,dW_{2k-1},
		\label{eq:l96_odd_app} \\[4pt]
		du_{2k}
		&= \bigl[(u_{2k+2} - u_{2k-2})\,u_{2k-1}
		- u_{2k} + F\bigr]\,dt
		+ \sigma\,dW_{2k}.
		\label{eq:l96_even_app}
	\end{align}
	The hidden-state equation \eqref{eq:l96_even_app} contains the term $u_{2k+2}\,u_{2k-1}$, a product of a hidden variable ($u_{2k+2}$, even-indexed) with an observed variable ($u_{2k-1}$, odd indexed). This is bilinear, however, it also contains $u_{2k-2}\,u_{2k-1}$, which is a product of a hidden variable ($u_{2k-2}$) with an observed variable.
	Moreover,  the same hidden variable $u_{2k+2}$ appears in both the observed equation \eqref{eq:l96_odd_app} (through $u_{2k-2} = u_{2(k-1)}$, i.e., a hidden neighbor) and the hidden equation, making the effective drift of $\mathbf{y}$ nonlinear in $\mathbf{y}$: the term $(u_{2k+2} - u_{2k-2})\,u_{2k-1}$ is a product of two hidden-state components $u_{2k+2}$ and
	$u_{2k-1}$ viewed from the perspective of the hidden-state equation (since $u_{2k-2}$ is even-indexed and hence hidden).  This violates
	the linear-in-$\mathbf{y}$ requirement \eqref{eq:cgns_lat}.
	
	\subsubsection*{Surrogate CGNS}
	The surrogate \eqref{eq:surr_obs}--\eqref{eq:surr_hid} restores
	conditional Gaussianity by two modifications: (i) products of hidden variables replace the quadratic hidden cross-products in the hidden equation with the {observed} neighbor,
	which are linear in $\mathbf{y}$; and (ii) an inflation noise $\sigma_{\mathrm{s}} = 5$ is added to the observed-state diffusion to compensate for the discarded nonlinear terms. The verification of the two CGNS structural conditions and the role of $\sigma_{\mathrm{s}}$ are described in Section~\ref{sec:exp_l96}.
	
	Figure~\ref{fig:l96_surrogate} shows the surrogate CGNS posterior for representative hidden variables under continuous observations, confirming that the surrogate captures the dominant dynamics of the
	hidden state despite the approximation.
	
	\begin{figure}[h]
		\centering
		\includegraphics[width=0.85\linewidth]{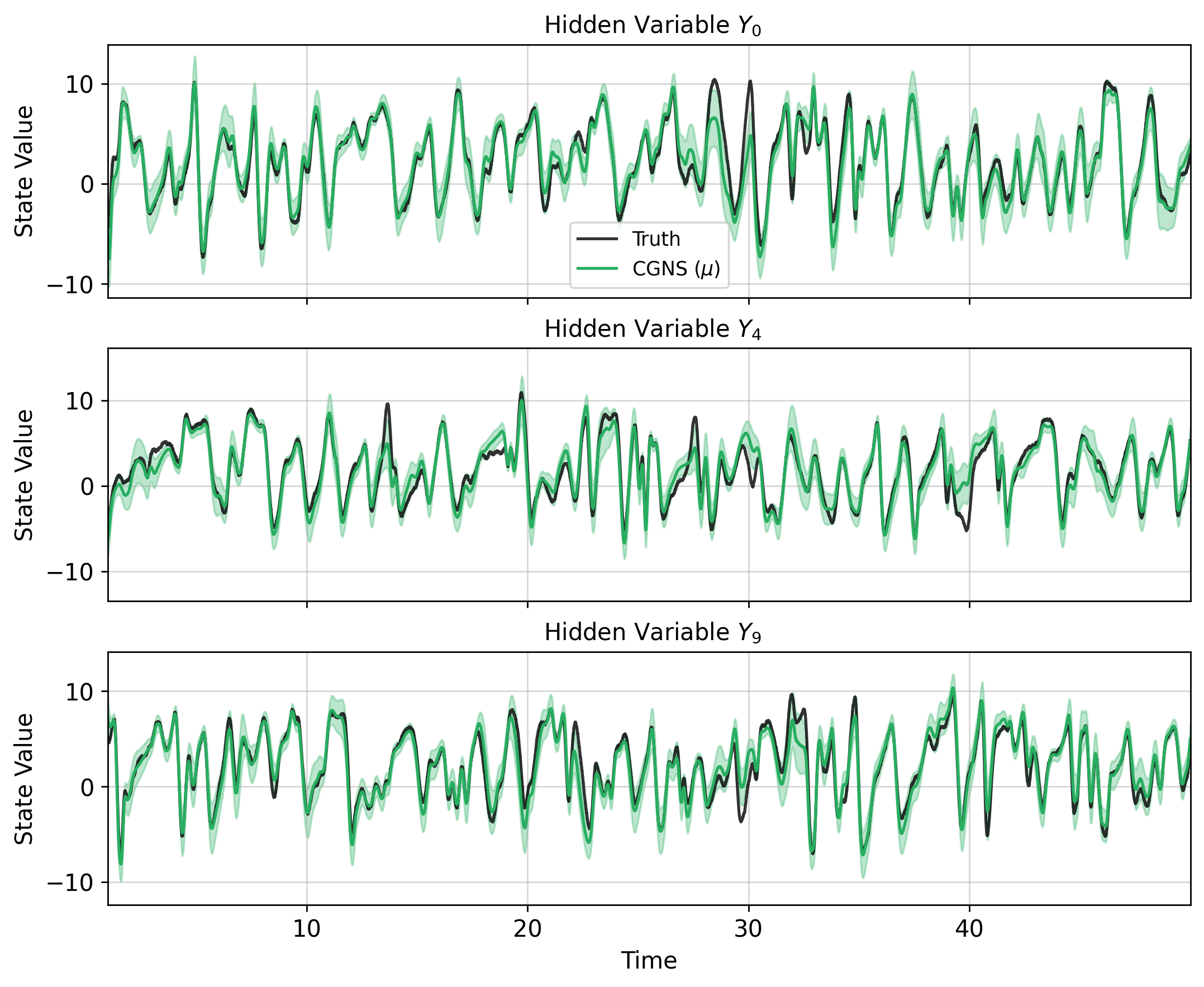}
		\caption{Surrogate CGNS posterior for representative hidden variables $Y_k = u_{2k}$ of the stochastic Lorenz-96 system under continuous observations.  The posterior mean  tracks the reference with uncertainty bands that correctly reflect the inflation noise $\sigma_{\mathrm{s}}$.}
		\label{fig:l96_surrogate}
	\end{figure}
	
	Figure~\ref{fig:score_surrogate} shows how the learned residual score complements the surrogate model. We compare the residual score based on normalized input values for two cases: the accurate surrogate CGNS used in this experiment and a version with smaller inflation noise, $\sigma_s = 4$.
	
	\begin{figure}[htbp]
	\centering
	\includegraphics[width=0.85\linewidth]{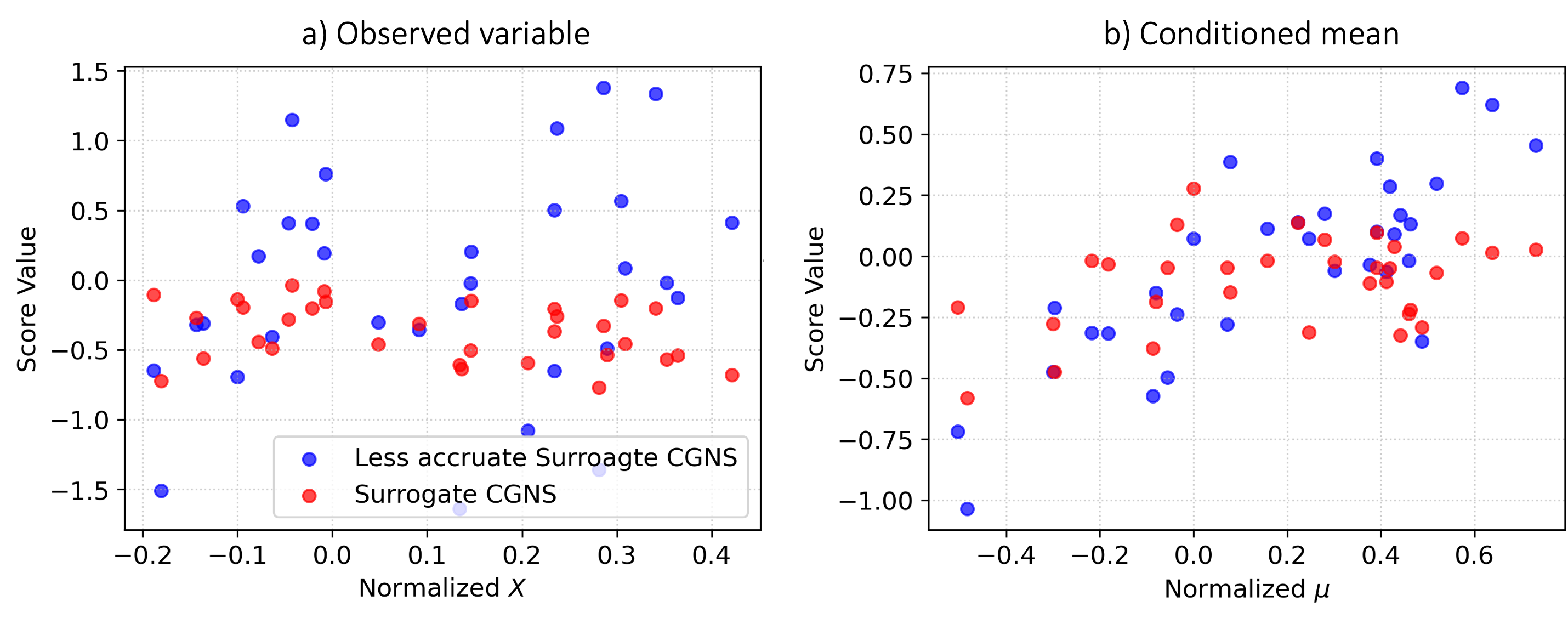}
	\caption{ The learned residual score corresponding to the first-indexed observed variable (a) and conditioned mean (b) for the cases: the less accurate surrogate model with $\sigma_s = 4$ and the accurate surrogate model $\sigma_s = 5$. The residual score is more pronounced when the surrogate model is less accurate.  }
	\label{fig:score_surrogate}
\end{figure}

\end{document}